\documentclass[a4paper]{article}[12pts]

\usepackage[english]{babel}
\usepackage[utf8x]{inputenc}
\usepackage[T1]{fontenc}
\usepackage{array,booktabs}
\normalsize
\usepackage{graphicx}

\usepackage{natbib}

\usepackage[a4paper,top=1in,bottom=1in,left=1in,right=1in,marginparwidth=1in]{geometry}

  \usepackage{setspace}
\usepackage{amsmath}
\usepackage{amssymb}
\usepackage{amsthm}
\usepackage{amsfonts, mathtools}
\usepackage{algorithm}
\usepackage{algorithmic}
\usepackage{bm}
 \usepackage{bbm}
\usepackage{comment}
\usepackage{graphicx}
\usepackage{multirow, booktabs}
\usepackage{caption}
\usepackage{subcaption}
\usepackage[colorinlistoftodos]{todonotes}
\usepackage[colorlinks=true, allcolors=blue]{hyperref}

\newcommand{\E}{\mathbb{E}}

\newtheorem{remark}{Remark}
\newtheorem{theorem}{Theorem}
\newtheorem{proposition}{Proposition}
\newtheorem{corollary}{Corollary}
\newtheorem{lemma}{Lemma}
\newtheorem{definition}{Definition}

\newtheorem{assumption}{Assumption}
\newtheorem{example}{Example}
\DeclareSymbolFont{letters}     {OML}{cmm} {m}{it}

\title{On Dynamic Pricing with   Covariates}
\author{Hanzhao Wang \and Kalyan Talluri \and Xiaocheng Li}
\date{\small
Imperial College Business School, Imperial College London\\
$\{$h.wang19, kalyan.talluri, xiaocheng.li$\}$@imperial.ac.uk\\
}
\DeclareMathOperator*{\argmax}{arg\,max}
\DeclareMathOperator*{\argmin}{arg\,min}

\begin{document}

\maketitle
\onehalfspacing
\begin{abstract}
We consider dynamic pricing with covariates under a generalized linear demand model: a seller can dynamically adjust the price of a product over a horizon of $T$ time periods, and at each time period $t$, the demand of the product is jointly determined by the price and an observable covariate vector $x_t\in\mathbb{R}^d$ through a generalized linear model with unknown co-efficients.  Most of the existing literature assumes the covariate vectors $x_t$'s are independently and identically distributed (i.i.d.); the few papers that relax this assumption either sacrifice model generality or yield sub-optimal regret bounds. In this paper, we show that UCB and Thompson sampling-based pricing algorithms can achieve an $O(d\sqrt{T}\log T)$ regret upper bound without assuming any statistical structure on the covariates $x_t$. Our upper bound on the regret matches the lower bound up to logarithmic factors.  We thus show that (i) the i.i.d. assumption is not necessary for obtaining low regret, and (ii) the regret bound can be independent of the  (inverse) minimum eigenvalue of the covariance matrix of the $x_t$'s, a quantity present in previous bounds. Moreover, we consider a constrained setting of the dynamic pricing problem where there is a limited and unreplenishable inventory and we develop theoretical results that relate the best achievable algorithm performance to a variation measure with respect to the temporal distribution shift of the covariates. We also discuss conditions under which a better regret is achievable and demonstrate the proposed algorithms' performance with numerical experiments.
\end{abstract}

\section{Introduction}

In this paper we consider the problem of dynamically pricing a product over time when additional covariate information is available. In the literature this is also referred to as ``feature-based dynamic pricing'', ``personalized dynamic pricing'', or ``dynamic pricing with side information''.  Over time periods $t=1,\ldots, T$ the seller observes covariates $x_t\in \mathbb{R}^d$ in each period, sets a price, and then observes the realized demand. The covariates generally are a set of features that characterize the pricing environment at both a macro level (market or environment) and a micro level (customer-specific).   The seller's objective is to maximize revenue, and the demand depends on the price and the covariates through a parametric function. The seller needs to balance the learning of the demand function with the earning of the revenue, the classic exploration-exploitation tradeoff. The performance of a pricing policy or algorithm is measured by the gap between the expected cumulative revenue obtained by the policy and that of an optimal policy which knows the demand function a priori.

At a high level, the existing literature on dynamic pricing with covariates can be grouped into three categories:
\begin{enumerate}
\item Allow arbitrary covariates but assume both the knowledge of the price-sensitivity coefficient and the parameters of the distribution of the demand shock. In Appendix~\ref{sec_modeling_issue} we discuss modeling issues associated with this assumption.
\item Assume i.i.d. covariates and allow an unknown price coefficient and an unknown parameter of the distribution of the demand shock.
\item Allow arbitrary covariates with an unknown price coefficient and parameters of the error distribution, but able to derive only a sub-optimal regret bound.
\end{enumerate}
The i.i.d. assumption in the second stream above can be restrictive in many situations in practice, say when there is serial correlation or peer effects in the data (we elaborate further in Appendix~\ref{sec_modeling_issue}). 

In this paper, we study a setting where the covariates $x_t$ can be arbitrary, i.e., no i.i.d. assumption  and the demand depends on the covariates and the price through a generalized linear model.  The main contributions of this paper are as follows:
\begin{itemize}
    \item We give a simple UCB-based pricing algorithm, which makes no assumptions on how the covariates are generated and yet achieves an optimal order of regret. This algorithm is inspired by the GLM-UCB algorithm for generalized linear reward bandits \citep{filippi2010parametric}, where however the reward is assumed to be generalized linear also in action also. Our reward function (revenue) does not have a generalized linear form over the action (price), and hence the analysis of the pricing algorithm has to take into account this misalignment between the reward and observation (demand).
    \item We next propose a Thompson sampling-based pricing algorithm, which is more computationally efficient than the UCB-based one while still enjoying an optimal order of regret.  By leveraging the special structure of the pricing problem we improve on two known limitations to  Thompson sampling for generalized linear bandits~\citep{abeille2017linear,agrawal2016linear}, namely: (1)  a sub-optimal order of regret with respect to the dimension of the covariates; (2) a requirement of convexity of the reward function with respect to the unknown parameters. 
    \item We extend the previous Thompson sampling dynamic pricing algorithm to a setting with an inventory constraint, where the seller has a limited and unreplenishable inventory at the beginning of the horizon. We show the distributional shift of the covariates adds an unavoidable term on the regret. Therefore, we modify the Thompson sampling algorithm into a dual-based algorithm to balance the inventory consumption by adaptively maintaining a dual variable.
\end{itemize}

The remainder of the paper is organized as follows. \S\ref{sec_lite_review} presents a brief review of the literature; \S\ref{sec_Model} gives the formal definition of the problem with assumptions and the performance measure; \S\ref{sec_UCB} designs a UCB-based pricing algorithm with a discussion on its computation efficiency; \S\ref{sec_TS} proposes a Thompson sampling pricing algorithm to alleviate the potential computation issue in the UCB pricing algorithm. We further extend the Thompson sampling pricing algorithm to a setting with an inventory constraint in \S\ref{sec_TS_C}. Finally, \S\ref{sec_num_diss} provides empirical simulations, and \S\ref{sec_conclusion} concludes the paper.

\subsection{Literature Review}
\label{sec_lite_review}

\begin{table}[ht!]
\centering
\scalebox{0.8}{
\renewcommand{\arraystretch}{1.5}
  \begin{tabular}{>{\centering}m{1.8in} |>{\centering}m{0.83in} |>{\centering}m{0.65in} |>{\centering\arraybackslash}m{1.4in}|>{\centering\arraybackslash}m{1.5in}}
\toprule
 Paper & Regret Bound  & Covariates & Demand Model & Key Assumptions \\
 \midrule
\cite{qiang2016dynamic} &$O(d \log T)$ & Mart. Diff.  & $x^\top\alpha+bp+\epsilon$ &  With known incumbent price\\\hline
 \cite{cohen2020feature} & $O(d^2 \log T)$
&Adver. & $\mathbbm{1}\{x^\top\alpha\geq p\}$ &  ---\\
\hline
 \cite{cohen2020feature} & $\tilde{O}(d^{19/6} T^{\frac{2}{3}})$ &
 Adver. & $\mathbbm{1}\{x^\top\alpha+\epsilon \geq p\}$ &  Sub-Gaussian noise\\
 \hline
 \cite{mao2018contextual} & $\tilde{O}( T^{\frac{d}{d+1}})$ &
 Adver. & $\mathbbm{1}\{f(x)\geq p\}$ &  $f(\cdot)$ is Lipschitz  \\
 \hline
 \cite{javanmard2017perishability} &$O(\sqrt{T})$ &Adver.& $\mathbbm{1}\{x^\top\alpha_t+\epsilon \geq p\}$ &  Log-concave noise with changing $\alpha_t$ \\
 \hline
 \cite{javanmard2019dynamic} &$O(d\log T)$ & i.i.d. & $\mathbbm{1}\{x^\top\alpha+\epsilon \geq p\}$ &  Log-concave noise and sparsity\\
 \hline
  \cite{javanmard2019dynamic}  &$\tilde{O}(s\log (d)\sqrt{T})$ & i.i.d. & $\mathbbm{1}\{x^\top\alpha+\epsilon \geq b p\}$ &  Log-concave noise and sparsity \\
  \hline
  \cite{luo2021distribution} &$\tilde{O}( d T^{\frac{2}{3}})$ & Adver. & $\mathbbm{1}\{x^\top\alpha+\epsilon \geq b p\}$ &  Log-concave noise with unknown distribution\\
  \hline
   \cite{xu2021logarithmic} &$O( d \log T)$ & Adver. & $\mathbbm{1}\{x^\top\alpha+\epsilon \geq  p\}$ & Strictly log-concave noise \\ \hline
       \cite{ban2021personalized}  & $\tilde{O}(s\log(d)\sqrt{T})$    &i.i.d.&   $g(x^\top\alpha+x^\top \beta \cdot p)+\epsilon$ &  Sub-Gaussian noise and sparsity\\   \hline
  Ours & $\tilde{O}(d\sqrt{T})$   & Adver. &  $g(x^\top\alpha+x^\top \beta \cdot p)+\epsilon$ &  Sub-Gaussian noise\\
  \bottomrule
\end{tabular}}
\caption{Summary of Existing Results: The notation $\tilde{O}(\cdot)$ omits the logarithmic factors, $d$ is the dimension of the covariates and $T$ is the horizon length. The column ``Regret Bound'' shows the regret upper bound, where $s$ is the number of non-zero coordinates of unknown parameters when sparsity structure is assumed. The column ``Covariates'' describes the assumption on the generation of the covariates, where ``Mart. Diff.'' stands for the the covariates sequence has martingale difference and ``Adver.'' means adversarial generations. In column ``Demand Model'', $\alpha, \beta\in \mathbb{R}^d$ and $b\in \mathbb{R}$, which are defaulted as fixed but unknown parameters without further comment in ``Key assumptions'', and $\epsilon$ is the unobserved error term.  The column ``Key assumptions''  summarizes key assumptions for the corresponding paper.}
\label{Tab:table_bound}
\end{table}
\cite{ban2021personalized} consider the generalized linear demand model which includes both the binary and linear demand models as special cases, extending the covariate-free case in \cite{broder2012dynamic}. We work with this generalized linear demand model, but our algorithm and analysis differ from that of \cite{ban2021personalized}. \cite{qiang2016dynamic} impose a martingale-type condition on the covariates which is similar to the i.i.d. assumption because both in essence ensure the minimum eigenvalue of the sample covariance matrix is bounded away from zero.  In Appendix~\ref{sec_non_iid}, we discuss the technical and practical advantages of removing the i.i.d. assumption.


\cite{ferreira2018online} consider the price-based network revenue management problem where there are finitely many allowable prices and a number of resource constraints.  \cite{javanmard2019dynamic} and \cite{ban2021personalized} pose their dynamic pricing problem in a high-dimensional setting, with a sparsity assumption, and perform a variable selection subroutine in their pricing algorithm via $\ell_1$ regularization. To this end, we believe the i.i.d. assumption and the dependence of the bounds on the minimum eigenvalue in both papers are necessary to handle the high-dimensional setting, in contrast to our low-dimensional one where we can remove those assumptions. \cite{keskin2023data} formulate a stylized framework for electricity pricing, where the covariates can be classified into three types based on their different forms of heterogeneity, and thus non-i.i.d. They develop a spectral clustering method to tailor the featured-based pricing, which is quite different from our algorithms. We summarize the existing works in Table \ref{Tab:table_bound}. For more on the history and origin of the dynamic pricing problem, especially the covariate-free case, we refer the readers to the review paper \cite{den2015dynamic}.


Our work is also related to generalized linear contextual bandits \citep{abeille2017linear,filippi2010parametric}. We utilize  techniques from this literature, while our work is differentiated in the following aspects: (1) there is a slight misalignment between the reward (revenue) and observation (demand) in the pricing setting, while in the bandits setting they are generally same; (2) in addition, the reward functions in these works are assumed to be generalized linear functions of actions; in contrast, our reward function (revenue) is not generalized linear in the action (price), only the observed demand is; (3) we further relax the assumptions and improve the regret bounds of \cite{abeille2017linear} by leveraging the special structure of the pricing problem.

\section{Problem Setup}
\label{sec_Model}
Consider the generalized linear demand model
\begin{equation}
\label{fun_demand}
    D_t=g(x_t^\top \alpha^*+x_t^\top \beta^* \cdot p_t)+\epsilon_t \quad \forall t=1,2,\ldots,T,
\end{equation}
where $\alpha^*, \beta^* \in \mathbb{R}^d$ are true unknown parameters, $g(\cdot)$ is a known link function, $x_t$ is an observable covariate vector (i.e., the seller observes $x_t$ before setting the price $p_t$), and $\epsilon_t$ is an unobservable and idiosyncratic demand shock in period $t$.  
 
The generalized linear model covers the two mainstream demand models, the binary demand model and the linear demand model, as special cases (see Appendix~\ref{sec_modeling_issue} for more details). Let $\mathcal{X}$ denote the domain of $x_t$ and $\theta^*\coloneqq (\alpha^*,\beta^*)$ as the concatenated true parameter vector.

We make the following assumptions for the demand model and the pricing rule. The boundedness assumption is standard in the dynamic pricing literature, and it is a rather harmless one for practical applications. We note that the last one of the assumptions can be usually met by suitable normalization of the covariates.

\begin{assumption}[Boundedness]  
\begin{itemize}
    \item[(a)] There exists $\bar{\theta}>0$ such that $\theta^* \in \Theta\coloneqq \{\theta\in \mathbb{R}^{2d}: \|\theta\|_2\leq \bar{\theta}\}$.
    \item[(b)] The seller sets prices in a range $[\underline{p},\bar{p}]$ for all $t=1,\ldots,T$.
    \item[(c)] There exists a known $\bar{D}>0$ such that $g(x^\top\alpha^*+x^\top\beta^* \cdot p)\leq \Bar{D}$ for all $x\in \mathcal{X}$ and $p \in [\underline{p},\bar{p}]$.
    \item[(d)] For all $x\in \mathcal{X}$ and $p \in [\underline{p},\bar{p}]$, $\|(x,px)\|_2\leq 1$.
\end{itemize}
\label{assp_bound}
\end{assumption}

\begin{assumption}[Properties of $g(\cdot)$]
$g(\cdot)$ is strictly increasing and differentiable with a bounded derivative over its domain. Specifically, there exist constants $\underline{g}, \bar{g} \in \mathbb{R}$ such that $0<\underline{g}\leq g'(z)\leq \bar{g} <\infty$ for all $z=x^\top\alpha+x^\top\beta \cdot p$ where $x\in \mathcal{X}$, $\theta \in \Theta$, and $p \in [\underline{p},\bar{p}]$.
\label{assp_g}
\end{assumption}

Assumption \ref{assp_g} is a standard assumption for the link function of the generalized linear model \citep{abeille2017linear,filippi2010parametric}, and it ensures the learnability of the parameters $(\alpha, \beta).$

\begin{assumption}[Demand Shock]
Assume $\{\epsilon_t, t=1,2\ldots\}$ forms a $\bar{\sigma}^2$-sub-Gaussian martingale difference sequence, i.e., $$\mathbb{E}[\epsilon_t|\mathcal{H}_{t-1}]=0 \text{ \ and \ } \log \left(\mathbb{E}\left[e^{s\epsilon_t}\big \vert \mathcal{H}_{t-1} \right]\right)\leq \frac{\bar{\sigma}^2 s^2}{2}$$ for all $s\in \mathbb{R}$, where $\mathcal{H}_t:=\sigma(p_1,\ldots,p_t,\epsilon_1,\ldots,\epsilon_t,x_1,\ldots,x_t,x_{t+1})$ and $\mathcal{H}_{0}:=\sigma\left(\emptyset, \Omega\right)$. Moreover, we assume $\bar{\sigma}^2$ is known a priori.
\label{assp_error}
\end{assumption}

Assumption \ref{assp_error} on the error term covers common distributions such as normal distribution and distributions with bounded support. The sub-Gaussian parameter $\bar{\sigma}^2$ works as an upper bound for the true variance of the random variable. We note that the filtration definition includes the covariates at time $t+1$. This small change allows the demand shock $\epsilon_t$ to be dependent on the covariates $x_t$ at that time period and gives more modeling flexibility. We remark that while such a martingale structure is assumed with respect to the demand shock, no assumption is imposed on the covariates $x_t$ other than the boundedness in Assumption \ref{assp_bound}.

\paragraph{Performance measure.}\

For simplicity, we denote $a^*_t\coloneqq x^\top_t\alpha^*$, $b^*_t \coloneqq x_t^\top\beta^*$. We write the expected revenue function as
$$r\left(p;a,b\right)\coloneqq p\cdot g(a+b\cdot p),$$
and the optimal expected revenue function as
$$r^*\left(a,b\right)\coloneqq \max_{p\in[\underline{p},\bar{p}]} \ p\cdot g(a+b\cdot p).$$
Throughout the paper, we assume this optimization problem can be efficiently solved.

Now, we define \textit{regret} as the performance measure for the problem. Specifically,
$$\text{Reg}_T^\pi(x_1,\ldots,x_T)  \coloneqq  \sum_{t=1}^T r^*(a^*_t,b^*_t) - \mathbb{E}\left[\sum_{t=1}^T r_t^{\pi}\right]$$
where $\pi$ denotes the online policy/algorithm in use, and the expectation is taken with respect to the demand shock $\epsilon_t$'s and the potential randomness from $\pi$. An admissible policy $\pi$  maps the  covariates $x_{1},...,x_t$, the past demand $D_1,\ldots,D_{t-1}$, and the past prices $p_1,\ldots,p_{t-1}$ to the price $p_t^{\pi}$ offered at $t$. Accordingly, $r_t^{\pi}$ denotes the revenue collected at time $t$ under $\pi$. The benchmark oracle (the first summation in above) is defined based on the optimal revenue function $r^*(a^*_t,b^*_t)$. It assumes knowledge of $(a^*_t,b^*_t)=(x_t^\top\alpha^*,x_t^\top\beta^*)$ but does not observe the realization of the demand shock $\epsilon_t$ when setting the price. In defining the regret, we allow the covariates $x_t$'s to be arbitrarily generated subject to Assumption \ref{assp_bound}. Therefore, no expectation is taken for $x_t$'s in the regret definition, and we seek a worst-case regret upper bound over all possible $x_t$'s. For the case when $x_t$'s are i.i.d., the regret definition involves one more layer of expectation on $x_t$'s. Thus our regret bound is  stronger and can directly translate into a regret bound for the i.i.d. case.
\section{Parameter Estimation and UCB-Based Pricing}

\label{sec_UCB}
In this section, we first introduce the maximum likelihood estimation for estimating the parameter $\theta^*$ based on the history data, which will be adopted in all pricing algorithms introduced later. The development of the approach is not new and the presentation here aims to generate specific intuitions for dynamic pricing. Based on it, we present the prototypical UCB-based algorithm for the problem.

\subsection{Maximum Quasi-likelihood Estimator and its Properties}
\label{sec_Quasi_MLE}
From the demand model \eqref{fun_demand}, we define the quasi-likelihood function \citep{wedderburn1974quasi} for the $t$-th observation as follows
\begin{equation}
    \label{eq_quasi_self}
    l_t(\theta)\coloneqq -\int_{D_{t}}^{g(z_{t}^\top \theta)}\frac{1}{h(u)}(u-D_{t})\mathrm{d}u,
\end{equation}
where $\theta=(\alpha, \beta)$ encapsulates the parameters, $z_{t}=(x_t, p_tx_t)$ is a column vector by concatenating the covariates, and $h(u)=g'(g^{-1}(u))$ for $u\in \mathbb{R}$. Since the distributions of the error terms (and thus the demands) are unknown, we can not maximize the exact likelihood function and can only use the quasi-likelihood.  To gain some intuition for $l_t(\cdot)$, we can treat it as the weighted mean of the gap between the observation $D_t$ and the estimator $g(z_t^\top \theta)$: by noting $1/h(u)=(g^{-1})'(u)$, it can be interpreted as a weight which captures the first-order information (derivative) of $\theta$ revealed at point $u$. Then the function $l_t$ as a first-order Taylor expansion proxy of $g^{-1}(\cdot)$ for the true likelihood function  aims to find a $\theta$ that minimizes the weighted gap between $g(z_{t}^\top \theta)$ and $D_t.$ To see this, when one maximizes $l_t$ over $\theta$, it achieves its maximum when $g(z_t^\top \theta)=D_t$ (if possible).

Based on $l_t$, we define the regularized quasi-likelihood estimator with parameter $\lambda>0$ as:
\begin{equation}
\label{eq_quasi_reg}
    \hat{\theta}_t:=\argmax_{\theta\in \Theta}
    -\frac{\lambda\underline{g}\|\theta\|_2^2}{2}  +\sum_{t'=1}^{t} l_{t'}(\theta),
\end{equation}
where $\underline{g}$ denotes the lower bound of $g'(\cdot)$ as defined in Assumption \ref{assp_g}. The estimator $\hat{\theta}_t$ will be used throughout the paper. The  regularization term avoids the singularity caused by the arbitrariness of the covariates $x_t$'s and also ensures a curvature for the likelihood function.

\paragraph{Estimation error of $\hat{\theta}_t$.}\

Define the (cumulative) sampled design matrix
$$M_{t}\coloneqq \lambda I_{2d}+\sum_{t'=1}^{t}z_{t'}z_{t'}^\top$$
where $I_{2d}$ is a $2d$-dimensional identity matrix, and the $M$-norm of a vector $z$ as
$\|z\|_M \coloneqq \sqrt{z^\top M z}$
for a positive definite matrix $M$. We use $\det M$ to denote the determinant of the matrix $M$. Then the lemma below bounds the estimation error of $\hat{\theta}_t$  according to the sampled design matrix $M_t$:
\begin{proposition}[\cite{filippi2010parametric}]
\label{prop_theta_est}
For any regularization parameter $\lambda>0$, the following bound 
$$\mathbb{P}\left( \exists t\in \{1,\ldots,T\}:  \left\|\hat{\theta}_t-\theta^* \right\|_{M_t}\geq2\sqrt{\lambda}\bar{\theta}+ \frac{2\bar{\sigma}}{\underline{g}}\sqrt{2\log\left(\frac{1}{\delta}\right)+\log\left(\frac{\det M_t}{\lambda^{2d}}\right)}\right)\leq \delta$$
holds for any $\delta\in(0,1)$.
\end{proposition}

The proposition is a key ingredient in removing the i.i.d. assumption on the covariates $x_t$. Specifically, if we impose an i.i.d. assumption on $x_t$ and assume the minimum eigenvalue of its covariance matrix is bounded away from zero, then we can upper bound the estimation error of $\theta^*$ in the Euclidean norm. 
As we do not make such assumptions, the above proposition tells us that we can still obtain an estimation error bound by measuring the distance according to the sampled design matrix $M_t.$ The following corollary makes the point even more precise by making the right-hand-side independent of $M_t$; we remark that this type of argument is commonly used in the linear bandits literature \citep{lattimore2020bandit}.

\begin{corollary}
For all $\lambda>0$,
$$\mathbb{P}\left( \exists t\in \{1,\ldots,T\}:  \left\|\hat{\theta}_t-\theta^* \right\|_{M_t}\geq 2\sqrt{\lambda}\bar{\theta}+ \frac{2\bar{\sigma}}{\underline{g}}\sqrt{2\log\left(T\right)+2d\log\left(\frac{2d\lambda +T}{2d\lambda}\right)} \right)\leq \frac{1}{T}. $$
\label{MLEbound}
\end{corollary}

\subsection{UCB-Based Pricing}
\label{sec_UCB_alg}

Algorithm \ref{alg_UCB} describes an upper confidence bound (UCB)-based pricing algorithm which naturally arises from the estimator $\hat{\theta}_t$ in \eqref{eq_quasi_reg}. At each time $t$, the algorithm first computes $\hat{\theta}_t$ and then constructs a confidence set based on the error bound in Corollary \ref{MLEbound}. Following the principle of optimism of UCB, the algorithm finds the most ``optimistic'' parameter within the confidence bound and sets the price pretending this parameter to be true. The algorithm design mostly mimics the GLM-UCB for generalized linear bandits \citep{filippi2010parametric} and here it differs from GLM-UCB in a misalignment of rewards and observations. Specifically, the (realized) reward at period $t$ collected from pulling the selected arm is just the observation at $t$ in generalized linear bandits and thus both the reward and observation are generalized linear functions. Here, the observations in the dynamic pricing setting contain the (realized) revenue and the demand at $t$, where the reward function is not a generalized linear function in action $p_t$ anymore (although the demand is). Therefore, when estimating the unknown parameter $\theta^*$, Algorithm \ref{alg_UCB} uses the observed demands by \eqref{eq_quasi_reg} instead of the observed rewards (revenues) as in \cite{filippi2010parametric}. This small twist is natural but critical in achieving a tighter regret bound (See Appendix~\ref{apx:diss_UCB} for more discussion).

\begin{algorithm}[ht!]
\centering
\caption{UCB Pricing}
\label{alg_UCB}
\begin{algorithmic}
\STATE{\textbf{Input:} Regularization parameter $\lambda$.}
\FOR{$t=1,\ldots,T$}
\STATE{Compute the estimators $\hat{\theta}_{t-1}$ by \eqref{eq_quasi_reg} and its confidence set
$$ \Theta_{t}:=\left\{\theta\in \Theta: \left\|\theta-\hat{\theta}_{t-1}\right\|_{M_{t-1}}\leq 2\sqrt{\lambda}\bar{\theta}+ \frac{2\bar{\sigma}}{\underline{g}}\sqrt{2\log (T)+2d\log\left(\frac{2d\lambda +T}{2d\lambda}\right)}\right\}.
$$}
\STATE{Observe covariates $x_t$ and choose the UCB parameter which maximizes the expected revenue:
\begin{equation}
\label{eq_UCB_opt}
   (\alpha_t,\beta_t)\coloneqq \argmax_{(\alpha,\beta)\in \Theta_{t}} \ \  r^*\left(x_t^\top\alpha,x_t^\top\beta\right).
\end{equation}}
\STATE{Set the price by
\begin{equation*}
    p_{t}=\argmax_{p\in [\underline{p},\bar{p}]} r(p;x^\top_t\alpha_t,x^\top_t\beta_t).
\end{equation*}
}
\ENDFOR
\end{algorithmic}
\end{algorithm}
%
Theorem \ref{thm_regret_bound} states the regret bound of Algorithm \ref{alg_UCB}. For dimension $d$ and horizon $T$, it meets the lower bound of the problem (See \citet{ban2021personalized}) up to logarithmic factors. Compared to the existing bounds \citep{qiang2016dynamic,javanmard2019dynamic, ban2021personalized}, our bound does not involve the term $\lambda_{\min}^{-1}$ where $\lambda_{\min}$ represents the minimum eigenvalue of the covariance matrix of $x_t$. For other parameters like $\bar{p}$ and $\bar{\theta}$, it is unclear whether their dependencies are optimal; We mention that they also appear in the existing regret bounds under the i.i.d. setting \citep{javanmard2019dynamic, ban2021personalized}.

\begin{theorem}
\label{thm:UCB_UB}
Under Assumptions \ref{assp_bound}, \ref{assp_g} and \ref{assp_error}, with any sequence $\{x_t\}_{t=1,\ldots,T}$, if we choose the regularization parameter $\lambda=1$, the regret of Algorithm \ref{alg_UCB} is upper bounded by
$$4\bar{g}\bar{p}\bar{\gamma}\sqrt{Td\log\left(\frac{2d+T}{2d}\right)}+\bar{p}\bar{D} = \tilde{O}(d\sqrt{T})$$
where $\bar{\gamma}=2\bar{\theta}+ \frac{2\bar{\sigma}}{\underline{g}}\sqrt{2\log T+2d\log\left(\frac{2d+T}{2d}\right)}$
represents an upper bound for the confidence volume.
\label{thm_regret_bound}
\end{theorem}

The proof largely mimics \cite{filippi2010parametric} except when leveraging the special structures of the pricing problem to deal with the fact that rewards (revenues) are not generalized linear functions in actions (prices) as discussed above. The analysis draws a connection between the dynamic pricing and the bandits problem, and also provides an alternative route for the existing analyses on the dynamic pricing with covariates problem (See \citet{qiang2016dynamic, ban2021personalized, zhu2020demands} among others). As noted earlier, this new analysis relaxes the i.i.d. assumption and removes the dependency on the inverse minimum eigenvalue of the covariance matrix.

\paragraph{Tractability of the UCB optimization problem \eqref{eq_UCB_opt}.} \

We note that in general the optimization subroutine \eqref{eq_UCB_opt} can be non-convex and hard to solve. In fact, computation efficiency is a common issue for UCB algorithms---the action selection step even for linear bandits can be NP-hard  \citep{dani2008stochastic}. One solution to compute the UCB optimization problem is through the Monte Carlo method: For each time $t$, the confidence set is an ellipsoid, so sampling from the confidence set can be done efficiently. Specifically, we randomly generate $M$ samples from the uniform distribution over the confidence set and denote the samples as $(\alpha_m,\beta_m)$'s. Then the optimization subroutine \eqref{eq_UCB_opt} can be approximately solved by
\begin{equation}
\label{eq_UCB_MC}
(\alpha^{\text{MC}}_t,\beta^{\text{MC}}_t) = \argmax_{m=1,\ldots,M} r^*(x_t^\top \alpha_m,x_t^\top\beta_m).
\end{equation}
In the numerical experiments, we try out different values of $M$ and examine the effect of sample size on the algorithm performance. With this Monte Carlo approach, the regret bound of Algorithm \ref{alg_UCB} will be revised by the following.
\begin{theorem}
\label{thm_UCB_MC_approx}
Under Assumptions \ref{assp_bound}, \ref{assp_g} and \ref{assp_error}, with any sequence $\{x_t\}_{t=1,\ldots,T}$, if we choose the regularization parameter $\lambda=1$, the regret of Algorithm \ref{alg_UCB} with a replacement of $(\alpha_t,\beta_t)$ by $(\alpha^{\text{MC}}_t,\beta^{\text{MC}}_t)$ is upper bounded by
$$4\bar{g}\bar{p}\bar{\gamma}\sqrt{Td\log\left(\frac{2d+T}{2d}\right)}+\bar{p}\bar{D}+\sum_{t=1}^T\mathbb{E}\left[r^*\left(x_t^\top\alpha_t,x_t^\top\beta_t\right)-r^*(x_t^\top\alpha^{\text{MC}}_t,x_t^\top\beta^{\text{MC}}_t)\right]$$
where $\bar{\gamma}=2\bar{\theta}+ \frac{2\bar{\sigma}}{\underline{g}}\sqrt{2\log T+2d\log\left(\frac{2d+T}{2d}\right)}$.
\end{theorem}

The additional summation term in the regret bound captures the approximation error brought by the Monte Carlo method. Generally, a finer bound can be obtained for this approximation error with some additional structure of the underlying function $r^*$ such as Lipschitzness. However, such a bound, or the Monte Carlo method itself, suffers from the curse of dimensionality (we defer more discussions to Appendix \ref{apx:diss_UCB}). 

Another source of computational suboptimality comes from solving the quasi-MLE problem \eqref{eq_quasi_self}. One cannot expect to solve the problem exactly in general, and an additional approximation error will also appear in the regret bound. We defer further discussions of this to Appendix~\ref{apx:appx_MLE}. To overcome the computational drawback of the UBC-based algorithm, we present a Thompson sampling-based algorithm in the next section.


\section{Thompson Sampling-based Pricing}
\label{sec_TS}
In this section we devise a  Thompson sampling-based method as a more efficient counterpoint to the  UCB-based algorithm of the previous section.  \citet{ferreira2016analytics} has discussed the possibility of applying Thompson sampling for revenue management where the authors studied the covariate-free case and analyzed the constrained dynamic pricing/revenue management problem. Our result focuses on the handling of covariates and, more importantly, we analyze the problem under a frequentist rather than a Bayesian setting. A common belief is that such a frequentist regret bound will suffer from suboptimal regret that scales on the order of $d^{3/2}$ \citep{abeille2017linear} with $d$ being the number of unknown parameters, compared to the $O(d)$ dependency under the UCB-based algorithm. We show that such suboptimality is surprisingly avoidable for this dynamic pricing problem, which renders the Thompson sampling both computationally efficient and regret-optimal. Our finding is similar in spirit to that of \citet{farias2022synthetically} on treatment effect estimation, though the analysis is different; both imply that the special structure such as a one-dimensional decision space, is helpful in reducing the Thompson sampling's regret order under the frequentist regime.



\subsection{Algorithm}
Algorithm \ref{alg_TS} presents the Thompson sampling-based pricing algorithm. The first step is the same as Algorithm \ref{alg_UCB}, which is to solve the optimization problem \eqref{eq_quasi_reg} and obtain the estimator $\hat{\theta}_{t-1}=(\hat{\alpha}_{t-1},\hat{\beta}_{t-1}).$ The standard Thompson sampling algorithm \citep{abeille2017linear} applies a randomized perturbation on the parameter $\hat{\theta}_{t-1}$ and obtains a perturbed estimation $\tilde{\theta}_{t-1}=\hat{\theta}_{t-1}+\eta_{t-1}$ with $\eta_{t-1}$ being some Gaussian noise. Then the Thompson sampling algorithm will pretend the parameter $\tilde{\theta}_{t-1}$ as the true $\theta^*$ and set the price accordingly. The key design of our Thompson sampling algorithm is to first project the parameter onto the direction of $x_t$ and obtain $(\hat{a}_t,\hat{b}_t)\coloneqq (x_t^\top \hat{\alpha}_{t-1}, x_t^\top \hat{\beta}_{t-1})$. Next, instead of perturbing in the original parameter space, we perturb this projected parameter to obtain $(\tilde{a}_t,\tilde{b}_t)$. Recall that the matrix $M_{t-1}$ measures the uncertainty of the current parameter estimation of $\hat{\theta}_{t-1}=(\hat{\alpha}_{t-1},\hat{\beta}_{t-1})$. To measure the uncertainty of $(\hat{a}_t,\hat{b}_t)$, we should project the uncertainty matrix as well by
\begin{equation}
\label{eq_M_tilde}
    \Tilde{M}_{t-1}\coloneqq \left( \begin{pmatrix}
x_t&\bm{0}\\
\bm{0}&x_t
\end{pmatrix}^\top M_{t-1}^{-1} \begin{pmatrix}
x_t&\bm{0}\\
\bm{0}&x_t
\end{pmatrix} \right)^{-1}
\end{equation}
where $\bm{0}$ is the zero vector with dimension $d$. It is easy to verify that $\Tilde{M}_{t-1}$ is well defined, i.e., the matrix inside the parentheses is invertible. Essentially, $\Tilde{M}_{t-1}$ guides the direction that we perform the perturbation in \eqref{eq_TS_samplestep}, and it shares the same intuition as the confidence set $\Theta_t$ in Algorithm \ref{alg_UCB}.


\begin{algorithm}[H]
\centering
\caption{Thompson Sampling Pricing}
\label{alg_TS}
\begin{algorithmic}
\STATE{\textbf{Input:} Regularization parameter $\lambda$.}
\FOR{$t=1,\ldots,T$}
\STATE{Compute the estimator $\hat{\theta}_{t-1}=(\hat{\alpha}_{t-1},\hat{\beta}_{t-1})$ by \eqref{eq_quasi_reg} and observe the covariates $x_t$.}
\STATE{Compute the estimator $(\hat{a}_t,\hat{b}_t)\coloneqq (x_t^\top \hat{\alpha}_{t-1}, x_t^\top \hat{\beta}_{t-1})$.}
\STATE{Sample $\eta_t\sim \mathcal{N}(0,I_{2})$ and compute the parameter
\begin{equation}
\label{eq_TS_samplestep}
    (\tilde{a}_{t},\tilde{b}_t)\coloneqq (\hat{a}_t,\hat{b}_t)+\left(2\sqrt{\lambda}\bar{\theta}+ \frac{2\bar{\sigma}}{\underline{g}}\sqrt{2\log T+2d\log\left(\frac{2d\lambda +T}{2d\lambda}\right)}\right)\tilde{M}_{t-1}^{-1/2}\eta_t.
\end{equation}
Set the price by
\begin{equation*}
    p_{t}=\argmax_{p\in [\underline{p},\bar{p}]} \ r(p;\tilde{a}_t,\tilde{b}_t),
\end{equation*}
and observe the demand $D_t$.}
\ENDFOR
\end{algorithmic}
\end{algorithm}

\subsection{Regret Analysis}
To analyze the regret of Algorithm \ref{alg_TS}, we need a minor modification of Assumption \ref{assp_g}.

\begin{assumption}[Properties of $g(\cdot)$]
Let $$\tilde{\Theta}:=\left\{\theta \in \mathbb{R}^{2d}:  \|\theta- \tilde{\theta}\|_2\leq 4\bar{\theta}\sqrt{\log (4T^2)}+ \frac{4\bar{\sigma}}{\underline{g}}\sqrt{\left(2\log T+2d\log\left(\frac{2d+T}{2d}\right)\right)\log (4T^2)} \text{ for some } \tilde{\theta}\in \Theta\right\}$$ where $\Theta$ is as defined in Assumption \ref{assp_bound}. We assume $g(z)$ is strictly increasing, differentiable and there exist constants $\underline{g}, \bar{g} \in \mathbb{R}$ such that $0<\underline{g}\leq g'(z)\leq \bar{g} <\infty$ for all $z=x^\top\alpha+x^\top\beta \cdot p$ where $x\in \mathcal{X}$, $\theta = (\alpha,\beta) \in \Tilde{\Theta}$, and $p \in [\underline{p},\bar{p}]$.
\label{assp_g'}
\end{assumption}

Assumption \ref{assp_g'} is a stronger version of Assumption \ref{assp_g} on the properties of $g$ in that it requires the same property of $g$ to hold across a larger region $\tilde{\Theta}$ than $\Theta$ in Assumption \ref{assp_g}. Specifically, the enlargement of $\Theta$ to $\tilde{\Theta}$ is to ensure the property of $g$ covers the range of randomized sampled parameters in the Thompson sampling step.

The following lemma is adapted from Lemma 3 of \citet{abeille2017linear}. It implies that the sampled parameters $(\tilde{a}_t,\tilde{b}_t)$ are ``optimistic'' against the true parameter $(a_{t}^*, b_{t}^*)$ with a certain probability for any convex function $f$. This probabilistic optimism is in parallel to the UCB design. The conditional part  holds with high probability as in the previous UCB case.
\begin{lemma}
\label{Abeille_opt_prob}
Let $(a_{t}^*, b_{t}^*)=(x_t^\top \alpha^*, x_t^\top\beta^*)$. Then for the sampled parameter $(\tilde{a}_t,\tilde{b}_t)$ from \eqref{eq_TS_samplestep} in Algorithm \ref{alg_TS} , the following inequality holds for any $t\ge1$ and  any function $f(a,b)$ that is convex in $(a,b),$
$$\mathbb{P}\left(f(\tilde{a}_t,\tilde{b}_t)\geq f(a_{t}^*, b_{t}^*) \bigg \vert \mathcal{H}_{t-1}, (\alpha^*, \beta^*)\in \Theta_{t}\right)\geq \frac{1}{4\sqrt{e\pi}},$$
where
$$ \Theta_{t}=\left\{\theta\in \Theta: \left\|\theta-\hat{\theta}_{t-1}\right\|_{M_{t-1}}\leq 2\sqrt{\lambda}\bar{\theta}+ \frac{2\bar{\sigma}}{\underline{g}}\sqrt{2\log (T)+2d\log\left(\frac{2d\lambda +T}{2d\lambda}\right)}\right\},
$$
is the confidence set defined in the UCB algorithm (Algorithm \ref{alg_UCB}).
\end{lemma}

The above lemma is critical in our regret analysis and that of \citet{abeille2017linear}. Yet there is a gap, as the revenue function is not convex in terms of the parameters $(a,b).$ The following lemma bridges this gap. It basically says that the optimal revenue of the sampled parameters will be optimistic as long as the sampled parameters are optimistic for a linear function.

\begin{lemma}
\label{lem:linear2reward}
If $\tilde{a}_t+\tilde{b}_t\cdot p^*_t\geq a^*_t+b^*_t\cdot p^*_t$ where $p^*_t\coloneqq \argmax_{p\in[\underline{p},\bar{p}]}r(p;a^*_t,b^*_t)$ is the optimal pricing at $t$, then $r^*(\tilde{a}_t,\tilde{b}_t)\geq r^*(a^*_t,b^*_t)$ where recall $r^*\left(a,b\right)= \max_{p\in[\underline{p},\bar{p}]} \ p\cdot g(a+b\cdot p)$.
\end{lemma}

The lemma justifies the algorithm design of perturbing in the projected parameter space  $(a,b)$ rather than the original parameter space $(\alpha,\beta).$ There are two benefits to doing this: (i) It reduces the dimensionality of the perturbation and the regret bound; though we conduct parameter estimation in the original $2d$-dimensional space, the key sampling and analysis happen in this $2$-dimensional space which is the reason for bound improvement. (ii) It relaxes the original requirement in \citet{abeille2017linear} that the reward function must be convex in the parameters, which limits the algorithm's application in the dynamic pricing context. Both benefits arise from the structure of the pricing and revenue function, and the proof of Lemma \ref{lem:linear2reward} is simple algebraic manipulation.

\begin{theorem}
Under Assumption \ref{assp_bound}, \ref{assp_error},  \ref{assp_g'}, with any sequence $\{x_t\}_{t=1,\ldots,T}$ and any $T\geq 6$, if we choose the regularization parameter  $\lambda=1$, the regret of Algorithm \ref{alg_TS} can be bounded by
$$36\Bar{g}\Bar{p}\bar{\gamma}\sqrt{2e \pi dT\log\left(\frac{2d+T}{2d}\right)\log(4T^2)}+2\bar{D}\bar{p}=\tilde{O}\left(d \sqrt{T}\right),$$
where $\bar{\gamma}=2\bar{\theta}+ \frac{2\bar{\sigma}}{\underline{g}}\sqrt{2\log T+2d\log\left(\frac{2d+T}{2d}\right)}$.
\label{thm_TS_NC}
\end{theorem}

Theorem \ref{thm_TS_NC} gives the regret bound of Algorithm \ref{alg_TS}. The proof of the theorem mimics the analysis by \citet{abeille2017linear} and the key idea is to construct random parameters $(\tilde{a}_t', \tilde{b}_t')$ such that (i) the revenue $r^*(\tilde{a}_t', \tilde{b}_t')$ upper bounds $r^*({a}^*_t, {b}^*_t)$ with high probability and (ii) the parameters $(\tilde{a}_t', \tilde{b}_t')$ connects to $(\tilde{a}_t, \tilde{b}_t)$ in a way that the single-step regret can be bounded by the volume of the confidence set (as the case of the UCB-based algorithm). Yet, in terms of the regret bound, we provide an analysis of the original Thompson sampling algorithm \citep{abeille2017linear} for comparison in Appendix~\ref{apx:TS_vs_newTS} which results in a regret bound on the order of $d^{\frac{3}{2}}\sqrt{T}$. Furthermore, in comparison with Algorithm \ref{alg_UCB}, Algorithm \ref{alg_TS} achieves the same order of regret bound with a much smaller computation cost. In all, these results say that the Thompson sampling algorithm can be a computationally efficient substitute for the UCB-based algorithm without any compromise in the theoretical guarantee.

\section{Pricing with Inventory Constraint}
\label{sec_TS_C}

In the previous section, we consider the case where there is no inventory constraint. Here we proceed to a constrained setting. Specifically, consider the seller starts with an inventory of $C$. Let $C_t$ be the remaining inventory at the beginning of time period $t$ and $C_1=C.$ Accordingly, there is a minor twist in demand because of this inventory constraint. We define the realized sale/demand by
$$\Tilde{D}_t\coloneqq \min\{C_t,D_t\}$$
where $D_t$ is the (true) demand in \eqref{fun_demand} and $C_t$ follows the dynamic
$$C_t\coloneqq C_1-\sum_{t'=1}^{t-1} \Tilde{D}_{t'}.$$
In other words, for the exact time period that the inventory is exhausted, the realized sale will be the demand $D_t$ truncated by the available inventory.

The seller's objective is to maximize the expected revenue over the horizon
$$\mathbb{E}\left[\sum_{t=1}^Tp^{\pi}_t\Tilde{D}_t\right]$$
where the policy $\pi$ specifies the price $p^{\pi}_t$.

\subsection{Model}
For the case of unlimited inventory, whether the covariates are i.i.d. or non-i.i.d. do not significantly affect the best achievable regret bound. For the constrained case however, it does make a difference. In this section we consider a setting where the covariates are independent but non-identical random vectors over time.

\begin{assumption}[Independent covariates]
  \label{assmp_gen_x_t}
The covariates $X_t \sim \mathcal{P}_t$ for $t=1,\ldots T$ where $\mathcal{P}_t$'s are unknown distributions. $X_t$'s are independent of each other and also with the demand shocks $\epsilon_t$'s.
\end{assumption}

Here we use the notation $X_t$ to distinguish these random covariates from the deterministic yet arbitrarily chosen $x_t$ in the previous sections. We remark that the assumption does not impose additional restrictions on the covariates than the arbitrarily chosen case. To recover the setting in \S \ref{sec_UCB} and \S \ref{sec_TS}, we can set $\mathcal{P}_t$ to be point-mass distributions with $\mathbb{P}(X_t=x_t)=1$. To the other extreme, if we have $\mathcal{P}_t=\mathcal{P}$, it recovers the setting of i.i.d. covariates. In this light, the aim of introducing $\mathcal{P}_t$'s is to prepare for a finer characterization of the non-i.i.d. covariates but without any compromise on modeling flexibility.


As for the unconstrained case, we define the benchmark revenue as
\begin{equation}
\label{eq_OPTbenchmark}
    \text{OPT}(x_1,\ldots,x_T,C)\coloneqq \max_{\pi\in \Pi_{\text{hind}}}\mathbb{E}\left[\sum_{t=1}^T p^{\pi}_t\Tilde{D}_t\right]
\end{equation}
where $\Pi_{\text{hind}}$ includes all pricing policies $\pi$ which maps
$$\left\{x_1,\ldots,x_T,D_{1},\ldots,D_{t-1}, p_1,\ldots,p_{t-1},\theta^*,C\right\}$$
to the price $p^{\pi}_t$ offered at time $t$. Here the expectation is taken only with respect to the $\epsilon_t, t=1,\ldots,T$ and the potential randomness of the pricing policy $\pi$. In defining the benchmark, the hindsight pricing policy is allowed to use the future covariates information and also assumes the knowledge of the true model parameter $\theta^*.$ Accordingly, the regret of an online policy is defined by
$$\text{Reg}_T^\pi(\mathcal{P}_1,\ldots,\mathcal{P}_T) \coloneqq \mathbb{E}\left[  \text{OPT}(X_1,\ldots,X_T,C) -\sum_{t=1}^Tp^{\pi}_t\Tilde{D}_t\right]$$
where the expectation is taken with respect to the demand shock $\epsilon_t$'s, covariates $X_t$, and potential randomness in $\pi$. The online policy $\pi \in \Pi_{\text{on}}$ does not observe future covariates information $\{x_{t'}\}_{t'=t+1}^T$ or know the true parameter $\theta^*$.

We introduce the following measure to quantify the variation of the distributions over time:
$$\mathcal{W}_T(\mathcal{P}_1,\ldots,\mathcal{P}_T)\coloneqq \sum_{t=1}^T \mathcal{W}(\mathcal{P}_t,\bar{\mathcal{P}}_T),$$
where $\mathcal{W}(\cdot,\cdot)$ is the Wasserstein distance function to measure the distance between two distributions \citep{kantorovich1960mathematical} with the Euclidean distance as the underlying metric and $\bar{\mathcal{P}}_T \coloneqq \frac{\sum_{t=1}^T \mathcal{P}_t}{T}$, i.e., the uniform mixture of $\mathcal{P}_t$'s. When $\mathcal{P}_t$'s are identical for $t=1,\ldots,T$, then $\mathcal{W}_T(\mathcal{P}_1,\ldots,\mathcal{P}_T)=0$. This measure  has been used to quantify the intensity of non-i.i.d.-ness or non-stationarity of online optimization algorithms under a full information environment \citep{jiang2020online} and under a partial information environment \citep{liu2022non}.

\begin{theorem}[Regret Lower Bound]
\label{thm_LB}
The following lower bound holds for any policy $\pi \in \Pi_{\text{on}}$:
$$\textup{Reg}_T^\pi(\mathcal{P}_1,\ldots,\mathcal{P}_T)=\Omega(\mathcal{W}_T(\mathcal{P}_1,\ldots,\mathcal{P}_T)).$$
\end{theorem}

Theorem \ref{thm_LB} gives a lower bound of the regret which is exactly the variation measure above. In the previous unconstrained case, we note that the regret is affected by the covariates $x_t$'s through the bound on its norm. In contrast, for this constrained case, the variation in the covariates does affect the regret lower bound. The intuition for this contrast is that for the unconstrained setting, the variation in distribution only affects the learning procedure, i.e., the estimation of the true model parameter. For the constrained setting, the variation results in an additional challenge in balancing the resource consumption: the seller does not know whether the future periods will be more profitable or less, and hence whether to reserve more inventory or less for the future periods.


\subsection{Algorithm}

To generate some intuition for the algorithm, we first write out the deterministic optimization problem for the pricing problem,
\begin{align}
    \label{eq_RT_Det}
    \max_{p_1,...,p_T}\  & \sum_{t=1}^T \mathbb{E}\left[p_t \cdot g(A_t+B_t\cdot p_t)\right] \\
\text{s.t.}\  &  \sum_{t=1}^T \mathbb{E}\left[ g(A_t+B_t\cdot p_t)\right]\leq cT \nonumber
\end{align}
where $A_t = X_{t}^\top \alpha^*$ and $B_t = X_{t}^\top \beta^*$. Here $c=C/T$ denotes the average available inventory per time period. The optimization problem replaces the random demand with the corresponding expectation and only requires the constraint to be satisfied on expectation. It is easy to verify that the optimal value of this deterministic program upper bounds the hindsight benchmark $\E[\text{OPT}(X_1,...,X_T,C)]$. This upper bound relationship makes the deterministic program a nice starting point for algorithm development and theoretical analysis in the literature \citep{gallego1997multiproduct,talluri2004theory}.

The Lagrangian of \eqref{eq_RT_Det} is
$$L(\mu,p_{1:T})\coloneqq c\mu T+\sum_{t=1}^T \mathbb{E}\left[(p_t-\mu)\cdot g(A_t+B_t\cdot p_t)\right]$$
where the primal decision variables are $p_1,...,p_T$ and the dual decision variable is $\mu$. Such a formulation inspires the development of dual-based algorithms \citep{agrawal2016linear, li2022online}.

First, let $\mu^*$ be the optimal dual solution, and then the optimal primal price should accordingly be determined by
\begin{equation}
p_{t}^* = \argmax_{p_{t}\in [\mu^*, \bar{p}]} \ (p_t-\mu^*)\cdot g(A_t+B_t\cdot p_t).
\label{eqn:primal_opt_constr}
\end{equation}
Intuitively, the additional term $\mu^*$ reflects the cost arising from the limited inventory and can be interpreted as the shadow price. Also, we remark that the price in \eqref{eq_RT_Det} should be a function of the covariates $X_t$; we omit this dependency in the notation for simplicity without hurting the generality.

Second, if we see the primal price variables as fixed and focus on the dual variable, the Lagrangian becomes
\begin{equation}
L(\mu,p_{1:T})=\mu\cdot \sum_{t=1}^T \mathbb{E}\left[c-g(A_t+B_t\cdot p_t)\right] +\sum_{t=1}^T \mathbb{E}\left[p_t\cdot g(A_t+B_t\cdot p_t)\right]
\label{eqn:dual_opt_constr}
\end{equation}
where only the first summation involves the dual variable.

Algorithm \ref{alg_TS_C} implements the Thompson sampling algorithm for this constrained setting. The parameter learning part is the same as Algorithm \ref{alg_TS}. The price optimization step \eqref{eq_price_con} optimizes \eqref{eqn:primal_opt_constr} using $\mu_t$ as an estimate of $\mu^*$. The update of the dual price $\mu_t$ in \eqref{eqn:OGD_mu} essentially performs a projected online gradient descent with respect to \eqref{eqn:dual_opt_constr} and the step size is chosen to be $\frac{\bar{p}}{\sqrt{(\Bar{D}^2+\Bar{\sigma}^2)T}}$.

\begin{algorithm}[ht!]
\centering
\caption{Thompson Sampling with Inventory Constraint}
\label{alg_TS_C}
\begin{algorithmic}
\STATE{\textbf{Input:} Regularization parameter $\lambda$, initial dual variable $\mu_1=0$.}
\FOR{$t=1,\ldots,T$}
\STATE{Compute the estimator $\hat{\theta}_{t-1}=(\hat{\alpha}_{t-1},\hat{\beta}_{t-1})$ by \eqref{eq_quasi_reg} and observe the covariates $X_t=x_t$.}
\STATE{Compute the estimator $(\hat{a}_t,\hat{b}_t)\coloneqq (x_t^\top \hat{\alpha}_{t-1}, x_t^\top \hat{\beta}_{t-1})$.}
\STATE{Sample $\eta_t\sim \mathcal{N}(0,I_{2})$ and compute the parameter
$$(\tilde{a}_{t},\tilde{b}_t)= (\hat{a}_t,\hat{b}_t)+\left(2\sqrt{\lambda}\bar{\theta}+ \frac{2\bar{\sigma}}{\underline{g}}\sqrt{2\log T+2d\log\left(\frac{2d\lambda +T}{2d\lambda}\right)}\right)\tilde{M}_{t-1}^{-1/2}\eta_t.$$
Set the price by
\begin{equation}
\label{eq_price_con}
    p_{t}=\argmax_{p \in [\mu_t, \bar{p}]} \ (p-\mu_t)\cdot g(\tilde{a}_{t}+\tilde{b}_{t}\cdot p),
\end{equation}
and observe the demand $D_t$. }
\STATE{
If the remaining inventory can not fulfill the demand, i.e., $D_t\ge C_t$, then the algorithm stops, otherwise update the dual variable:
\begin{equation}
\mu_{t+1}=\text{Proj}_{[0,\Bar{p}]}\left(\mu_{t}+\frac{\bar{p}}{\sqrt{(\Bar{D}^2+\Bar{\sigma}^2)T}}\cdot (D_{t}-c)\right)
\label{eqn:OGD_mu}
\end{equation}
and the remaining inventory
$$C_{t+1} = C_t - D_t.$$
}
\ENDFOR
\end{algorithmic}
\end{algorithm}


\subsection{Regret Analysis}

For regret analysis, we introduce a few additional assumptions to ensure feasibility and tractability:
\begin{assumption}[Feasibility and Tractability]
\label{assmp_feasible}
We assume
\begin{itemize}
    \item[(a)] The average inventory per time period $c=C/T\in[0,\bar{D}]$.
    \item[(b)]  For all $x\in\mathcal{X}$, we have $x^\top \beta^*<0$ and $g(x^\top \alpha^*+x^\top \beta^* \cdot \bar{p})< c$.
    \item[(c)] The function $d\cdot g^{-1}(d)$ is a convex function for $d\in[0,\Bar{D}]$.
\end{itemize}
\end{assumption}
Part (a) and Part (b) ensure that the constrained problem is well-defined. Part (c) ensures the single-period pricing problem can be efficiently solved and is a standard assumption in the revenue management literature \citep{gallego1997multiproduct,chen2022primal,chen2022network} and holds for many $g(\cdot)$'s, including $g(x)=x$, $g(x)=x^2$ and $g(x)=\exp(x)$.


\begin{theorem}[Regret Upper Bound]
  \label{thm_TS_C_UB}
  Under Assumption \ref{assp_bound}, \ref{assp_error}, \ref{assp_g'}, \ref{assmp_gen_x_t}, \ref{assmp_feasible} with $T\geq 6$, if we choose the regularization parameter  $\lambda=1$, the regret of Algorithm \ref{alg_TS_C} can be bounded by
  \begin{multline*}
      40\Bar{g}\Bar{p}\bar{\gamma}\sqrt{2e \pi dT\log\left(\frac{2d+T}{2d}\right)\log(4T^2)}+\Bar{p}\sqrt{(\bar{D}^2+\bar{\sigma}^2)T}+7\Bar{D}\Bar{p}+\sqrt{2}(\bar{p}+\Bar{p}^2)\bar{g}\bar{\theta} \mathcal{W}_T(\mathcal{P}_1,\ldots,\mathcal{P}_T)\\
      =\tilde{O}\left(d \sqrt{T}+\mathcal{W}_T(\mathcal{P}_1,\ldots,\mathcal{P}_T)\right),
  \end{multline*}
  where $\bar{\gamma}=2\bar{\theta}+ \frac{2\bar{\sigma}}{\underline{g}}\sqrt{2\log T+2d\log\left(\frac{2d+T}{2d}\right)}$
\end{theorem}

To interpret the regret bound, we first make a comparison against Theorem \ref{thm_LB}. The variation measure $\mathcal{W}_T$ term appears in both the lower and upper bounds, and thus it provides a tight characterization of how the change of distribution affects the algorithm performance. Next, when the distributions $\mathcal{P}_t$'s are identical to each other, the $\mathcal{W}_T$ term disappears and the result recovers the existing bounds for the i.i.d. setting \citep{keskin2014dynamic,li2020dynamic}. 

Surprisingly the regret bound under the partial-information setting of the pricing problem, 
(i)~matches the one for full-information online constrained stochastic optimization \citep{jiang2020online}, and (ii)~is strictly better than the one for bandits with knapsacks (BwK) \citep{liu2022non} which is partial-information setting as well. This is because, in our partial-information pricing problem  the underlying demand model does not change over time, i.e., $(\alpha^*, \beta^*)$ remains the same over time; only the covariates distribution $\mathcal{P}_t$ changes over time. If  the demand model were to change over time, the regret bound above will become aligned with the BwK regret bound of \citet{liu2022non}. Finally, in comparison with the previous unconstrained setting, the results in this section show that the variation in the covariates distribution substantially affects the algorithm performance under the constrained setting.


\section{Numerical Experiments}
\label{sec_num_diss}

We consider the linear demand model $ D_t = x_t^\top \alpha^*+x_t^\top \beta^* \cdot p_t+\epsilon_t$ with two groups of numerical experiments:
\begin{itemize}
\item[(a)] Without inventory constraint: we implement and compare the performance of Algorithm \ref{alg_UCB} (with different numbers of Monte Carlo samples), Algorithm \ref{alg_TS}, and benchmark algorithms.
\item[(b)] With inventory constraint: we implement Algorithm \ref{alg_TS_C} and compare its performance against benchmark algorithms. We also examine the effect of different $\mathcal{W}_T$'s on the regret.
\end{itemize}
All the numeric results (such as the curves in the plot) are reported based on the average performance over 100 simulation trials. More experiment details are deferred to Appendix~\ref{sec_appn_num}.

\paragraph{Without inventory constraint.}\

\textbf{Environment Setup.} We set the horizon $T=1500$. The whole horizon is split into two phases with equal length. In the first phase, the first half of the covariates (dimension $1$ to $d/2$) are i.i.d. generated over time, while the second half (dimension $d/2+1$ to $d$) are all zero; in the second phase, the first half of the covariates (dimension $1$ to $d/2$) are all zero, while the second half (dimension $d/2+1$ to $d$) are i.i.d. generated over time. Moreover, we set the allowable price range as $[0.1,5]$. For each simulation trial, the true parameter $\alpha^*$ is generated by a uniform distribution over $\frac{1}{\sqrt{d}}[1,2]^{d}$ and the true parameter $\beta^*$ is generated by a uniform distribution over $-\frac{1}{\sqrt{d}}[0,1]^{d}$, the covariates (if non-zero) are always generated i.i.d. from a uniform distribution over $\frac{1}{\sqrt{d}}[0,1].$ The normalizing factor $\frac{1}{\sqrt{d}}$ ensures the demand always stays on the same magnitude for different $d$.

\textbf{Results.} Figure \ref{fig:d6} and \ref{fig:d12} compare the cumulative regrets of Algorithm \ref{alg_UCB} (UCB) with different Monte Carlo samples, and Algorithm \ref{alg_TS} (TS), with two benchmarks under $d=6$ and $d=12$ respectively. The first benchmark algorithm is Algorithm \ref{alg_TS_exante}, the original Thompson sampling  (TS\_Ori) which does the random perturbation in the original parameter space instead of the projected space. For the second benchmark, we adapt the covariate-free constrained iterated least square (CILS) algorithm \citep{keskin2014dynamic} designed for the i.i.d. covariates setting. We refer to Appendix~\ref{sec_appn_num} for more details about CILS. We notice that the CILS algorithm performs well for the first phase but fails to learn the parameter
during the second phase. UCB performs stably well under different numbers of Monte Carlo samples. At the beginning of the second phase, the curves of UCB and TS grow quickly but then they all flatten. In addition,  TS outperforms the original Thompson sampling TS\_Ori in both figures, which is consistent with the improved factor of $\sqrt{d}$ in our theoretical regret bound.

\begin{figure}[ht!]
     \centering
     \begin{subfigure}[b]{0.45\textwidth}
         \centering
         \includegraphics[width=\textwidth]{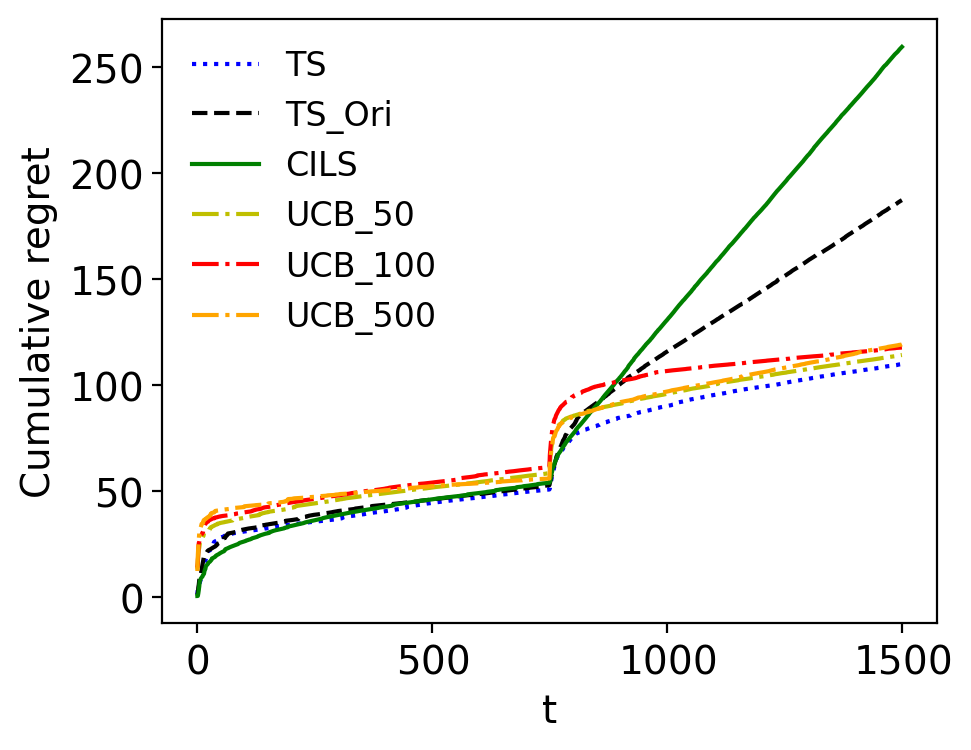}
         \caption{Cumulative regret, $d=6$.}
         \label{fig:d6}
     \end{subfigure}
     \begin{subfigure}[b]{0.45\textwidth}
         \centering
         \includegraphics[width=\textwidth]{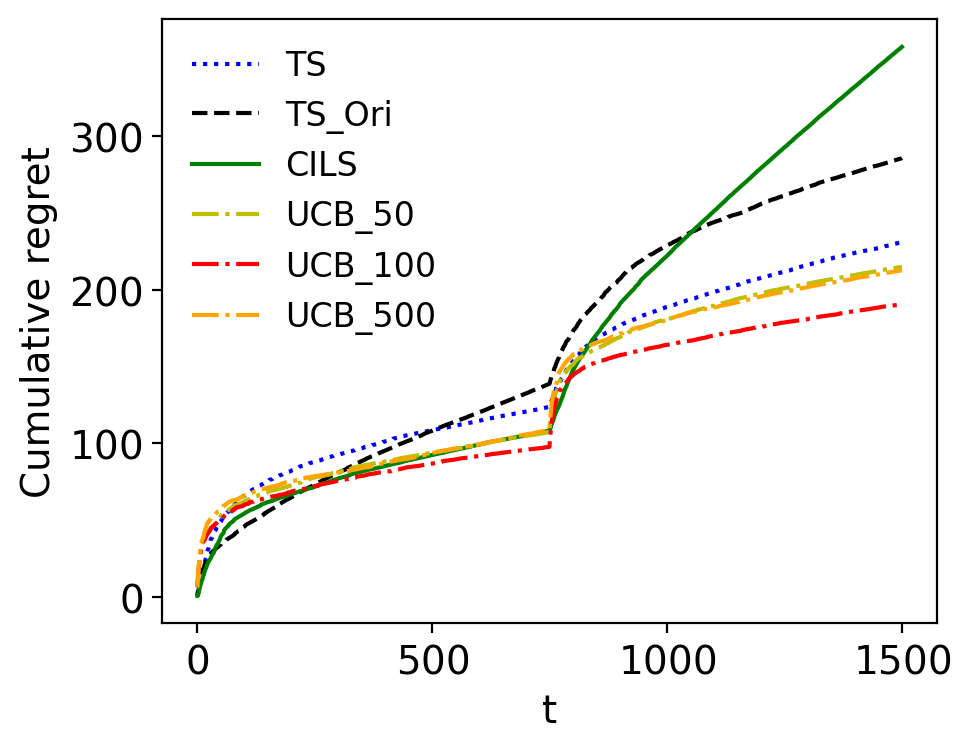}
         \caption{Cumulative regret, $d=12$.}
         \label{fig:d12}
     \end{subfigure}
        \caption{Experiment (a): performances of algorithms without inventory constraint.}
        \label{fig:uncon}
\end{figure}

\paragraph{With inventory constraint.}\

\textbf{Environment Setup.} We also consider splitting the horizon $T$ into two phases with equal length. The covariates follow three patterns with different variation measure $\mathcal{W}_T$'s:
\begin{itemize}
    \item Large:  the covariates are generated i.i.d. from a uniform distribution over $\frac{0.1}{\sqrt{d}}[0,1]$ in the first phase and  $\frac{5}{\sqrt{d}}[0,1]$ in the second phase, which gives the largest $\mathcal{W}_T$ among the three patterns.
    \item Small: the covariates are generated i.i.d. from a uniform distribution over $\frac{1}{\sqrt{d}}[0,1]$ in the first phase and  $\frac{5}{\sqrt{d}}[0,1]$ in the second phase, where $\mathcal{W}_T$ is smaller than the case of Large.
    \item Zero: the covariates are generated i.i.d. from a uniform distribution over  $\frac{5}{\sqrt{d}}[0,1]$ in both phases, where the variation measure $\mathcal{W}_T=0$.
\end{itemize}

We set the averaged inventory over periods as $c=C/T=0.5$. As a reference point, the averaged optimal demands for the above three patterns when there is no inventory constraint are approximately $1, 1.2, 1.9,$ respectively. We set the dimension of covariates $d=6$. The generations of $\theta^*$ and the range of pricing remain the same as before.

\textbf{Results.} Figure \ref{fig:reg_con} reports the performance of Algorithm \ref{alg_TS_C} (TS\_Dual) over $T=100,300,\ldots,1500$ with two other benchmarks under the Large pattern. The benchmark algorithm Greedy\_Single solves the deterministic single-period pricing problem under inventory $c$ with using $\hat{\theta}_t$ as $\theta^*$ and algorithm Greedy\_Dual is same as TS\_Dual expect replacing $(\tilde{a}_t,\tilde{b}_t)$ by $(\hat{a}_t,\hat{b}_t)$, i.e., without the Thompson sampling step. More details of these two algorithms can be found in Appendix~\ref{sec_appn_num}.  Figure \ref{fig:reg_con} shows TS\_Dual outperforms two benchmark algorithms consistently for different $T$'s. Furthermore, its regret curve does not scale with $O(\sqrt{T})$. This verifies the results in Theorem \ref{thm_TS_C_UB} where the regret upper bound includes an additional term $\mathcal{W}_T$ which also depends on $T$. We further examine the influence of $\mathcal{W}_T$ on the regret. We plot the regrets of TS\_Dual over $T=100,300,\ldots,1500$ under different patterns in Figure \ref{fig:reg_con_var}. We observe that the regrets are consistently increasing in $\mathcal{W}_T$ for different $T$'s, which numerically coincides the results in Theorem \ref{thm_TS_C_UB}.

\begin{figure}[ht!]
     \centering
     \begin{subfigure}[b]{0.4\textwidth}
         \centering
         \includegraphics[width=\textwidth]{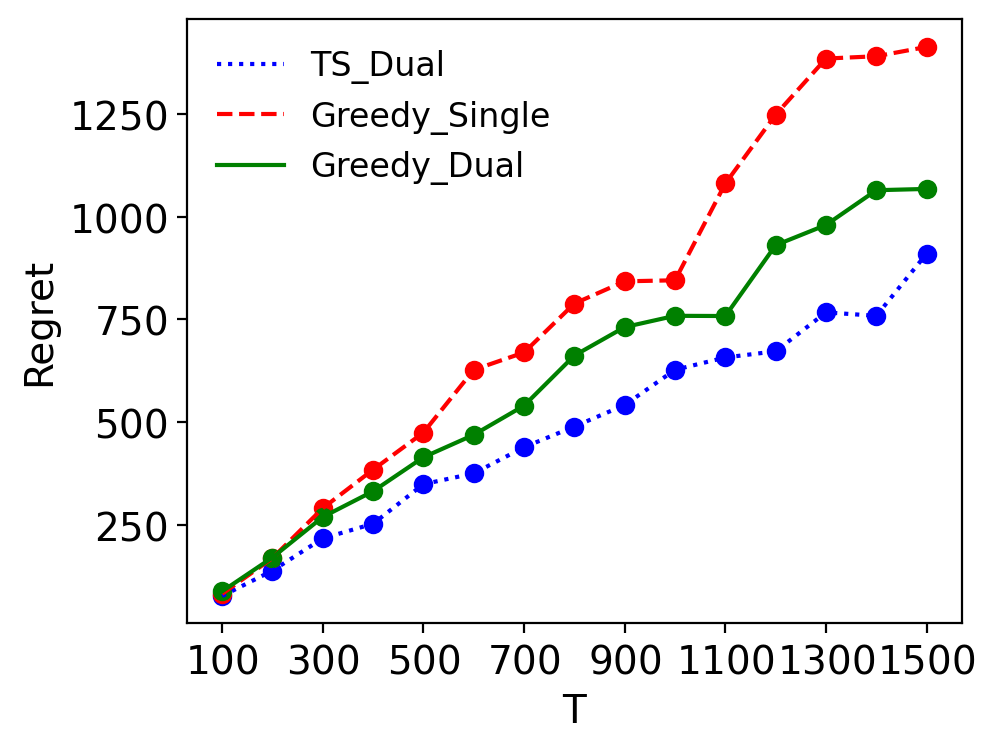}
         \caption{Regret under large $\mathcal{W}_T$.}
         \label{fig:reg_con}
     \end{subfigure}
     \begin{subfigure}[b]{0.4\textwidth}
         \centering
         \includegraphics[width=\textwidth]{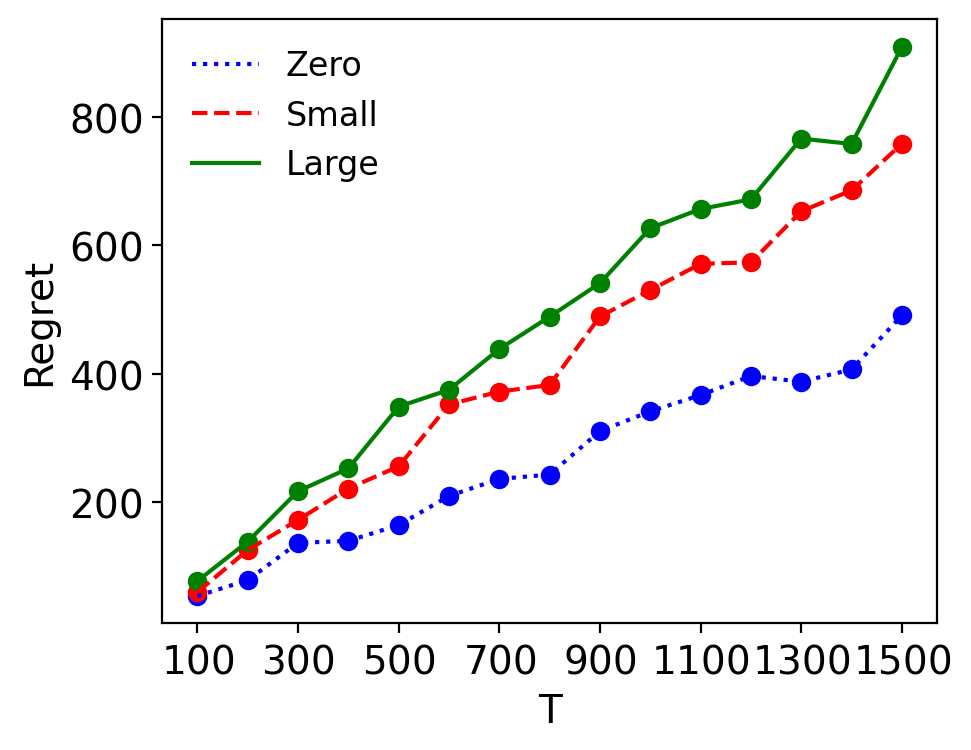}
         \caption{Regret of TS\_Dual when varying $\mathcal{W}_T$.}
         \label{fig:reg_con_var}
     \end{subfigure}
        \caption{Experiment (b): performances of algorithms with inventory constraint.}
        \label{fig:con}
\end{figure}

\section{Conclusion}
\label{sec_conclusion}
In this paper, we consider the dynamic pricing problem under a generalized linear demand model where the demand is dependent on the price and arbitrary covariates. We propose UCB-based and Thompson sampling-based pricing algorithms inspired by bandits literature, which can both achieve the optimal order of regret by leveraging the special structure of the pricing problem. We further consider the pricing problem with inventory constraint and show the Thompson sampling algorithm can be extended to this case.


\bibliographystyle{informs2014} 
\bibliography{bundle1} 


%
%
%
\appendix

\section{Better Regret with Known Price Coefficient}

\label{sec_logT}

In the main paper, all the regret bounds are on the order of $O(\sqrt{T}).$ Here we consider a setting of known price coefficient, i.e., $\beta^*$ in \eqref{fun_demand} is known and show that a $O(\log T)$ regret is achievable under such a setting. This means that the knowledge of the price coefficient is critical in determining the regret order. The algorithm and analysis presented in the following are mainly for two purposes: (i) to identify a difference between the dynamic pricing problem and the bandits problem; (ii) to explain the achievability of $O(\log T)$ regret dependency in the literature. In the rest of this subsection, we denote $p^*((\alpha,\beta),x)=\argmax_{p\in [\underline{p},\bar{p}]}r(p;x^\top\alpha,x^\top\beta)$ as the optimal price under the covariates $x\in\mathcal{X}$ and the parameter $(\alpha,\beta)$.

\begin{assumption}[Known $\beta^*$ and smoothness]
\label{assp_logT_gamma}
Assume $\beta^*$ in \eqref{fun_demand} is known. In addition, assume there exists a constant $C$ such that the optimal expected revenue function satisfies
$$\left|r^*(x^\top \alpha^*,x^\top\beta^*)-r\left(p^*(\theta,x);x^\top \alpha^*,x^\top\beta^*\right)\right|\leq C(x^\top \alpha^*-x^\top \alpha)^2,$$
for all $x\in\mathcal{X}$, $\theta, \theta^*\in\Theta$ with $\theta=(\alpha,\beta^*)$ and $\theta^*=(\alpha^*,\beta^*).$
\end{assumption}

To interpret the condition in the assumption,  the left-hand-side represents the revenue loss caused by using a wrong parameter $\theta$ for pricing, while the right-hand-side is quadratic in terms of the linear estimation error. One sufficient condition for the assumption is that $r(p;x^\top \alpha^*,x^\top\beta^*)$ is continuously twice differentiable with respect to $p$ for all possible $\theta^*$ and $x$, and $p^*(\theta,x)$ is Lipschitz in $x^\top \alpha$. Essentially, this condition does not impose extra restriction upon Assumptions \ref{assp_bound}, \ref{assp_g}, and \ref{assp_error}; in other words, almost all the demand models that satisfy the previous assumptions also meet this condition under the knowledge of $\beta^*$. For example, this condition can be met by the binary demand model \eqref{BDM2} with a log-concave unknown noise \citep{javanmard2019dynamic} and by the linear demand model \eqref{LDM}. It is also analogous to the ``well separation'' condition in the covariate-free case \citep{broder2012dynamic}.

To proceed with the algorithm description, we first slightly revise the MLE estimator in Section \ref{sec_Quasi_MLE} for known $\beta^*$. Specifically, we redefine the misspecified likelihood function for the case of known $\beta^*$ as
\begin{equation*}
\tilde{l}_t(\alpha)\coloneqq -\int_{D_{t}}^{g(x_{t}^\top \alpha+x_t^\top\beta^*\cdot p_t)}\frac{1}{h(u)}(u-D_{t})\mathrm{d}u,
\end{equation*}
where the function $h(u)$ is the same as in Section \ref{sec_Quasi_MLE}. Then the estimator becomes
\begin{equation}
    \hat{\alpha}_t:=\argmax_{\alpha\in \Theta_{\alpha}}
    -\frac{\lambda \underline{g} \|\alpha\|_2^2}{2}+\sum_{\tau=1}^t \tilde{l}_{\tau}(\alpha),
    \label{eq_quasi_reg_logT}
\end{equation}
where $\Theta_{\alpha}$ denotes the subspace $\{\alpha: (\alpha,\beta^*)\in \Theta\}$. Compared to the previous case of unknown $\beta^*$, the only change made here is to plug in the known value of $\beta^*$ and to restrict the attention to estimating the unknown $\alpha^*.$ Accordingly, we revise the definition of the (cumulative) design matrix as  $$\Tilde{M}_t\coloneqq \lambda I_d+\sum_{\tau=1}^t x_{\tau} x_{\tau}^\top$$ with $I_d$ as an identity matrix of dimension $d$.

The following result is parallel to  Corollary \ref{MLEbound}. We omit the proof as it is the same as the previous case of unknown $\gamma^*$ except for some minor notation changes.

\begin{corollary}
\label{MLEbound_logT}
For all $\lambda>0$,
$$\mathbb{P}\left( \exists t\in \{1,...,T\}:  \left\|\hat{\alpha}_t-\alpha^* \right\|_{\tilde{M}_t}\geq 2\sqrt{\lambda}\bar{\theta}+ \frac{2\bar{\sigma}}{\underline{g}}\sqrt{2\log\left(T\right)+d\log\left(\frac{d\lambda +T}{d\lambda}\right)} \right)\leq \frac{1}{T}.$$
\end{corollary}

Algorithm \ref{alg_CE} describes a certainty-equivalent pricing policy. At each time step, the algorithm performs a regularized quasi-MLE to obtain the estimator $\hat{\alpha}_{t-1}$. Then it assumes $\hat{\alpha}_{t-1}$ to be the true parameter and finds the corresponding optimal price.

\begin{algorithm}[ht!]
\caption{Certainty-Equivalent Pricing}
\label{alg_CE}
\begin{algorithmic}
\STATE{\textbf{Input:} Regularization parameter $\lambda$.}
\FOR{$t=1,...,T$}
\STATE{Compute the estimator $\hat{\alpha}_{t-1}$ by \eqref{eq_quasi_reg_logT}, observe feature $x_t$ and set the price by
\begin{equation*}
    p_{t}=p^*\left(\hat{\theta}_{t-1}, x_t\right)
\end{equation*}
where $\hat{\theta}_{t-1}=(\hat{\alpha}_{t-1},\beta^*)$}.
\ENDFOR
\end{algorithmic}
\end{algorithm}

\begin{theorem}
Under Assumptions \ref{assp_bound}, \ref{assp_g}, \ref{assp_error}, \ref{assp_logT_gamma} and with any sequence $\{x_t\}_{t=1,...,T}$, if we choose the regularization parameter $\lambda=1$, the regret of Algorithm \ref{alg_CE} is upper bounded by
$$2C\bar{\gamma}^{'2}d\log\left(\frac{d+T}{d}\right)+\bar{p}\Bar{D} = \tilde{O}(d^2\log^2 T)$$
where $\bar{\gamma}^{'}=2\bar{\theta}+ \frac{2\bar{\sigma}}{\underline{g}}\sqrt{2\log T+d\log\left(\frac{d+T}{d}\right)}$ and $C$ is defined in Assumption \ref{assp_logT_gamma}.
\label{thm_regret_bound_LogT}
\end{theorem}

Theorem \ref{thm_regret_bound_LogT} provides a regret upper bound for Algorithm \ref{alg_CE}. We remark that it is unnecessary for Algorithm \ref{alg_CE} to compute the estimator at every time step. \cite{javanmard2019dynamic} solve an L$_1$ regularized linear regression on geometric time intervals, and the scheme can also be applied to Algorithm \ref{alg_CE} with the same order of regret bound. The frequent or infrequent estimation scheme makes no analytical difference and the choice mainly accounts for computation consideration. \cite{xu2021logarithmic} study the special case of the binary demand model with unit price coefficient (i.e., known $\beta^*$), and they derive the same order of regret bound as Theorem \ref{thm_regret_bound_LogT} under arbitrary covariates using online Newton's method. The intuition is that the convergence rate of Newton's method is on the same order with the MLE estimator, so the corresponding output can be viewed as an approximate MLE estimator at each time step, and the approximation will not deteriorate the regret performance. This is aligned with Theorem \ref{thm:UCB_UB_2} and Theorem \ref{thm_TS_UB_appx} where an approximate quasi-MLE estimator is used.

In general, many existing the $o(\sqrt{T})$ regret bounds \citep{broder2012dynamic, javanmard2017perishability, javanmard2019dynamic, xu2021logarithmic} fall into this paradigm of
$$\text{known price coefficient} \ + \ \text{certainty-equivalent policy}.$$
Intuitively, when the price coefficient is known, the price $p_t$ will not interfere the learning of $\alpha^*.$ Thus there is no need to do price exploration like UCB or TS, and the regret purely reflects the cumulative learning rate of $\alpha^*.$ This disentanglement of pricing decisions from parameter estimation makes the setting of known $\beta^*$ analogous to the ``full information'' setting in online learning literature. In contrast, when the price coefficient is unknown, the pricing decisions will affect the learning rate of $\beta^*$, thus the setting of unknown $\beta^*$ is more aligned with the ``partial information'' setting such as the bandits problem.

\paragraph{Interpreting the result under a linear demand model.}\

We use a linear demand model to further illustrate the contrast between $O(\sqrt{T})$ and $O(\log T)$ regret dependency. Consider the demand follows
$$D_t = a^* + b^* p_t +\epsilon_t$$
for some $a^*>0$ and $b^*<0.$ At time $t$, the seller sets the price by $p_t=-\frac{\hat{a}_t}{2\hat{b}_t}$ for some estimators $\hat{a}_t$ and $\hat{b}_t$ which could be from either Algorithm \ref{alg_UCB} (optimistic estimators) or Algorithm \ref{alg_CE} (CE estimators).
Then the single step regret can be expressed by a function of the true parameters and the estimators,
\begin{equation}
 \text{Reg}_t = r^*(a^*,b^*) - r(p_t;a^*,b^*) = -2b^*\left(\frac{\hat{a}_t}{2\hat{b}_t}-\frac{a^*}{2b^*}\right)^2  = -\frac{(a^*\hat{b}_t-\hat{a}_tb^*)^2}{2b^*\hat{b}_t^2}
 \label{eqn_explain}
\end{equation}
\begin{itemize}
    \item When the price coefficient is known, $\hat{b}_t = b^*$. The equality becomes
     $$\text{Reg}_t = -\frac{1}{2b^*} (\hat{a}_t-a^*)^2.$$
    \item When the price coefficient is unknown, we cannot do more than a first-order Talyor expansion when we want to upper bound Reg$_t$ by the estimation error, i.e.,
    \begin{equation}
        \text{Reg}_t \le  c(|\hat{a}_t-a^*|+|\hat{b}_t-b^*|)
        \label{eqn_explain_2}
    \end{equation}
     for some $c>0.$
\end{itemize}
For this example, whether the price coefficient $b^*$ is known determines the space in which we view the right-hand-side of \eqref{eqn_explain} as a function of $\hat{a}_t$ and $\hat{b}_t$. Intuitively, suppose that the estimation error is on the order of $\sqrt{1/t}$ (the intuition is precise when the covariates are i.i.d.). Then the right-hand-side will recover two different regret bounds under the two settings.

Generally, we remark that the first-order bound like \eqref{eqn_explain_2} under a proper norm is always the first step for the analysis of UCB and TS algorithms, including our analysis for the dynamic pricing problem. For linear bandits problem, the LinUCB algorithm \citep{chu2011contextual,abbasi2011improved} can directly obtain this first-order bound and under TS algorithms, it can be obtained by some Bayesian arguments \citep{russo2014learning2} or by maintaining a constant probability of choosing an optimistic action with anti-concentration sampling \citep{abeille2017linear}.

\paragraph{Proof of Theorem \ref{thm_regret_bound_LogT}.}\

\begin{proof}
Denote $\bar{\gamma}^{'}=2\bar{\theta}+ \frac{2\bar{\sigma}}{\underline{g}}\sqrt{2\log T+d\log\left(\frac{d+T}{d}\right)}$. Revise the definition of the ``good event'' as
$$\tilde{\mathcal{E}}:=\left\{ \left\|\hat{\alpha}_t-\alpha^*\right\|_{\tilde{M}_t}\leq \bar{\gamma}^{'} \text{ for } t=0,...,T-1\right\},$$
From Corollary \ref{MLEbound_logT}, we know
$$\mathbb{P}(\mathcal{\tilde{E}}) \ge 1 -1/T.$$

Under the event $\mathcal{\tilde{E}},$ the single period regret can be bounded by
\begin{align*}
   \text{Reg}_t&= r^*(x_t^\top \alpha^*,x_t^\top \beta^*)-r(p_t;x_t^\top \alpha^*,x_t^\top \beta^*)\\
    &\leq C\left|x_t^\top \alpha^*-x_t^\top \hat{\alpha}_{t-1}\right|^2\\
    &\leq C\|x_t\|^2_{\tilde{M}^{-1}_{t-1}}\|\hat{\alpha}_{t-1}-\alpha^*\|_{\tilde{M}_{t-1}}^2\\
    &\leq C\bar{\gamma}^{'2}\|x_t\|_{\tilde{M}^{-1}_{t-1}}^2.
\end{align*}
Here the second line is from Assumption \ref{assp_logT_gamma}, the third line is from  Holder's inequality, and the last inequality is by the event $\tilde{\mathcal{E}}$.

Thus, the total expected regret (the expectation is with respect to the randomness of demand shocks) can be bounded by
\begin{align*}
    \text{Reg}_T^{\pi_\text{CE}}(x_1,...,x_T)&=\sum_{t=1}^T \mathbb{E}\left[\text{Reg}_t\cdot \mathbbm{1}_{\tilde{\mathcal{E}}}\right]+\mathbb{E}\left[\text{Reg}_t\cdot \mathbbm{1}_{\tilde{\mathcal{E}}^{c}}\right]\\
    &\leq \sum_{t=1}^T C\bar{\gamma}^{'2}\|x_t\|_{\tilde{M}^{-1}_{t-1}}^2+\mathbb{P}\left(\tilde{\mathcal{E}}^{c}\right)\cdot \bar{p}T\bar{D}\\
    &\leq 2C\bar{\gamma}^{'2}d\log\left(\frac{d+T}{d}\right)+\bar{p}\Bar{D}
\end{align*}
where $\pi_\text{CE}$ denotes Algorithm \ref{alg_CE} and $\tilde{\mathcal{E}}^c$ denotes the complement of the event $\tilde{\mathcal{E}}$. The last inequality is because of Lemma \ref{EPL_lemma}.
\end{proof}

\section{More Discussions on the Existing Literature}
\label{sec_modeling_issue}

Here we discuss some modeling issues and assumptions in the existing dynamic pricing literature and also show how the general demand model \eqref{fun_demand} captures the binary demand model and the linear demand model.

\subsection{Coefficient of Price and the Distribution of the Random Shock}
We first point out an issue with the assumptions on the stochastic quasi-linear demand model or the so-called \textit{binary demand} model.
The binary demand model is stated as follows. At time $t$, the demand is
\begin{equation}
 D_{t} = \begin{cases}
  1,& \text{if} \ \ x_t^\top \alpha^* +\zeta_t \ge p_t \\
  0,& \text{if} \ \ x_t^\top \alpha^* +\zeta_t < p_t
\end{cases} \label{BDM1}
\end{equation}
where $x_t\in \mathbb{R}^d$ denotes the covariate vector, $\alpha^*\in \mathbb{R}^d$ is the linear coefficient vector, $p_t$ is the price, and $\zeta_t$ models an unobserved utility shock. Under this model, $x_t^\top \alpha^* +\zeta_t$ represents the customer's utility where the term $x_t^\top \alpha^*$ captures the part of the utility explained by the covariates $x_t.$

An assumption often made is (i) the coefficient of $p_t$ is fixed at 1 and (ii) $\zeta_t$'s are i.i.d. and follow a distribution with known parameters. The vector $\alpha^*$ then is assumed unknown that we wish to learn from observations over time.  This is roughly the setting used in \citet{javanmard2017perishability, cohen2020feature, xu2021logarithmic} and part of \citet{javanmard2019dynamic}.

There are two issues in general with this setting:
\begin{enumerate}
\item Estimation of the parameters of the distribution generating $\zeta_t$ cannot be separated from the estimation of the $\alpha^*$.  So it raises the question of how one arrives at this knowledge of the parameters of the distribution of $\zeta_t$'s.
\item Somewhat related to the above point --  say the seller starts collecting new data covariates $y_t\in\mathbb{R}^m$ in addition to $x_t$.   The additional covariates $y_t$ can potentially improve the prediction of customer utility and thus reduce the variance of the demand shock $\zeta_t$. Assuming the variance is known will then be problematic.
\end{enumerate}

\subsubsection{Parametric Distribution of the Error Term}

\label{sec_pe_2}
Assume the distribution of $\zeta_t$ belongs to a parametric family of distributions with some unknown parameter(s).  Say $\zeta_t$ has zero mean and is scale-invariant with variance $\sigma^{-2}$ for some unknown $\sigma>0$.

Scaling appropriately, model \eqref{BDM1} is then equivalent to the following binary demand model where the distribution of $\zeta_t$ is known but the price coefficient $\sigma$ is unknown,
\begin{equation}
D_{t} = \begin{cases}
  1,& \text{if} \ \ x_t^\top \alpha^* +\tilde{\zeta}_t \ge \sigma p_t \\
  0,& \text{if} \ \ x_t^\top \alpha^* +\tilde{\zeta}_t < \sigma p_t
\end{cases}
\label{BDM2}
\end{equation}
where $\tilde{\zeta}_t$ follows some known distribution with mean zero and variance one.

%
So in \eqref{BDM2} we can assume either (i) known price coefficient (say fixed to 1) and unknown variance, or (ii) unknown price coefficient and known variance (say 1).   We note that it is unnecessary to assume both parts unknown because the thresholding conditions in \eqref{BDM1} and \eqref{BDM2} are scale-invariant.

\begin{remark}
In the literature, \cite{javanmard2019dynamic} and \cite{luo2021distribution} consider the model \eqref{BDM2}. Specifically, \cite{javanmard2019dynamic} focus on the high-dimensional setting and impose the i.i.d. assumption on the covariates; \cite{luo2021distribution} allow a non-parametric structure for the shock distribution (more general than \eqref{BDM2}) and achieve an $\tilde{O}(dT^{2/3})$ regret. \cite{xu2021logarithmic} also mention the model \eqref{BDM2} and leave the achievability of $\sqrt{T}$ regret as an open question. The generalized linear model replaces the price coefficient $\sigma$ in \eqref{BDM2} with a linear function of the covariates $x_t$. Thus it can be viewed as a further generalization of the model \eqref{BDM2}. The rationale is that the extent to which the covariates $x_t$ explain the customer utility can be dependent not only on some unknown constant $\sigma$ but also on $x_t$, and the dependency is unknown. Our result resolves the open question in \citet{xu2021logarithmic}.
\end{remark}

\paragraph{Linear demand model.}

Another prevalent demand model (\cite{ den2014simultaneously, keskin2014dynamic,keskin2018incomplete,qiang2016dynamic, javanmard2019dynamic, ban2021personalized, bastani2021meta}) is the \textit{linear demand model}, where the demand is
\begin{equation}
 D_t = x_t^\top \alpha^*+x_t^\top \beta^* \cdot p_t+\epsilon_t   \label{LDM}
\end{equation}
and $\alpha^*$ and $\beta^*$ are unknown parameter vectors. The demand shock $\epsilon_t$'s are mean-zero random variables, and it makes no essential difference to adapt the distribution of $\epsilon_t$ to the history up to time $t$.

Under the linear demand model, the distribution of $\epsilon_t$ makes no difference to the optimal pricing strategy as long as it is mean-zero and sub-Gaussian. This explains why papers on the linear demand model, unlike the case of the binary demand model's unobserved utility shock $\zeta$, do not assume knowledge of the distribution of $\epsilon_t$.

\subsection{Generalized Linear Demand Model Class}

Now we show how the general model \eqref{fun_demand} recovers the binary demand model \eqref{BDM2} and the linear demand model \eqref{LDM} as special cases.

\begin{example}[Binary Demand Model]
The binary demand model \eqref{BDM2} (also \eqref{BDM1}) is considered by a number of papers as in Table \ref{Tab:table_bound} (denoted by ``binary model'' in the ``Demand Model'' column). We first restrict the first dimension of $x_t$ to be always $1$. To recover \eqref{BDM2}, we can then set the function $g(\cdot)$ in \eqref{fun_demand} to be the cumulative distribution function of $-\tilde{\zeta}_t$, $\beta^*=(-\sigma,0,\ldots,0)^\top$, and
\begin{equation*}
 \epsilon_t =
\begin{cases}
      1-g\left(x_t^\top\alpha^*-\sigma p_t\right) & \text{ w.p. $g\left(x_t^\top\alpha^*-\sigma p_t\right)$},\\
    -g\left(x_t^\top\alpha^*-\sigma p_t\right) & \text{w.p. $1-g\left(x_t^\top\alpha^*-\sigma p_t\right)$}.
    \end{cases}
\end{equation*}
Thus it becomes the binary demand model \eqref{BDM2}.  The parameter $\sigma$ is an unknown parameter that represents the price coefficient or describes the variance of the utility shock in \eqref{BDM2}; the sub-Gaussian parameter $\bar{\sigma}^2$ can be easily chosen as $1/4$ from the boundedness of $\epsilon_t.$
\label{eg1}
\end{example}

\begin{example}[Linear Demand Model]
\label{eg2}
The linear demand model \eqref{LDM} can be easily recovered from \eqref{fun_demand} by letting the function $g(\cdot)$ be an identity function. As in Example \ref{eg1}, we restrict the first dimension of $x_t$ to be always $1$. Then, the model in \citet{qiang2016dynamic} can be recovered by setting $\beta^*=(b,0,\ldots,0)$; the covariate-free linear demand model in \citet{den2014simultaneously, keskin2014dynamic} can be recovered by setting $\alpha^*=(a,0,\ldots,0)$ and $\beta^*=(b,0,\ldots,0)$.
\end{example}

\subsection{Practical Motivation for No Assumptions on the Covariate Distributions}
\label{sec_non_iid}
The covariates are typically customer information or features of the environment and are almost always exogenous and not in the control of the firm.   Making assumptions on their distribution is restrictive for the following reasons:
\begin{itemize}
    \item Network/peer effects: A wide range of products are influenced by network or peer effects \citep{seiler2017does,goolsbee2002evidence,bailey2019peer, nasr2018continuous}.   \cite{baardman2020detecting} show that demand prediction can be improved by incorporating such effects.  Thus, the customers' features are more likely to exhibit short-term dependencies.
    \item Seasonality and life-cycle of product:  Seasonality, day-of-week and time-of-day patterns create serial correlation in the covariates \citep{neale2009managing}. Product life-cycle effects (for tech or fashion items) also come into play, so the distribution of the covariates changes over the life-cycle of the product as the customer segment mix may be different at different stages of the product life-cycle.
    \item Competitors: Competing products and the action of the competitors influence demand \citep{armstrong2005demand}.  However, the effect from competitors' actions is complicated and the covariates of the environment cannot be identically distributed.
\end{itemize}

\section{Appendix for Proofs}
Throughout the following proofs, we denote
$$\bar{\gamma}\coloneqq 2\bar{\theta}+ \frac{2\bar{\sigma}}{\underline{g}}\sqrt{2\log T+2d\log\left(\frac{2d+T}{2d}\right)}.$$
\subsection{Proofs for Section \ref{sec_UCB}}
\label{sec_apend_UCB}
\subsubsection{Proofs for Subsection \ref{sec_Quasi_MLE}}
We first introduce some notations and ancillary lemmas. First, the gradient and Hessian of $l_t$ are
\begin{equation}
\nabla l_t(\theta)=\frac{g'(z_t^\top\theta)\cdot z_t}{h(g(z_t^\top \theta))}\left(D_t-g(z_t^\top \theta) \right)= \xi_{t}(\theta)z_{t}
\end{equation}
\begin{equation} \quad \nabla^2 l_t (\theta)=-\eta_{t}(\theta)z_{t}z_{t}^\top,
\end{equation}
where $\xi_{t}(\theta)\coloneqq D_{t}-g\left(z_t^\top \theta\right)$ and
$\eta_{t}(\theta)\coloneqq g'\left(z_t^\top \theta \right)$.  The concise form of the gradient and Hessian justifies the choice of $h(u)$ in \eqref{eq_quasi_self}.

The following lemma states that under a non-anticipatory pricing policy/algorithm, the sequence of $\{\xi_t(\theta^*)\}_{t=1}^{T}$ is a martingale difference sequence adapted to history observations with zero-mean $\bar{\sigma}^2$-sub-Gaussian increments.

\begin{lemma}
\label{MD_seq}
For $t=1,\ldots,T$, we have
$$\mathbb{E}\left[\xi_{t}(\theta^*)|\mathcal{H}_{t-1}\right]=0.$$
In addition, $\xi_{t}(\theta^*)|\mathcal{H}_{t-1}$ is $\bar{\sigma}^2$-sub-Gaussian.
\end{lemma}
\begin{proof}
Note that
$$\mathbb{E}\left[\xi_{t}(\theta^*)|\mathcal{H}_{t-1}\right]=\mathbb{E}\left[\epsilon_t|\mathcal{H}_{t-1}\right]=0.$$
Both the last part and the sub-Gaussianity come from Assumption \ref{assp_error}.
\end{proof}
Let the (cumulative) score function
$$S_{t}\coloneqq \sum_{t'=1}^{t}\frac{\xi_{t'}(\theta^*)}{\bar{\sigma}}z_{t'}.$$

The following theorem measures $S_t$'s deviation in terms of the metric induced by $M_t.$ It can be easily proved by an application of the martingale maximal inequality on the sequence of $\xi_t(\theta^*).$

\begin{theorem}[Theorem 20.4, \cite{lattimore2020bandit}]
\label{Lattimore}
\label{thm_first_order}
For any regularization parameter $\lambda>0$ and $\delta\in(0,1)$,

$$\mathbb{P}\left( \exists t\in \{1\ldots,T\}:  \left\|S_{t}\right\|^2_{M_{t}^{-1}}\geq 2\log\left(\frac{1}{\delta}\right)+\log\left(\frac{\det M_{t}}{\lambda^{2d}}\right)\right)\leq \delta. $$
\end{theorem}

Now we are ready to prove Proposition \ref{prop_theta_est}:
\paragraph{Proof of Proposition \ref{prop_theta_est}.}\

\begin{proof}
We perform a second-order Taylor's expansion for the objective function of regularized quasi-MLE \eqref{eq_quasi_reg} around the true parameter $\theta^*$. Let
$$Q_t(\theta) \coloneqq \sum_{t'=1}^t l_{t'}(\theta).$$
We have
\begin{align}
    \label{Taylor}
    &Q_t(\theta^*)-\frac{\lambda\underline{g}\|\theta^*\|^2_2}{2}-Q_t(\theta)+\frac{\lambda\underline{g}\|\theta\|^2_2}{2}\\
    =&-\left\langle \nabla Q_t(\theta^*)-\lambda\underline{g}\theta^*,\theta-\theta^*\right\rangle-\frac{1}{2}\left\langle\theta-\theta^*,\left(\nabla^2 Q_t(\theta')-\lambda \underline{g} I_{2d}\right)(\theta-\theta^*)\right\rangle \nonumber
\end{align}
for some $\theta'$ on the line segment between $\theta$ and $\theta^*$.

By the optimality of $\hat{\theta}_t$,
$$Q_t(\theta^*)-\frac{\lambda \underline{g}\|\theta^*\|^2_2}{2}\leq Q_t (\hat{\theta}_t)-\frac{\lambda \underline{g} \|\hat{\theta}_t\|^2_2}{2}.$$

Then from (\ref{Taylor}), we have
\begin{equation}
\label{eq_Taylor_2}
\left\langle\nabla Q_t(\theta^*)-\lambda  \underline{g} \theta^*,\hat{\theta}_t-\theta^*\right\rangle+\frac{1}{2}\left\langle\hat{\theta}_t-\theta^*,\left(\nabla^2 Q_t(\tilde{\theta}')-\lambda  \underline{g} I_{2d}\right)(\hat{\theta}_t-\theta^*)\right\rangle \geq0,
\end{equation}
for some $\tilde{\theta}'$ on the line segment between $\theta$ and $\theta^*$.

From Assumption $\ref{assp_g}$,
\begin{equation*}
    -\nabla^2 Q_t(\tilde{\theta})= \sum_{t'=1}^t g'\left( z_t^\top \tilde{\theta} \right) z_{t'}z_{t'}^\top \geq \underline{g}\cdot \sum_{t'=1}^{t}z_{t'}z_{t'}^\top \quad \forall \tilde{\theta} \in \Theta.
\end{equation*}

Further, with Holder's inequality and \eqref{eq_Taylor_2},
\begin{align*}
    \left\|\nabla Q_t(\theta^*) -\lambda \underline{g} \theta^*\right \|_{M_t^{-1}}\left\|\hat{\theta}_t-\theta^*\right\|_{M_t} \geq& \left\langle \nabla Q_t(\theta^*)-\lambda \underline{g} \theta^*,\hat{\theta}_t-\theta^* \right\rangle\\
    \geq &\frac{1}{2}\left\langle\hat{\theta}_t-\theta^*,\left(-\nabla^2 Q_t(\theta')+\lambda \underline{g} I_{2d}\right)(\hat{\theta}_t-\theta^*)\right\rangle \\
    \geq & \frac{1}{2}\underline{g}\left\langle\hat{\theta}_t-\theta^*,M_t (\hat{\theta}_t-\theta^*)\right\rangle\\
    = & \frac{1}{2}\underline{g} \left\|\hat{\theta}_t-\theta^*\right\|_{M_t}^2
\end{align*}
almost surely. Consequently,
\begin{equation}
    \left\|\nabla Q_t(\theta^*) -\lambda \underline{g} \theta^*\right \|_{M_t^{-1}}\geq  \frac{1}{2}\underline{g} \left\|\hat{\theta}_t-\theta^*\right\|_{M_t} \quad a.s.
    \label{tmp_taylor}
\end{equation}

Recall that
$$S_{t}=\sum_{t'=1}^{t}\frac{\xi_{t'}(\theta^*)}{\bar{\sigma}}z_{t'}=\frac{1}{\bar{\sigma}}\nabla Q_t(\theta^*),$$
which implies
\begin{align*}
    \frac{1}{2\bar{\sigma}}\underline{g} \left\|\hat{\theta}_t-\theta^*\right\|_{M_t}
    \leq& \frac{1}{\bar{\sigma}}\left\|-\nabla Q_t(\theta^*)+\lambda \underline{g}\theta^*\right  \|_{M_t^{-1}} \\
    = &\frac{1}{\bar{\sigma}}\left\|-\bar{\sigma} S_t+\lambda \underline{g} \theta^*\right \|_{M_t^{-1}}\\
    \leq& \left\|S_t\right\|_{M^{-1}_{t}}+\frac{\sqrt{\lambda}\underline{g}}{\bar{\sigma}}\sqrt{(\theta^*)^\top (\lambda M_{t}^{-1})\theta^*}\\
    \leq& \left\|S_t\right\|_{M^{-1}_{t}}+\frac{\sqrt{\lambda}\underline{g}}{\bar{\sigma}}\|\theta^*\|_2.
\end{align*}
Here the first line comes from \eqref{tmp_taylor}, the second line comes from the definition of $S_t$, the third lines comes from the norm inequality, and the last line is from the fact that $\lambda M_{t}^{-1} \le I_{2d}$. Thus, we complete the proof from combining Theorem \ref{Lattimore} with $\|\theta^*\|_2\leq \bar{\theta}$.

\end{proof}

\paragraph{Proof of Corollary \ref{MLEbound}.}\
\begin{proof}
Given that $\|z_{t}\|_2^2\leq 1$ by assumption, we can apply Lemma 19.4 of \cite{lattimore2020bandit} (purely algebraic analysis with no stochasticity) with $\delta=\frac{1}{T}$ in Proposition \ref{prop_theta_est} to obtain the result.
\end{proof}
\subsubsection{Proofs for Subsection \ref{sec_UCB_alg} }

For Theorem \ref{thm_regret_bound}, the proof idea is very intuitive. As the algorithm represents the confidence set based on the matrix $M_{t-1}$, the current observation $x_t$ will either induce a small single step regret or reduce the confidence set significantly. Then we upper bound the regret of the algorithm to a summation sequence involving the covariates $x_t$'s and matrices $M_t$'s and then employ the elliptical potential lemma (as follows) to conclude the proof.

\begin{lemma}[Elliptical Potential Lemma, \citep{lai1982least}]
\label{EPL_lemma}
For any constant $\lambda\geq 1$ and sequence of $\{x_t\}_{t\geq 1}$ with $\|x_t\|_2\leq 1$ for all $t\geq 1$ and $x_t\in \mathbb{R}^d$, define the sequence of covariance matrices:
$$\Sigma_0:=\lambda I_d, \quad \Sigma_t:=\lambda I_d+\sum_{t'=1}^tx_{t'}x_{t'}^\top \ \ \forall t\geq 1,$$
where $I_d$ is the identity matrix with dimension $d$. Then for any $T\geq 1$, the following inequality holds
$$\sum_{t=1}^T\|x_{t}\|^2_{\Sigma_{t-1}^{-1}}\leq 2d\log\left(\frac{\lambda d+T}{\lambda d}\right).$$
\end{lemma}
\paragraph{Proof of Theorem \ref{thm_regret_bound}.}\

\begin{proof}
We define the ``good event'' as
$\mathcal{E}=\left\{\theta^*\in \Theta_{t} \text{ for } t=1,\ldots,T\right\}.$
From Corollary \ref{MLEbound}, we know
\begin{equation}
    \label{bad_prob}
    \mathbb{P}(\mathcal{E}) \ge 1 -1/T.
\end{equation}

At time $t$, under the event $\mathcal{E}$, the choice of $\theta_t=(\alpha_t,\beta_t)$ in Algorithm \ref{alg_UCB} ensures
\begin{equation}
\label{UCB_pricing}
    r^*(x_t^\top\alpha_t,x_t^\top\beta_t)\geq r^*(x_t^\top\alpha^*,x_t^\top\beta^*).
\end{equation}

Thus, under the event $\mathcal{E},$ the single period regret can be bounded by
\begin{align*}
   \text{Reg}_t \coloneqq r^*(x_t^\top\alpha^*,x_t^\top\beta^*)-r(p_t;x_t^\top\alpha^*,x_t^\top\beta^*)
    &\leq r^*(x_t^\top\alpha_t,x_t^\top\beta_t)-r(p_t;x_t^\top\alpha^*,x_t^\top\beta^*)\\
    &=p_t\cdot \left(g\left(x_t^\top \alpha_{t}+x_t^\top \beta_{t}\cdot p_t\right) -g\left(x_t^\top \alpha^*+x_t^\top \beta^*\cdot p_t\right)\right)\\
    &\leq \bar{g}\bar{p}\left| z_t^\top(\theta_t-\theta^*) \right|\\
    &\leq \bar{g}\bar{p}\|z_t\|_{M_{t-1}^{-1}}\|\theta_t-\theta^*\|_{M_{t-1}}\\
    &\leq \bar{g}\bar{p}\|z_t\|_{M_{t-1}^{-1}}\left(\|\theta_t-\hat{\theta}_{t-1}\|_{M_{t-1}}+\|\hat{\theta}_{t-1}-\theta^*\|_{M_{t-1}}\right)\\
    &\leq 2\bar{g}\bar{p}\bar{\gamma}\|z_t\|_{M_{t-1}^{-1}}
\end{align*}
where the functions $r^*$ and $r$ are introduced in the Section \ref{sec_Model}. Here the first line is from \eqref{UCB_pricing}, the third line is from Assumption \ref{assp_bound}, the fourth line is from Holder's inequality, and the last inequality is by Corollary
\ref{MLEbound} under the event $\mathcal{E}$.

Thus, the total expected regret (the expectation is with respect to the randomness of demand shocks) can be bounded by
\begin{align*}
   \text{Reg}_T^{\text{UCB}}(x_1,\ldots,x_T)&=\sum_{t=1}^T \mathbb{E}\left[\text{Reg}_t\cdot \mathbbm{1}_{\mathcal{E}}\right]+\mathbb{E}\left[\text{Reg}_t\cdot \mathbbm{1}_{\mathcal{E}^c}\right]\\
    &\leq 2\sum_{t=1}^T\bar{g}\bar{p}\bar{\gamma}\|z_t\|_{M_{t-1}^{-1}}+\mathbb{P}\left(\mathcal{E}^c\right)\cdot \bar{p}\bar{D}T\\
    &\leq 2\bar{g}\bar{p}\bar{\gamma}\sqrt{T\sum_{t=1}^T\|z_t\|^2_{M_{t-1}^{-1}}}+\mathbb{P}\left(\mathcal{E}^c\right)\cdot \bar{p}\bar{D}T\\
    &\leq 4\bar{g}\bar{p}\bar{\gamma}\sqrt{Td\log\left(\frac{2d\lambda+T}{2d\lambda}\right)}+\bar{p}\bar{D}
\end{align*}
where UCB denotes the pricing policy specified by Algorithm \ref{alg_UCB} and $\mathcal{E}^c$ denotes the complement of the event $\mathcal{E}.$ The second inequality is by Holder's inequality and the last inequality is because \eqref{bad_prob} and Lemma \ref{EPL_lemma}.
\end{proof}

\paragraph{Proof of Theorem \ref{thm_UCB_MC_approx}.}\
Note by definitions, the single period regret at $t$ can be bounded by:

\begin{align*}
    \text{Reg}_t &\coloneqq r^*(x_t^\top\alpha^*,x_t^\top\beta^*)-r(p_t;x_t^\top\alpha^*,x_t^\top\beta^*)\\
&= r^*(x_t^\top\alpha^*,x_t^\top\beta^*)-r^*(x_t^\top\alpha_t,x_t^\top\beta_t)+r^*(x_t^\top\alpha_t,x_t^\top\beta_t)-r(p_t;x_t^\top\alpha^*,x_t^\top\beta^*)\\
&= r^*(x_t^\top\alpha^{\text{MC}}_t,x_t^\top\beta^{\text{MC}}_t)-r(p_t;x_t^\top\alpha^*,x_t^\top\beta^*)+ r^*(x_t^\top\alpha^*,x_t^\top\beta^*)-r^*(x_t^\top\alpha_t,x_t^\top\beta_t)\\
&+r^*(x_t^\top\alpha_t,x_t^\top\beta_t)-r^*(x_t^\top\alpha^{\text{MC}}_t,x_t^\top\beta^{\text{MC}}_t).
\end{align*}
Then following the identical proof idea of Theorem \ref{thm_regret_bound} (noting $r^*(x_t^\top\alpha^*,x_t^\top\beta^*)-r^*(x_t^\top\alpha_t,x_t^\top\beta_t)<0$ under the good event $\mathcal{E}$), we can conclude the result.
\subsection{Proof for Section \ref{sec_TS}}
\label{sec_appn_TS}

For the original Thompson Sampling, its analysis in \cite{abeille2017linear} needs its key Lemma 3. However, Algorithm \ref{alg_TS} can not directly apply this lemma: (1) the sampling step in \eqref{eq_TS_samplestep} is based on $(\hat{a}_t,\hat{b}_t)$ and $\Tilde{M}_{t-1}$ but we want it to be still conditional on the event $\theta^*=(\alpha^*,\beta^*)\in \Theta_t$ due to Corollary \ref{MLEbound}; (2) the reward function is non-convex, and not Lipschitz in the estimators even in the linear demand model (so our reward function can not satisfy the setting of convex optimization problems with linear observation as discussed in Section 6 or the Assumption 4 in \citet{abeille2017linear}). In summary, we want to lower bound the probability $$\mathbb{P}\left(r^*(\tilde{a}_t,\tilde{b}_t)\geq r^*(a^*_t,b^*_t) \bigg \vert \mathcal{H}_{t-1}, \theta^*\in \Theta_{t}\right),$$
where $r^*(a,b)$ can be non-convex and non-Lipschitz in $(a,b)$.

To solve the above challenge, we will first show  $(\tilde{a}_t,\tilde{b}_t)$ has a distribution matched with the sample $(x_t^\top \Tilde{\alpha}_{t-1}, x_t^\top \Tilde{\beta}_{t-1})$ conditional on $\mathcal{H}_{t-1}$ in Lemma \ref{lemma_same_dis}, where
$$(\Tilde{\alpha}_{t-1},\Tilde{\beta}_{t-1}) \coloneqq \hat{\theta}_{t-1}+\bar{\gamma}M_{t-1}^{-1/2}\eta_t$$
is defined as the original Thompson sampling parameters  (see Appendix~\ref{apx:TS_vs_newTS} for details), to connect with the Lemma 3 in \cite{abeille2017linear} and get Lemma \ref{Abeille_opt_prob}. Then we utilize the special structure of pricing problem shown in Lemma \ref{lem:linear2reward} to finally conclude Lemma \ref{lemma_opt_samp}:

\begin{lemma}
    $ (x_t^\top \Tilde{\alpha}_{t-1}, x_t^\top \Tilde{\beta}_{t-1})|\mathcal{H}_{t-1} \stackrel{d}{=}(\tilde{a}_t,\tilde{b}_t)|\mathcal{H}_{t-1}$, where $\stackrel{d}{=}$ means equal in distribution.
    \label{lemma_same_dis}
\end{lemma}

\begin{proof}
        Let $\eta'\sim \mathcal{N}(0,I_{2d})$, we have
        \begin{align*}
        (x_t^\top \Tilde{\alpha}_{t-1}, x_t^\top \Tilde{\beta}_{t-1})|\mathcal{H}_{t-1}&=\begin{pmatrix}
x_t&\bm{0}\\
\bm{0}&x_t
\end{pmatrix}^\top \hat{\theta}_{t-1}+\bar{\gamma} \begin{pmatrix}
x_t&\bm{0}\\
\bm{0}&x_t
\end{pmatrix}^\top M_{t-1}^{-1/2}\eta'\\
&\stackrel{d}{=} (\hat{a}_t,\hat{b}_t)+\bar{\gamma} \mathcal{N}\left(0, \tilde{M}_{t-1}^{-1}\right)\\
&\stackrel{d}{=} (\hat{a}_t,\hat{b}_t)+\bar{\gamma}\tilde{M}_{t-1}^{-1/2}\eta\\
&=(\tilde{a}_t,\tilde{b}_t)|\mathcal{H}_{t-1}.
        \end{align*}
\end{proof}

\paragraph{Proof for Lemma \ref{Abeille_opt_prob}.}\

\begin{proof}
By Lemma \ref{lemma_same_dis} and  Lemma 3 in \cite{abeille2017linear}, we can get the result.
\end{proof}
\paragraph{Proof for Lemma \ref{lem:linear2reward}.}\

\begin{proof}
Assume $\tilde{a}_t+\tilde{b}_t\cdot p^*_t\geq a^*_t+b^*_t\cdot p^*_t$, then
\begin{align*}
    r^*(\tilde{a}_t,\tilde{b}_t)
    \geq &p^*_t \cdot g(\tilde{a}_t+\tilde{b}_t\cdot p^*_t)\\
    \geq &p^*_t \cdot  g(a^*_t+b^*_t\cdot p^*_t)\\
    =& r^*(a^*_t,b^*_t),
\end{align*}
where the first line is by the the definition of $r^*$, the second line is by $g$ is increasing from Assumption \ref{assp_g'} and the last line is again by the definition of $r^*$.
\end{proof}

\begin{lemma}
\label{lemma_opt_samp}
$$\mathbb{P}\left(r^*(\tilde{a}_t,\tilde{b}_t)\geq r^*(a^*_t,b^*_t) \bigg \vert \mathcal{H}_{t-1}, \theta^*\in \Theta_{t}\right)\geq \frac{1}{4\sqrt{e\pi}}.$$
\end{lemma}
\begin{proof}
Denote the optimal pricing at $t$ by $p^*_t\coloneqq \argmax_{p\in[\underline{p},\bar{p}]}r(p;a^*_t,b^*_t)$. Then since $p^*_t\in\mathcal{H}_{t-1}$, by Lemma \ref{Abeille_opt_prob} and the fact that linear function is convex,
\begin{equation*}
\mathbb{P}\left(\tilde{a}_t+\tilde{b}_t\cdot p^*_t\geq a^*_t+b^*_t\cdot p^*_t \bigg \vert \mathcal{H}_{t-1}, \theta^*\in \Theta_{t}\right)\geq \frac{1}{4\sqrt{e\pi}}.
\end{equation*}

With Lemma \ref{lem:linear2reward} we can conclude the result.
\end{proof}
\textbf{Remark.} Note the above result can be extended to the reward function with cost $r^*(\mu,a,b)\coloneqq \max_{p\in[\mu,\bar{p}]} (p-\mu)\cdot g(a+b\cdot p)$ for any cost $\mu \in [\underline{p},\bar{p}]$, which will be applied in the proof of Theorem \ref{thm_TS_C_UB}.

\subsubsection{Proof of Theorem \ref{thm_TS_NC}}
\begin{proof}
Let $\bar{\kappa}\coloneqq 2\bar{\gamma}\sqrt{\log (4T^2)}$. Define
$$\tilde{\Theta}_{t}\coloneqq \left\{(a,b) \in \mathbb{R}^{2}:  \|(a,b)-(\hat{a}_t,\hat{b}_t) \|_{\tilde{M}_{t-1}}\leq \bar{\kappa}  \right\}.$$

Now, we define the good event as
$$\mathcal{E}\coloneqq \{\theta^*\in \Theta_t,\  (\tilde{a}_t,\tilde{b}_t)\in \tilde{\Theta}_t \ \text{for} \ t=1,\ldots,T\}$$ where recall
$$\Theta_{t}=\left\{\theta\in \Theta: \left\|\theta-\hat{\theta}_{t-1}\right\|_{M_{t-1}}\leq \bar{\gamma}\right\}.$$

Then by Definition \ref{assp: AbeilleTS} and Corollary \ref{MLEbound}, we have
\begin{equation}
\label{eq_bad_prob}
   \mathbb{P}(\mathcal{E})\geq 1-\frac{2}{T}.
\end{equation}

 Now under event $\mathcal{E}$,  the regret can be decomposed into
\begin{align*}
    \text{Reg}^{\text{TS}}_T\cdot \mathbbm{1}_{\mathcal{E}}&=\sum_{t=1}^T \left(r^*(a^*_t,b^*_t)-r(p_t;a^*_t,b^*_t)\right) \cdot \mathbbm{1}_{\mathcal{E}}\\
&= \sum_{t=1}^T \left( r^*(a^*_t,b^*_t)-r^*(\tilde{a}_t,\tilde{b}_t)+r^*(\tilde{a}_t,\tilde{b}_t)-r(p_t;a^*_t,b^*_t) \right)\cdot \mathbbm{1}_{\mathcal{E}},
\end{align*}
where $\text{TS}$ denotes the pricing policy specified by Algorithm \ref{alg_TS}.

Let $$\text{Reg}^{(1)}_t:=  \left(r^*(a^*_t,b^*_t)-r^*(\tilde{a}_t,\tilde{b}_t)\right)\cdot\mathbbm{1}_{\mathcal{E}},$$
$$\text{Reg}^{(2)}_t:= \left(r^*(\tilde{a}_t,\tilde{b}_t)-r(p_t;a^*_t,b^*_t)\right)\cdot\mathbbm{1}_{\mathcal{E}}.$$

We first focus on analyzing $\text{Reg}^{(1)}_t$.  Denote
$$\Theta^{\text{OPT}}_t\coloneqq \left\{(a,b)\in\Tilde{\Theta}_t:r^*(a,b)\geq r^*(a^*_t,b^*_t)\right\}.$$
Then by Lemma \ref{lemma_opt_samp},
       $$\mathbb{P}\left( r^*(\tilde{a}_t,\tilde{b}_t)\geq r^*(a^*_t,b^*_t)  \bigg \vert \mathcal{H}_{t-1}, \theta^*\in \Theta_{t}\right)\geq \frac{1}{4\sqrt{e\pi}}.$$
    Further, by Definition \ref{assp: AbeilleTS}, when $T\geq 6$, $$\mathbb{P}\left(  (\tilde{a}_t,\tilde{b}_t)\notin \Tilde{\Theta}_t \bigg \vert \mathcal{H}_{t-1}, \theta^*\in \Theta_{t}\right)\leq \frac{1}{T^2}\leq  \frac{1}{8\sqrt{e\pi}}.$$
    Then by applying union bound to the above two events,
    \begin{equation}
    \label{eq_opt_samp}
    \mathbb{P}\left( (\tilde{a}_t,\tilde{b}_t)\in \Theta_t^{\text{OPT}}  \bigg \vert \mathcal{H}_{t-1}, \theta^*\in \Theta_{t}\right)\geq \frac{1}{4\sqrt{e\pi}}-\frac{1}{8\sqrt{e\pi}}=\frac{1}{8\sqrt{e\pi}}.
    \end{equation}

For any $(\tilde{a},\tilde{b})\in \Theta^{\text{OPT}}_t$ , we have
\begin{align}
\mathbb{E}\left[\text{Reg}^{(1)}_t|\mathcal{H}_{t-1}\right]&\leq \mathbb{E}\left[\left(r^*(\tilde{a},\tilde{b})-r^*(\tilde{a}_t,\tilde{b}_t)\right)\cdot \mathbbm{1}_{\mathcal{E}}\bigg \vert  \mathcal{H}_{t-1}\right]\nonumber \\
    &=  \mathbb{E}\left[\left(r^*(\tilde{a},\tilde{b})-p_t\cdot g(\tilde{a}_t+\tilde{b}_t\cdot p_t)\right)\cdot \mathbbm{1}_{\mathcal{E}}\bigg \vert \mathcal{H}_{t-1}\right]\nonumber \\
    &\leq  \mathbb{E}\left[p^*(\tilde{a},\tilde{b})\cdot\left( g(\tilde{a}+\tilde{b}\cdot p^*(\tilde{a},\tilde{b}))- g(\tilde{a}_t+\tilde{b}_t\cdot p^*(\tilde{a},\tilde{b}))\right)\cdot\mathbbm{1}_{\mathcal{E}}\bigg \vert  \mathcal{H}_{t-1}\right]\nonumber \\
    &\leq \Bar{g}\Bar{p} \mathbb{E}\left[\big \vert \tilde{a}+\tilde{b}\cdot p^*(\tilde{a},\tilde{b})- (\tilde{a}_t+\tilde{b}_t\cdot p^*(\tilde{a},\tilde{b}))\big \vert \cdot\mathbbm{1}_{\mathcal{E}}\bigg \vert \mathcal{H}_{t-1}\right]\nonumber \\
     &= \Bar{g}\Bar{p} \mathbb{E}\left[\big \vert (1,p^*(\tilde{a},\tilde{b}))^\top(\tilde{a}-\tilde{a}_t,\tilde{b}-\tilde{b}_t) \big \vert \cdot\mathbbm{1}_{\mathcal{E}}\bigg \vert  \mathcal{H}_{t-1}\right]\nonumber \\
     &\leq \Bar{g}\Bar{p} \mathbb{E}\left[ \|(1,p^*(\tilde{a},\tilde{b}))\|_{\Tilde{M}^{-1}_{t-1}}\|(\tilde{a}-\tilde{a}_t,\tilde{b}-\tilde{b}_t) \|_{\Tilde{M}_{t-1}} \cdot\mathbbm{1}_{\mathcal{E}}\bigg \vert  \mathcal{H}_{t-1}\right]\nonumber \\
      &=\Bar{g}\Bar{p} \mathbb{E}\left[ \|\tilde{z}_t(\tilde{a},\tilde{b}))\|_{M^{-1}_{t-1}}\|(\tilde{a}-\tilde{a}_t,\tilde{b}-\tilde{b}_t) \|_{\Tilde{M}_{t-1}} \cdot\mathbbm{1}_{\mathcal{E}}\bigg \vert  \mathcal{H}_{t-1}\right]\nonumber \\
    &\leq 4\Bar{g}\Bar{p}\bar{\gamma}\sqrt{\log(4T^2)}   \mathbb{E}\left[\|\tilde{z}_t(\tilde{a},\tilde{b}))\|_{M^{-1}_{t-1}}\cdot\mathbbm{1}_{\mathcal{E}}\bigg \vert  \mathcal{H}_{t-1}\right], \label{ineq_theta_tilde}
\end{align}
where $p^*(a,b)=\argmax_{p\in [\underline{p},\bar{p}] }r(p;a,b)$ and $\tilde{z}_t(a,b)\coloneqq (x_t,p^*(a,b)x_t)$. Here the first line is by the definition of $\Theta^{\text{OPT}}_t$,
the third line is from the optimality of $p_t$, the fourth line is from Assumptions \ref{assp_bound}  and \ref{assp_g'}, the sixth line is by Holder's inequality, the seventh line is by
$$\|(1,p^*(\tilde{a},\tilde{b}))\|_{\Tilde{M}^{-1}_{t-1}}=\sqrt{(1,p^*(\tilde{a},\tilde{b}))^\top \begin{pmatrix}
x_t&\bm{0}\\
\bm{0}&x_t
\end{pmatrix}^\top M_{t-1}^{-1} \begin{pmatrix}
x_t&\bm{0}\\
\bm{0}&x_t
\end{pmatrix} (1,p^*(\tilde{a},\tilde{b})) }=\|\tilde{z}_t(\tilde{a},\tilde{b}))\|_{M^{-1}_{t-1}},$$
and the last line comes from the event $\mathcal{E}$ and $(\tilde{a},\tilde{b})\in \Theta^{\text{OPT}}_t$ with triangle inequality.

Let $(\tilde{a}_{t}',\tilde{b}_t')$ be an independent copy of $(\tilde{a}_{t},\tilde{b}_t)$ following the same distribution. Then we have
\begin{align}
   \mathbb{E}\left[\text{Reg}^{(1)}_t \bigg \vert \mathcal{H}_{t-1}\right] & \le 4\Bar{g}\Bar{p}\bar{\gamma}\sqrt{\log(4T^2)}   \mathbb{E}\left[\|\tilde{z}_t(\tilde{a}'_t,\tilde{b}'_t)\|_{M_{t-1}^{-1}}\cdot\mathbbm{1}_{\mathcal{E}}\bigg \vert (\tilde{a}_{t}',\tilde{b}_t')\in \Theta_{t}^{\text{OPT}}, \mathcal{H}_{t-1}\right] \nonumber \\
   & \le 4\Bar{g}\Bar{p}\bar{\gamma}\sqrt{\log(4T^2)}   \mathbb{E}\left[\|\tilde{z}_t(\tilde{a}'_t,\tilde{b}'_t)\|_{M_{t-1}^{-1}}\cdot\mathbbm{1}_{\mathcal{E}}\cdot\mathbbm{1}_{(\tilde{a}'_t,\tilde{b}'_t)\in \Theta_{t}^{\text{OPT}}}\bigg \vert \mathcal{H}_{t-1}\right]\mathbb{P}^{-1}\left((\tilde{a}'_t,\tilde{b}'_t)\in \Theta_{t}^{\text{OPT}}\vert \mathcal{H}_{t-1}, \theta^*\in \Theta_{t-1}\right) \nonumber\\
   & \le 4\Bar{g}\Bar{p}\bar{\gamma}\sqrt{\log(4T^2)}   \mathbb{E}\left[\|\tilde{z}_t(\tilde{a}'_t,\tilde{b}'_t)\|_{M_{t-1}^{-1}}\bigg \vert \mathcal{H}_{t-1}\right] \cdot \mathbb{P}^{-1}\left((\tilde{a}'_t,\tilde{b}'_t)\in \Theta_{t}^{\text{OPT}}\vert \mathcal{H}_{t-1}, \theta^*\in \Theta_{t-1}\right) \nonumber\\
  & \leq   4\Bar{g}\Bar{p}\bar{\gamma}\sqrt{\log(4T^2)} \mathbb{E}\left[\|\tilde{z}_t(\tilde{a}'_t,\tilde{b}'_t)\|_{M_{t-1}^{-1}}\bigg \vert \mathcal{H}_{t-1} \right]\cdot 8\sqrt{e\pi} \nonumber\\
  & = 32\Bar{g}\Bar{p}\Bar{\gamma}\sqrt{e \pi \log(4T^2)}   \mathbb{E}\left[\|z_t\|_{M_{t-1}^{-1}}\bigg \vert \mathcal{H}_{t-1}\right] \label{ineq_reg_t_1}.
\end{align}
where $z_t=(x_t, p_tx_t)$ and $p_t$ is the price used in the algorithm. Here the first line comes by replacing $(\tilde{a},\tilde{b})$ in \eqref{ineq_theta_tilde} with a randomized parameter of $(\tilde{a}_{t}',\tilde{b}_t')$ restricted to the set $\Theta_{t}^{\text{OPT}}.$ The second line comes from the property of conditional expectation. The third line removes the indicator functions. The fourth line applies \eqref{eq_opt_samp}. The last line comes from the definition of $\tilde{z}_t(\tilde{a}'_t,\tilde{b}'_t)$ and $z_t.$

Next, for $\text{Reg}^{(2)}_t$,

\begin{align}
\mathbb{E}\left[\text{Reg}^{(2)}_t \bigg \vert \mathcal{H}_{t-1} \right] &= \mathbb{E}\left[p_t\cdot (g(\tilde{a}_t+\tilde{b}_t\cdot p_t)- g(a^*_t+b^*_t\cdot p_t))\cdot \mathbbm{1}_{\mathcal{E}} \bigg \vert \mathcal{H}_{t-1} \right] \nonumber\\
&\leq \Bar{g}\Bar{p}\mathbb{E}\left[|\tilde{a}_t+\tilde{b}_t\cdot p_t- (a^*_t+b^*_t\cdot p_t)|\cdot \mathbbm{1}_{\mathcal{E}} \bigg \vert \mathcal{H}_{t-1} \right] \nonumber\\
&=\Bar{g}\Bar{p}\mathbb{E}\left[|\tilde{a}_t+\tilde{b}_t\cdot p_t-(\hat{a}_t+\hat{b}_t\cdot p_t)+(\hat{a}_t+\hat{b}_t\cdot p_t)- (a^*_t+b^*_t\cdot p_t)|\cdot \mathbbm{1}_{\mathcal{E}} \bigg \vert \mathcal{H}_{t-1} \right] \nonumber\\
&\leq\Bar{g}\Bar{p}\mathbb{E}\left[\left(\|(1,p^*(\tilde{a}_t,\tilde{b}_t))\|_{\Tilde{M}^{-1}_{t-1}}\|(\tilde{a}_t-\hat{a}_t,\tilde{b}_t-\hat{b}_t) \|_{\Tilde{M}_{t-1}}  +\|z_t\|_{M_{t-1}^{-1}}\|\hat{\theta}_{t-1}-\theta^*\|_{M_{t-1}}\right)\cdot \mathbbm{1}_{\mathcal{E}} \bigg \vert \mathcal{H}_{t-1} \right] \nonumber\\
&= \Bar{g}\Bar{p}\mathbb{E}\left[\left(\|z_t\|_{M_{t-1}^{-1}}\|(\tilde{a}_t-\hat{a}_t,\tilde{b}_t-\hat{b}_t) \|_{\Tilde{M}_{t-1}}  +\|z_t\|_{M_{t-1}^{-1}}\|\hat{\theta}_{t-1}-\theta^*\|_{M_{t-1}}\right)\cdot \mathbbm{1}_{\mathcal{E}} \bigg \vert \mathcal{H}_{t-1} \right] \nonumber\\
&\leq (2\sqrt{\log(4T^2)}+1)\Bar{g}\Bar{p}\bar{\gamma} \mathbb{E}\left[\|z_t\|_{M_{t-1}^{-1}} \bigg \vert \mathcal{H}_{t-1} \right] \nonumber\\
&< 4\Bar{g}\Bar{p}\bar{\gamma}\sqrt{e \pi \log(4T^2)} \mathbb{E}\left[\|z_t\|_{M_{t-1}^{-1}} \bigg \vert \mathcal{H}_{t-1} \right] \label{ineq_regt_2}
\end{align}
where $z_t=(x_t, p_tx_t)$ and $p_t$ is the price used in the algorithm. The second line is from Assumptions \ref{assp_bound}  and \ref{assp_g'}, the fourth line is by Holder's inequality, the fifth line is due to the same computation used for $\text{Reg}^{(1)}_t$, and the sixth line is from the definition of event $\mathcal{E}$. Now we can combine them and conclude the final result:

\begin{align*}
   \text{Reg}_T^{\text{TS}}(x_1,\ldots,x_T)&\leq \sum_{t=1}^T \mathbb{E}\left[\left(\text{Reg}^{(1)}_t+\text{Reg}^{(2)}_t\right)\cdot \mathbbm{1}_{\mathcal{E}}\right]+\mathbb{E}\left[\Bar{p}\bar{D}T \cdot \mathbbm{1}_{\mathcal{E}^c}\right]\\
    &\leq  36\Bar{g}\Bar{p}\bar{\gamma}\sqrt{e \pi \log(4T^2)}\mathbb{E}\left[\sum_{t=1}^T\|z_t\|_{M_{t-1}^{-1}}\right]+\mathbb{P}\left(\mathcal{E}^c\right)\cdot \Bar{p}\bar{D}T\\
    &\leq 36\Bar{g}\Bar{p}\bar{\gamma}\sqrt{e \pi \log(4T^2)}\mathbb{E}\left[\sqrt{T\sum_{t=1}^T\|z_t\|^2_{M_{t-1}^{-1}}}\right]+\mathbb{P}\left(\mathcal{E}^c\right)\cdot \bar{p}\bar{D}T\\
    &\leq  36\Bar{g}\Bar{p}\bar{\gamma}\sqrt{2e \pi dT\log\left(\frac{2d\lambda+T}{2d\lambda}\right)\log(4T^2)}+2\bar{D}\bar{p}
\end{align*}
where  $\mathcal{E}^c$ denotes the complement of the event $\mathcal{E}.$ The first line is by the single-period regret will be bounded by $\Bar{p}\Bar{D}$, the third line is by Holder's inequality and the last inequality is because \eqref{eq_bad_prob} and Elliptical Potential Lemma, i.e. Lemma \ref{EPL_lemma}.
\end{proof}

\subsection{Proofs for Section \ref{sec_TS_C}}
\label{sec_appn_TS_C}

\subsubsection{Proof of Theorem \ref{thm_LB}}
\begin{proof}
Assume $\alpha^*=(2,0)$ and $\beta^*=(0,-1)$ and there exists some $K>1$, $c=\frac{K}{2}$ and $x_t=(K,K)$ for all $t=1,\ldots,\frac{T}{2}$. For the second half of the horizon, the nature randomly chooses between the following two cases: (1) $x_t=(K,K-\delta)$ or (2) $x_t=(K,K+\delta)$ with some $\delta\in(1,K)$ for all $t=\frac{T}{2}+1,\ldots,T$. And further we assume there exists no error on demands, $\mathbb{P}(X_t=x_t)=1$ for all $t=1,\ldots,T$ and $g(x)=x$, i.e. with price $p$, $D_t(p)=2K- K\cdot p$ for all $t=1,\ldots,\frac{T}{2}$ and $D_t(p)=2K- (K-\delta)\cdot p$ or $D_t(p)=2K-(K+\delta)\cdot p$ for all $t=\frac{T}{2}+1,\ldots,T$ based on the nature's choice.

Now we consider the optimal pricing policy under the first case. Denote the proportion of inventory used in the second half horizon as $q$. Note $p\cdot D_t(p)$ is concave in $p$ for all $t$ and thus by this concavity, it is easy to see for a fixed $q$ the optimal pricing policy is $p_t=1+q$ for $t=1,\ldots,\frac{T}{2}$ and $p_t=\frac{K(2-q)}{K-\delta}$ for $t=\frac{T}{2}+1,\ldots,T$. So it is sufficient to find the optimal $q^*(\delta)$ for the optimal policy. Indeed, the revenue under optimal pricing of a given $q$ is:
$$T\cdot \left(\frac{K(1+q)(1-q)}{2}+\frac{K^2q(2-q)}{2(K-\delta)}\right),$$
and thus the optimal $q$ can be computed as $q^*(\delta)=\frac{K}{2K-\delta}$ with optimal revenue as $ \frac{TK^2}{2(2K-\delta)(K-\delta)}+\frac{TK}{2}$. Intuitively, when $\delta$ is larger, since customers are less sensitive to the price in the second half, the seller should save more inventory for the second half. Similarly, for the second case the optimal the optimal $q$ can be computed as $q^*(\delta)=\frac{K}{2K+\delta}$ with optimal value as $ \frac{TK^2}{2(2K+\delta)(K+\delta)}+\frac{TK}{2}$.

Then for any online policy $\pi$, if the proportion of inventory used in the second half horizon of it $q^{\pi}\leq \frac{1}{2}$, then with probability $\frac{1}{2}$ the environment will become the first case and will cause the regret at least:
$$ \frac{TK^2}{2(2K-\delta)(K-\delta)}+\frac{TK}{2}-T\cdot \left(\frac{3K}{8}+\frac{3K^2}{8(K-\delta)}\right)=\frac{KT\delta^2}{8(2K-\delta)(K-\delta)}.$$
Similarly if the proportion of inventory used in the second half horizon of it $q^{\pi}>\frac{1}{2}$, then with probability $\frac{1}{2}$ the environment will become the second case and will cause the regret at least:
$$ \frac{KT\delta^2}{8(2K+\delta)(K+\delta)}.$$
Note  $\mathcal{W}_T(\mathcal{P}_1,\ldots,\mathcal{P}_T)=\sum_{t=1}^T\frac{\delta}{\sqrt{2}}=\frac{T\delta}{\sqrt{2}}$ for both two cases, we finish the proof.
\end{proof}

\subsubsection{Proof of Theorem \ref{thm_TS_C_UB}}
In this subsection, we denote $\tilde{r}(p;\mu,a,b)=(p-\mu)\cdot g(a+b\cdot p)$ as the revised revenue function with virtual cost $\mu$, and $\tilde{r}^*(\mu,a,b)=\max_{p\in [\mu,\bar{p}]} \tilde{r}(p;\mu,a,b)$ as the corresponding optimal revised revenue function (assume $\mu\leq \Bar{p}$).

\noindent \textbf{Proof sketch.} The proof idea of Theorem \ref{thm_TS_C_UB} is to decompose the regret as follows:
$$\mathbb{E} \left[\sum_{t=1}^{\tau+1}\left(\Tilde{r}^*(\mu_t,A^*_t,B^*_t)-\Tilde{r}(p_t;\mu_t,A^*_t,B^*_t)+\mu_t (c-g(\tilde{a}_t+\tilde{b}_t\cdot p_t ))\right)+\bar{p}c\cdot (T-\tau-1) \right]+O\left(\mathcal{W}_T(\mathcal{P}_1,\ldots,\mathcal{P}_T)\right),$$
where $\tau$ is the termination time of the algorithm:
$$\tau \coloneqq \max \left\{ t=1,\ldots,T-1: \sum_{t'=1}^t D_{t'}\leq  cT \right\},$$
which is the last period before the inventory is exhausted or equal to $T-1$ when still has inventory at $T-1$. The last term $O\left(\mathcal{W}_T(\mathcal{P}_1,\ldots,\mathcal{P}_T)\right)$ measures the effect of the non-stationarity, and we can then focus on bounding the terms in the expectation:
\begin{itemize}
    \item $\mathbb{E} \left[\sum_{t=1}^{\tau+1}\Tilde{r}^*(\mu_t,A^*_t,B^*_t)-\Tilde{r}(p_t;\mu_t,A^*_t,B^*_t)\right].$ This term captures the revised regret before algorithm stops based on the revised revenue function $\Tilde{r}(p;\mu_t,A^*_t,B^*_t)$, which considers the cost of each unit as $\mu_t$. Since the price $p_t$ is the optimal price from sampled variables  $(\tilde{a}_t,\tilde{b}_t)$, this term can be bounded by a similar analysis as in Theorem \ref{thm_TS_NC}.
    \item $\mathbb{E} \left[\sum_{t=1}^{\tau+1}\mu_t (c-g(\tilde{a}_t+\tilde{b}_t\cdot p_t ))\right].$ This term captures the ``extra revenues/negative regrets'' from using inventory with the ``price'' $\mu_t$, which can be bounded by $c\Bar{p}\mathbbm{E}[\tau+1-T]+O(\sqrt{T})$. Intuitively, the algorithm will consume all inventory $cT$ at $\tau+1$ while the target consumption is only $c\cdot (\tau+1)$. Thus, these extra consumed inventories should earn some extra revenues or negative regrets. Thanks to the online gradient descent, such extra revenues can be bounded by $c\bar{p}\mathbbm{E}[\tau+1-T]+O(\sqrt{T})$. The analysis is similar to \cite{agrawal2016linear}, which, however, focuses on the UCB algorithm under a stationary environment.
    \item $\mathbb{E} \left[\Bar{p}\Bar{c}\cdot (T-\tau-1)\right].$ This term captures the regret caused by the early stop of the algorithm: the algorithm stops at $\tau+1\leq T$. Such an early stop means the algorithm may consume the inventory too fast and can at most make $\mathbb{E} \left[\Bar{p}\Bar{c}\cdot (T-\tau-1)\right]$ loss. However, such loss can be covered by the extra revenues $\mathbb{E} \left[\sum_{t=1}^{\tau+1}\mu_t (c-g(\tilde{a}_t+\tilde{b}_t\cdot p_t ))\right]$ as discussed above.
\end{itemize}

\paragraph{Proof of Theorem \ref{thm_TS_C_UB}.}\
\begin{proof}
 Here we rewrite the deterministic optimization problem \eqref{eq_RT_Det} as the following in a more detailed way:
\begin{align*}
    r^*_T(\{\mathcal{P}_t\})= \max& \sum_{t=1}^T \mathbb{E}_t\left[p_t(X_t)g(A_t^*+B_t^*\cdot p_t(X_t))\right] \\
\text{s.t.} &  \sum_{t=1}^T \mathbb{E}_t\left[ g(A_t^*+B_t^*\cdot p_t(X_t))\right]\leq cT, \nonumber \\
& p_t(x): \mathcal{X}\rightarrow [\underline{p},\Bar{p}] \ \text{is a measurable function for}\  t=1,\ldots,T, \nonumber
\end{align*}
where $\mathbb{E}_t$ is the expectation taken with respect to $X_t\sim \mathcal{P}_t$. Then one standard result in revenue management literature \citep{gallego1997multiproduct,talluri2004theory,jiang2020online} is the benchmark  $\mathbb{E}[\textup{OPT}(X_1,\ldots,X_T,cT)]$ can be upper bounded by $r^*_T(\{\mathcal{P}_t\})$:
\begin{lemma}[Lemma 1 in \citep{jiang2020online}]
\label{lemma_opt_det_ub}
Under Assumption \ref{assp_bound}, \ref{assp_error}, \ref{assp_g'}, \ref{assmp_gen_x_t}, \ref{assmp_feasible},
    $$\mathbb{E}[\textup{OPT}(X_1,\ldots,X_T,cT)]\leq r^*_T(\{\mathcal{P}_t\}),$$
    where the expectation is with respect to $X_t\sim \mathcal{P}_t$ for $t=1,\ldots,T$.
\end{lemma}
Although $r^*_T(\{\mathcal{P}_t\})$ is more tractable than $\mathbb{E}[\textup{OPT}(X_1,\ldots,X_T,cT)]$, we need to further upper bound it with a pricing policy which can utilize the dual variables $\mu_t$'s solved by the algorithm. We define the termination time
$$\tau = \max \left\{ t=1,\ldots,T-1: \sum_{t'=1}^t D_{t'}\leq  cT \right\},$$
which is the last period before the inventory is exhausted or equals to $T-1$ when still has inventory at $T-1$. Further, recall $\tilde{r}^*(\mu,a,b)=\max_{p\in [\mu,\bar{p}]}\Tilde{r}(p;\mu,a,b)$ and $(A^*_t,B^*_t)=(X^\top_t\alpha^*,X^\top_t\beta^*)$, then

\begin{lemma}
    \label{lemma_reviseOPT}
    $$r^*_T\left(\{\mathcal{P}_t\}\right)\leq  \mathbb{E} \left[\sum_{t=1}^{\tau+1}\left(c\mu_t+\Tilde{r}^*(\mu_t,A^*_t,B^*_t)\right)+\bar{p}c\cdot (T-\tau-1) \right]+\sqrt{2}(\bar{p}\vee \Bar{p}^2)\bar{g}\bar{\theta} \mathcal{W}_T(\mathcal{P}_1,\ldots,\mathcal{P}_T).$$
\end{lemma}

The above lemma says the benchmark performance can be bounded by three parts: $\sum_{t=1}^{\tau+1}\left(c\mu_t+\Tilde{r}^*(\mu_t;A_t^*,B_t^*)\right)$ captures the Lagrangian formulation (with the dual variables $\mu_t$'s) of revenues gotten when inventory is exhausted, $\bar{p}c\cdot (T-\tau-1)$ captures the maximum revenue gotten after inventory is exhausted, and $\mathcal{W}_T(\mathcal{P}_1,\ldots,\mathcal{P}_T)$ captures the effect of non-stationarity of $\mathcal{P}_t$'s. The proof of the above lemma is mainly based on the dual formulations of both $r^*_T\left(\{\mathcal{P}_t\}\right)$ and Wasserstein distance function. We postpone the proof to Appendix~\ref{sec_anc_lemmas}.

Now we introduce some similar notations as in the proof of Theorem \ref{thm_TS_NC}. Let $\bar{\gamma}=2\bar{\theta}+ \frac{2\bar{\sigma}}{\underline{g}}\sqrt{2\log T+2d\log\left(\frac{2d+T}{2d}\right)}$,   $\bar{\kappa}=2\bar{\gamma}\sqrt{\log (4T^2)}$ and
$$\tilde{\Theta}_{t}= \left\{(a,b) \in \mathbb{R}^{2}:  \|(a,b)-(\hat{a}_t,\hat{b}_t) \|_{\tilde{M}_{t-1}}\leq \bar{\kappa}  \right\}.$$

We define the good event as
$$\mathcal{E}= \{\theta^*\in \Theta_t,\  (\tilde{a}_t,\tilde{b}_t)\in \tilde{\Theta}_t \ \text{for} \ t=1,\ldots,T\}$$ where recall
$$\Theta_{t}=\left\{\theta\in \Theta: \left\|\theta-\hat{\theta}_{t-1}\right\|_{M_{t-1}}\leq \bar{\gamma}\right\}.$$

Then by Definition \ref{assp: AbeilleTS} and Corollary \ref{MLEbound}, we have
\begin{equation}
\label{eq_bad_prob_TSC}
   \mathbb{P}(\mathcal{E})\geq 1-\frac{2}{T}.
\end{equation}

For $t=1,\ldots,\tau+1$, we denote
$$\text{Reg}^{(1)}_t:= \left(c\mu_t+\tilde{r}^*(\mu_t,A^*_t,B^*_t)-r(p_t;\tilde{a}_t,\tilde{b}_t)\right)\cdot\mathbbm{1}_{\mathcal{E}},$$
$$\text{Reg}^{(2)}_t:= \left(r(p_t;\tilde{a}_t,\tilde{b}_t)-r(p_t;A^*_t,B^*_t)\right)\cdot\mathbbm{1}_{\mathcal{E}},$$
where $\mu_t,p_t$ are from Algorithm \ref{alg_TS_C}.

We first focus on analyzing $\text{Reg}^{(1)}_t$. Denote
$$\Theta^{\text{OPT}}_t\coloneqq \left\{(a,b)\in\Tilde{\Theta}_t: \tilde{r}^*(\mu_t,a,b) \geq \tilde{r}^*(\mu_t,A^*_t,B^*_t) \right\}.$$
Then by (the remark below) Lemma \ref{lemma_opt_samp},
       $$\mathbb{P}\left( \tilde{r}^*(\mu_t,\tilde{a}_t,\tilde{b}_t) \geq \tilde{r}^*(\mu_t,A^*_t,B^*_t)  \bigg \vert \mathcal{H}_{t-1}, \theta^*\in \Theta_{t}\right)\geq \frac{1}{4\sqrt{e\pi}},$$
by noting $\mu_t\in \mathcal{H}_{t-1}$. With similar computation as in the proof of Theorem \ref{thm_TS_NC}, we have
    \begin{equation}
    \label{eq_opt_samp_TSC}
    \mathbb{P}\left( (\tilde{a}_t,\tilde{b}_t)\in \Theta_t^{\text{OPT}}  \bigg \vert \mathcal{H}_{t-1}, \theta^*\in \Theta_{t}\right)\geq \frac{1}{8\sqrt{e\pi}}.
    \end{equation}

For any $(\tilde{a},\tilde{b})\in \Theta^{\text{OPT}}_t$ , we have
\begin{align}
\mathbb{E}\left[\text{Reg}^{(1)}_t \bigg \vert  \mathcal{H}_{t-1}\right]&\leq \mathbb{E}\left[ \left(c\mu_t+\max_{p\in[\mu_t,\bar{p}]}\tilde{r}(p;\mu_t,\tilde{a},\tilde{b}) -r(p_t;\tilde{a}_t,\tilde{b}_t)\right)\cdot \mathbbm{1}_{\mathcal{E}}\bigg \vert  \mathcal{H}_{t-1}\right]\nonumber \\
    &=  \mathbb{E}\left[ \left(\mu_t\cdot(c-g(\tilde{a}_t+\tilde{b}_t\cdot p_t))+\max_{p\in[\mu_t,\bar{p}]}\tilde{r}(p;\mu_t,\tilde{a},\tilde{b}) -\tilde{r}(p_t;\mu_t,\tilde{a}_t,\tilde{b}_t)\right)\cdot \mathbbm{1}_{\mathcal{E}}\bigg \vert  \mathcal{H}_{t-1}\right]\nonumber \\
        &\leq  \mathbb{E}\left[ \left(\mu_t\cdot(c-g(\tilde{a}_t+\tilde{b}_t\cdot p_t))+\tilde{r}(\tilde{p}^*_t;\mu_t,\tilde{a},\tilde{b}) -\tilde{r}(\tilde{p}^*_t;\mu_t,\tilde{a}_t,\tilde{b}_t)\right)\cdot \mathbbm{1}_{\mathcal{E}}\bigg \vert  \mathcal{H}_{t-1}\right]\nonumber \\
    &\leq    \mathbb{E}\left[ \left(4\Bar{g}\Bar{p}\bar{\gamma}\sqrt{\log(4T^2)}\|\tilde{z}_t(\tilde{a},\tilde{b}))\|_{M^{-1}_{t-1}}+\mu_t\cdot (c-g(\tilde{a}_t+\tilde{b}_t\cdot p_t) \right)\cdot \mathbbm{1}_{\mathcal{E}} \bigg \vert \mathcal{H}_{t-1}\right]\nonumber
\end{align}
where $\tilde{z}_t(a,b)\coloneqq (x_t,\tilde{p}^*_tx_t)$ and $\tilde{p}^*_t\coloneqq \argmax_{p\in [\mu_t,\bar{p}] }\tilde{r}(p;\mu_t, \tilde{a},\tilde{b})$. Here the first line is by definition of $\Theta^{\text{OPT}}_t$, the second line is by definition of $\tilde{r}(p;\mu, a,b)$, the third line is by the optimality of $p_t$, and the last line is from similar computation steps as in \eqref{ineq_theta_tilde}. Then by the same computation steps in \eqref{ineq_reg_t_1}, we get
\begin{equation}
    \label{eq_reg_inv_reg1}
    \mathbb{E}\left[\text{Reg}^{(1)}_t \bigg \vert  \mathcal{H}_{t-1}\right]\leq 32\Bar{g}\Bar{p}\Bar{\gamma}\sqrt{e \pi \log(4T^2)}   \mathbb{E}\left[\|z_t\|_{M_{t-1}^{-1}}\bigg \vert \mathcal{H}_{t-1}\right] +\mathbb{E}\left[\mu_t\cdot (c-g(\tilde{a}_t+\tilde{b}_t\cdot p_t) )\cdot \mathbbm{1}_{\mathcal{E}} \bigg \vert \mathcal{H}_{t-1}\right],
\end{equation}
where $z_t=(x_t, p_tx_t)$ and $p_t$ is the price used in the algorithm. Thus, $\mathbb{E}\left[\sum_{t=1}^{\tau+1} \text{Reg}_t^{(1)}\right]$ can be bounded by
\begin{align}
    \mathbb{E}\left[\sum_{t=1}^{\tau+1} \text{Reg}_t^{(1)}\right]&= \mathbb{E}\left[\sum_{t=1}^{T} \mathbbm{1}_{\{t\leq \tau+1\}}\mathbb{E}\left[\text{Reg}_t^{(1)}\big\vert \mathcal{H}_{t-1}\right]\right] \nonumber\\
    &\leq 32\Bar{g}\Bar{p}\Bar{\gamma}\sqrt{e \pi \log(4T^2)}\sum_{t=1}^T  \mathbb{E}\left[\|z_t\|_{M_{t-1}^{-1}}\right]+ \mathbb{E}\left[\sum_{t=1}^{T} \mathbbm{1}_{\{t\leq \tau+1\}} \mathbb{E}\left[\mu_t\cdot (c-g(\tilde{a}_t+\tilde{b}_t\cdot p_t) )\cdot \mathbbm{1}_{\mathcal{E}} \bigg \vert \mathcal{H}_{t-1}\right] \right] \nonumber\\
    &\leq 32\Bar{g}\Bar{p}\bar{\gamma}\sqrt{2e \pi dT\log\left(\frac{2d\lambda+T}{2d\lambda}\right)\log(4T^2)}+ \mathbb{E}\left[\sum_{t=1}^{T} \mathbbm{1}_{\{t\leq \tau+1\}} \mathbb{E}\left[\mu_t\cdot (c-g(\tilde{a}_t+\tilde{b}_t\cdot p_t) )\cdot \mathbbm{1}_{\mathcal{E}} \bigg \vert \mathcal{H}_{t-1}\right] \right] \nonumber\\
    &=32\Bar{g}\Bar{p}\bar{\gamma}\sqrt{2e \pi dT\log\left(\frac{2d\lambda+T}{2d\lambda}\right)\log(4T^2)}+ \mathbb{E}\left[\sum_{t=1}^{\tau+1} \mu_t\cdot (c-g(\tilde{a}_t+\tilde{b}_t\cdot p_t) )\cdot \mathbbm{1}_{\mathcal{E}} \right] \nonumber\\
    &\leq 36\Bar{g}\Bar{p}\bar{\gamma}\sqrt{2e \pi dT\log\left(\frac{2d\lambda+T}{2d\lambda}\right)\log(4T^2)}+c\bar{p}\mathbb{E}[\tau+1-T]+\Bar{p}\sqrt{(\bar{D}^2+\bar{\sigma}^2)T}+2\Bar{p}\Bar{D} \label{eq_reg_TSC_reg1}
\end{align}
where the first line is by the tower rule and $ \mathbbm{1}_{\{t\leq \tau+1\}}\in \mathcal{H}_{t-1}$, the second line is by \eqref{eq_reg_inv_reg1}, the third line is by Lemma \ref{EPL_lemma}, the fourth line is again by the tower rule and $ \mathbbm{1}_{\{t\leq \tau+1\}}\in \mathcal{H}_{t-1}$, and the last line is by the lemma below whose proof can be found in Appendix \ref{sec_anc_lemmas}:
\begin{lemma}
\label{lemma_GD_inv}
$$\mathbb{E}\left[\sum_{t=1}^{\tau+1} \mu_t\cdot (c-g(\tilde{a}_t+\tilde{b}_t\cdot p_t) )\cdot \mathbbm{1}_{\mathcal{E}} \right]\leq c\bar{p}\mathbb{E}[\tau+1-T]+\Bar{p}\sqrt{(\bar{D}^2+\bar{\sigma}^2)T}+2\Bar{p}\Bar{D}+4\Bar{g}\Bar{p}\bar{\gamma}\sqrt{2e \pi dT\log\left(\frac{2d\lambda+T}{2d\lambda}\right)\log(4T^2)}.$$
\end{lemma}

And with the same argument used in the proof of Theorem \ref{thm_TS_NC}, $\mathbb{E}\left[\sum_{t=1}^{\tau+1} \text{Reg}_t^{(2)}\right]$ can be bounded by
\begin{equation}
\label{eq_reg_TSC_reg2}
    \mathbb{E}\left[\sum_{t=1}^{\tau+1} \text{Reg}_t^{(2)}\right]= \mathbb{E}\left[\sum_{t=1}^{T} \mathbbm{1}_{\{t\leq \tau+1\}}\mathbb{E}\left[\text{Reg}_t^{(2)}\big\vert \mathcal{H}_{t-1}\right]\right]\leq 4\Bar{g}\Bar{p}\bar{\gamma}\sqrt{2e \pi dT\log\left(\frac{2d\lambda+T}{2d\lambda}\right)\log(4T^2)}.
\end{equation}

We use $\text{TSC}$(Thompson Sampling with Constraint) to denote the Algorithm \ref{alg_TS_C}, then
\begin{align*}
   \text{Reg}_T^{\text{TSC}}&\leq \mathbb{E} \left[\sum_{t=1}^{\tau+1}\left(c\mu_t+\Tilde{r}^*(\mu_t,A_t^*,B_t^*)\right)+\bar{p}c\cdot (T-\tau-1) \right]+\sqrt{2}(\bar{p}\vee \Bar{p}^2)\bar{g}\bar{\theta} \mathcal{W}_T(\mathcal{P}_1,\ldots,\mathcal{P}_T)-\mathbb{E} \left[\sum_{t=1}^{\tau} r(p_{t};A^*_{t},B^*_{t})\right]\\
   &\leq \mathbb{E}\left[\sum_{t=1}^{\tau+1} \left(\text{Reg}_t^{(1)}+\text{Reg}_t^{(2)}+(c\mu_t+\Tilde{r}^*(X_t,\mu_t))\cdot \mathbbm{1}_{\mathcal{E}^c}\right)+\bar{p}c\cdot (T-\tau-1)+r(p_{\tau+1};A^*_{\tau+1},B^*_{\tau+1}) \right]\\
   &+\sqrt{2}(\bar{p}\vee \Bar{p})^2\bar{g}\bar{\theta} \mathcal{W}_T(\mathcal{P}_1,\ldots,\mathcal{P}_T)\\
   &\leq \mathbb{E}\left[\sum_{t=1}^{\tau+1} (\text{Reg}_t^{(1)}+\text{Reg}_t^{(2)})\right]+\mathbbm{P}(\mathcal{E}^c)\cdot (c+\Bar{D})\bar{p}T+\bar{p}c\mathbb{E}\left[T-\tau-1 \right]+\bar{D}\Bar{p}+\sqrt{2}(\bar{p}\vee \Bar{p}^2)\bar{g}\bar{\theta} \mathcal{W}_T(\mathcal{P}_1,\ldots,\mathcal{P}_T)\\
      &\leq \mathbb{E}\left[\sum_{t=1}^{\tau+1} (\text{Reg}_t^{(1)}+\text{Reg}_t^{(2)})\right]+2(c+\Bar{D})\bar{p}+\bar{p}c\mathbb{E}\left[T-\tau-1 \right]+\bar{D}\Bar{p}+\sqrt{2}(\bar{p}\vee \Bar{p}^2)\bar{g}\bar{\theta} \mathcal{W}_T(\mathcal{P}_1,\ldots,\mathcal{P}_T)\\
   &\leq  40\Bar{g}\Bar{p}\bar{\gamma}\sqrt{2e \pi dT\log\left(\frac{2d\lambda+T}{2d\lambda}\right)\log(4T^2)}+\Bar{p}\sqrt{(\bar{D}^2+\bar{\sigma}^2)T}+7\Bar{D}\Bar{p}+\sqrt{2}(\bar{p}\vee \Bar{p}^2)\bar{g}\bar{\theta} \mathcal{W}_T(\mathcal{P}_1,\ldots,\mathcal{P}_T),
\end{align*}
where the first inequality is by Lemma \ref{lemma_reviseOPT}, the second inequality is by definitions, the third inequality is by $\mu_t\leq \bar{p}$ and $r(p_{\tau+1};A^*_{\tau+1},B^*_{\tau+1}),\tilde{r}^*(X_t,\mu_t)\leq \bar{p}\bar{D}$, the fourth inequality is by \eqref{eq_bad_prob_TSC} and the last inequality is by  \eqref{eq_reg_TSC_reg1}, \eqref{eq_reg_TSC_reg2} and Assumption \ref{assmp_feasible}.
\end{proof}

\subsubsection{Proofs for Ancillary Lemmas}
\label{sec_anc_lemmas}
\paragraph{Proof for Lemma \ref{lemma_reviseOPT}.}\

\begin{proof}
Denote  $\tilde{\mathcal{H}}_t\coloneqq \sigma(p_1,\ldots,p_t,\epsilon_1,\ldots,\epsilon_t,X_1,\ldots,X_t)$ and $\tilde{\mathcal{H}}_{0}=\sigma\left(\emptyset, \Omega\right)$, then
\begin{align*}
    r^*_T\left(\{\mathcal{P}_t\}\right)&\leq \min_{\mu\geq 0} c\mu T+\sum_{t=1}^T \mathbb{E}_t\left[ \tilde{r}^*(\mu,X_t^\top\alpha^*,X^\top_t\beta^*)\right] \nonumber \\
    &=\min_{\mu\geq 0} c\mu T+\sum_{t=1}^T \mathbb{E}_{\Bar{\mathcal{P}}_T}\left[ \tilde{r}^*(\mu,X^\top\alpha^*,X^\top\beta^*)\right] \nonumber \\
    &=\mathbb{E}\left[\sum_{t=1}^{\tau+1} \left(\min_{\mu\geq 0} c\mu +\mathbb{E}_{\Bar{\mathcal{P}}_T}\left[ \tilde{r}^*(\mu,X^\top\alpha^*,X^\top\beta^*)\right]\right)\right]+\mathbb{E}\left[\sum_{t=\tau+2}^{T} \left(\min_{\mu\geq 0} c\mu +\mathbb{E}_{\Bar{\mathcal{P}}_T}\left[ \tilde{r}^*(\mu,X^\top\alpha^*,X^\top\beta^*)\right]\right)\right] \nonumber \\
    &\leq \mathbb{E}\left[\sum_{t=1}^{T} \mathbbm{1}_{\{t\leq \tau+1\}}\cdot \left( c\mu_t +\mathbb{E}_{\Bar{\mathcal{P}}_T}\left[ \tilde{r}^*(\mu_t,X^\top\alpha^*,X^\top\beta^*)\big \vert \Tilde{\mathcal{H}}_{t-1}\right]\right) \right]\\
    &+\mathbb{E}\left[\sum_{t=\tau+2}^{T} \left(\min_{\mu\geq 0} c\mu +\mathbb{E}_{\Bar{\mathcal{P}}_T}\left[ \tilde{r}^*(\mu,X^\top\alpha^*,X^\top\beta^*)\right]\right)\right]\\
    &\leq  \mathbb{E}\left[\sum_{t=1}^{T} \mathbbm{1}_{\{t\leq \tau+1\}}\cdot \left( c\mu_t +\mathbb{E}_{t}\left[ \tilde{r}^*(\mu_t,X_t^\top\alpha^*,X_t^\top\beta^*)\big \vert \Tilde{\mathcal{H}}_{t-1}\right]\right) \right]+\sqrt{2} (\bar{p}\vee\Bar{p}^2)\Bar{g}\Bar{\theta}\mathcal{W}_T(\mathcal{P}_1,\ldots,\mathcal{P}_T)\\
    &+c\bar{p}\cdot \mathbb{E}[T-\tau-1],
\end{align*}
where the first line is from weak duality and the definition of $\tilde{r}^*$, the second line is because $\Bar{\mathcal{P}}_T$ is the uniform mixture of $\mathcal{P}_t$'s, and the third line is by splitting the summation into two parts and use $\sum_{T+1}^T(\cdot)=0$ for simplifying the notation,  the fourth line is by the suboptimal $\mu_t$ and $\mu_t,\mathbbm{1}_{\{t\leq \tau+1\}}\in \Tilde{\mathcal{H}}_{t-1}$ with the tower rule of conditional expectation, and the last line is from the following two lemmas and their proofs can be found in this subsection.
\begin{lemma}
\label{lemma_OPT_refined1}
For $t=1,\ldots,\tau+1$,
    $$\mathbb{E}_{\Bar{\mathcal{P}}_T}\left[ \tilde{r}^*(\mu_t,X^\top\alpha^*,X^\top\beta^*)\big \vert \Tilde{\mathcal{H}}_{t-1}\right]\leq \mathbb{E}_t\left[\tilde{r}^*(\mu_t,X_t^\top\alpha^*,X_t^\top\beta^*)\big\vert \Tilde{\mathcal{H}}_{t-1}\right]
    +\sqrt{2} (\bar{p}\vee\Bar{p}^2)\Bar{g}\Bar{\theta} \mathcal{W}(\mathcal{P}_t,\bar{\mathcal{P}}_T) \quad \textup{almost surely.}$$
\end{lemma}
\begin{lemma}
\label{lemma_OPT_refined2}
    $$\mathbb{E}\left[\sum_{t=\tau+2}^{T} \left(\min_{\mu\geq 0} c\mu +\mathbb{E}_{\Bar{\mathcal{P}}_T}\left[ \tilde{r}^*(\mu,X^\top\alpha^*,X^\top\beta^*)\right]\right)\right]\leq c\bar{p}\cdot \mathbb{E}[T-\tau-1].$$
\end{lemma}
Further, since $\mu_t,\mathbbm{1}_{\{t\leq \tau+1\}}\in \Tilde{\mathcal{H}}_{t-1}$, with the tower rule of conditional expectation, we have
$$\mathbb{E}\left[\sum_{t=1}^{T} \mathbbm{1}_{\{t\leq \tau+1\}}\cdot \left( c\mu_t +\mathbb{E}_{t}\left[ \tilde{r}^*(\mu_t,X_t^\top\alpha^*,X_t^\top\beta^*)\big \vert \Tilde{\mathcal{H}}_{t-1}\right]\right) \right]=\mathbb{E}\left[\sum_{t=1}^{\tau+1} \left( c\mu_t +\tilde{r}^*(\mu_t,X_t^\top\alpha^*,X_t^\top\beta^*)\right)\right],$$
and can conclude the result.
\end{proof}

\paragraph{Proof for Lemma \ref{lemma_GD_inv}.}\

\begin{proof}

\begin{align}
   \mathbb{E}\left[\sum_{t=1}^{\tau+1} \mu_t\cdot (c-g(\tilde{a}_t+\tilde{b}_t\cdot p_t) )\cdot \mathbbm{1}_{\mathcal{E}} \right]&= \mathbb{E}\left[\sum_{t=1}^{\tau+1} \left(\mu_t\cdot \left((c-D_t)+g(A^*_t+B^*_t\cdot p_t)-g(\tilde{a}_t+\tilde{b}_t\cdot p_t)+\epsilon_t  \right)\right)\cdot   \mathbbm{1}_{\mathcal{E}} \right] \nonumber \\
     &= \mathbb{E}\left[\sum_{t=1}^{\tau+1} \left(\mu_t\cdot(c-D_t+\epsilon_t)  \right)\cdot   \mathbbm{1}_{\mathcal{E}} \right] \nonumber\\
     &+\mathbb{E}\left[\sum_{t=1}^{\tau+1} \left(\mu_t\cdot(g(A^*_t+B^*_t\cdot p_t)-g(\tilde{a}_t+\tilde{b}_t\cdot p_t))  \right)\cdot   \mathbbm{1}_{\mathcal{E}} \right] \nonumber\\
    &\leq \mathbb{E}\left[\sum_{t=1}^{\tau+1} \left(\mu_t\cdot(c-D_t+\epsilon_t)  \right)\cdot   \mathbbm{1}_{\mathcal{E}} \right] +4\Bar{g}\Bar{p}\bar{\gamma}\sqrt{2e \pi dT\log\left(\frac{2d\lambda+T}{2d\lambda}\right)\log(4T^2)}, \nonumber
\end{align}
where the first line is by definition of $D_t$ and the  inequality is by similar computation steps used in proof of Theorem \ref{thm_TS_NC} with $0\leq\mu_t\leq \Bar{p}$.

Denote the random variable $\mu^*\coloneqq \argmin_{\mu\in[0,\bar{p}]} \sum_{t=1}^{\tau+1} \mu\cdot (c-D_t)$. By a standard analysis of online gradient descent on Online Convex Optimization (e.g., Theorem 3.1. in  \cite{hazan2016introduction}), we have

$$\sum_{t=1}^{\tau+1} \mu_t\cdot (c-D_t)\leq \sum_{t=1}^{\tau+1} \mu^*\cdot (c-D_t)+\frac{\bar{p}^2}{2\eta}+\frac{\eta}{2}\sum_{t=1}^{\tau+1}(c-D_t)^2 \quad \text{almost surely},$$
where $\eta=\frac{\bar{p}}{\sqrt{(\bar{D}^2+\Bar{\sigma}^2)T}}$. We claim $\sum_{t=1}^{\tau+1} \mu^*\cdot (c-D_t)\leq c\Bar{p}(\tau+1-T)$ almost surely: if $\tau=T-1$, then we can pick $\mu=0$ and thus $\sum_{t=1}^{\tau+1} \mu^*\cdot (c-D_t)\leq 0$. If $\tau<T-1$, then by its definition we have $\sum_{t=1}^{\tau+1} D_t >cT$, and thus we can pick $\mu=\bar{p}$ and $$\sum_{t=1}^{\tau+1} \mu^*\cdot (c-D_t)\leq c\Bar{p}(\tau+1-T).$$
Further,
\begin{align*}
    \mathbb{E}\left[\sum_{t=1}^{\tau+1}(c-D_t)^2\right]&=\mathbb{E}\left[\sum_{t=1}^{\tau+1}\mathbb{E}\left[(c-D_t)^2\big\vert \mathcal{H}_{t-1},p_t \right]\right]\\
    &=\mathbb{E}\left[\sum_{t=1}^{\tau+1} \mathbb{E}\left[(c-g(A^*_t+B^*_t\cdot p_t))^2+\epsilon_t^2-2\epsilon_t\cdot (c-g(A^*_t+B^*_t\cdot p_t))\big\vert \mathcal{H}_{t-1},p_t \right]\right]\\
    &=\mathbb{E}\left[\sum_{t=1}^{\tau+1} \mathbb{E}\left[(c-g(A^*_t+B^*_t\cdot p_t))^2+\epsilon_t^2\big\vert \mathcal{H}_{t-1},p_t \right]\right]\\
    &=\mathbb{E}\left[\sum_{t=1}^{\tau+1} (c-g(A^*_t+B^*_t\cdot p_t))^2\right]+\mathbb{E}\left[\sum_{t=1}^{\tau+1}\mathbb{E}\left[\epsilon_t^2\big\vert \mathcal{H}_{t-1}\right]\right]\\
    &\leq \mathbb{E}\left[(\tau+1)\Bar{D}^2+(\tau+1)\Bar{\sigma}^2\right]\\
    &\leq T(\bar{D}^2+\Bar{\sigma}^2),
\end{align*}
where the first line is by the tower rule of conditional expectation and $\mathbbm{1}_{t\leq \tau+1} \in \mathcal{H}_{t-1}$, the third line is because $\epsilon_t$ is independent with $p_t$ conditional on $\mathcal{H}_{t-1}$ and has zero mean, while $A^*_t,B^*_t\in\mathcal{H}_{t-1}$, the fifth line is by Assumption \ref{assp_g'} and the property of sub-Gaussian variables (e.g., Proposition 2.5.2 in \cite{vershynin2018high}), and the last line is by $\tau\leq T-1$ by definition.

Combine above all with $\eta=\frac{\bar{p}}{\sqrt{(\bar{D}^2+\Bar{\sigma}^2)T}}$, we can conclude
\begin{equation}
\label{eq_DG_result}
    \mathbb{E}\left[\sum_{t=1}^{\tau+1} \mu_t\cdot (c-D_t)\right]\leq c\bar{p}\mathbb{E}\left[(\tau+1-T) \right]+\Bar{p}\sqrt{(\bar{D}^2+\Bar{\sigma}^2)T}.
\end{equation}

Further,
\begin{align*}
    \mathbb{E}\left[\sum_{t=1}^{\tau+1} (\mu_t\cdot (c-D_t))\cdot \mathbbm{1}_{\mathcal{E}}\right]&=\mathbb{E}\left[\sum_{t=1}^{\tau+1} \mu_t\cdot (c-D_t)\right]-\mathbb{E}\left[\sum_{t=1}^{\tau+1} (\mu_t\cdot (c-D_t))\cdot \mathbbm{1}_{\mathcal{E}^c}\right]\\
    &\leq \mathbb{E}\left[\sum_{t=1}^{\tau+1} \mu_t\cdot (c-D_t)\right]+\mathbb{E}\left[\sum_{t=1}^{\tau+1}\bar{p}\bar{D}\cdot\mathbbm{1}_{\mathcal{E}^c}\right]+\mathbb{E}\left[  \sum_{t=1}^{\tau+1}\mu_t \epsilon_t \cdot \mathbbm{1}_{\mathcal{E}^c}\right]\\
     &\leq \mathbb{E}\left[\sum_{t=1}^{\tau+1} \mu_t\cdot (c-D_t)\right]+2\Bar{p}\bar{D}+\mathbb{E}\left[  \sum_{t=1}^{\tau+1}\mu_t \epsilon_t \cdot \mathbbm{1}_{\mathcal{E}^c}\right],
\end{align*}
where the second line is by $D_t\leq \Bar{D}+\epsilon_t$ and $\mu_t\in[0,\bar{p}]$, the third line is by \eqref{eq_bad_prob_TSC}. And thus,
\begin{align*}
    \mathbb{E}\left[\sum_{t=1}^{\tau+1} (\mu_t\cdot (c-D_t+\epsilon_t))\cdot \mathbbm{1}_{\mathcal{E}}\right]&\leq \mathbb{E}\left[\sum_{t=1}^{\tau+1} \mu_t\cdot (c-D_t)\right]+2\Bar{p}\bar{D}+\mathbb{E}\left[  \sum_{t=1}^{\tau+1}\mu_t \epsilon_t\right]\\
    &= \mathbb{E}\left[\sum_{t=1}^{\tau+1} \mu_t\cdot (c-D_t)\right]+2\Bar{p}\bar{D}+\mathbb{E}\left[  \sum_{t=1}^{T}\mathbbm{1}_{\{t\leq\tau+1\}}\mu_t \mathbb{E}[\epsilon_t\big \vert \mathcal{H}_{t-1}]\right]\\
    &\leq c\bar{p}\mathbb{E}[\tau+1-T]+\Bar{p}\sqrt{(\bar{D}^2+\bar{\sigma}^2)T}+2\Bar{p}\Bar{D},
\end{align*}
where the second line is by the tower rule and  $\mu_t,\mathbbm{1}_{\{t\leq\tau+1\}}\in \mathcal{H}_{t-1}$ and last line is by \eqref{eq_DG_result}. And combine all above we can conclude the lemma.
\end{proof}

\paragraph{Proof for Lemma \ref{lemma_OPT_refined1}.}\

\begin{proof}
    We first note given any $\mu_t\in[0,\bar{p}]$, $\Tilde{r}^*(\mu_t,x^\top\alpha^*,x^\top\beta^*)$ is Lipschitz in $x$: for $x,x' \in \mathcal{X}$, without loss of generality, we assume $\Tilde{r}^*(\mu_t,x^\top\alpha^*,x^\top\beta^*)\geq\Tilde{r}^*(\mu_t,(x')^\top\alpha^*,(x')^\top\beta^*)$ and denote $\tilde{p}^*_t(x)\coloneqq \argmax_{p\in[\mu_t,\Bar{p}]}\Tilde{r}(p;\mu_t,x^\top \alpha^*,x^\top\beta^*)$, then
\begin{align*}
  &\Tilde{r}^*(\mu_t,x^\top\alpha^*,x^\top\beta^*)-\Tilde{r}^*(\mu_t,(x')^\top\alpha^*,(x')^\top\beta^*)\\
  =&(\tilde{p}^*_t(x)-\mu_t)g(x^\top\alpha^*+x^\top \beta^*\cdot \tilde{p}^*_t(x))-(\tilde{p}^*_t(x')-\mu_t)g((x')^\top\alpha^*+(x')^\top \beta^*\cdot \tilde{p}^*_t(x'))\\
    \leq& (\tilde{p}^*_t(x')-\mu_t)\cdot \big \vert g(x^\top\alpha^*+x^\top \beta^*\cdot\tilde{p}^*_t(x'))-g((x')^\top\alpha^*+(x')^\top \beta^*\cdot \tilde{p}^*_t(x'))) \big \vert\\
    \leq& \Bar{g}\Bar{p}|(x-x',x-x')^\top (\alpha^*, \tilde{p}^*_t(x')\beta^*)|\\
    \leq&\sqrt{2} (\bar{p}\vee\Bar{p}^2)\Bar{g}\Bar{\theta}\|x-x'\|_2,
\end{align*}
where the first line is by definitions, the second line is by the optimality of $\tilde{p}^*_t(x)$, and the last two lines are by Assumption \ref{assp_bound} and \ref{assp_g'} and Holder's inequality. Thus, by the dual formulation of Wasserstein distance \citep{kantorovich1958space}, we have for $t=1,\ldots,\tau+1$,
$$\mathbb{E}_{\Bar{\mathcal{P}}_T}\left[\Tilde{r}^*(\mu_t;X^\top\alpha^*,X^\top\beta^*)\big\vert \Tilde{\mathcal{H}}_{t-1}\right]-\mathbb{E}_t\left[\Tilde{r}^*(\mu_t;X_t^\top\alpha^*,X_t^\top\beta^*)\big\vert \Tilde{\mathcal{H}}_{t-1}\right]
    \leq\sqrt{2} (\bar{p}\vee\Bar{p}^2)\Bar{g}\Bar{\theta} \mathcal{W}(\mathcal{P}_t,\bar{\mathcal{P}}_T). $$
\end{proof}

\paragraph{Proof for Lemma \ref{lemma_OPT_refined2}.}\

\begin{proof}
Define the demand optimization problem:
\begin{align}
    \label{eq_R_barP_Det_Demand}
    r^*_{\bar{\mathcal{P}}_T}\coloneqq \max&  \ \mathbb{E}_{\Bar{\mathcal{P}}_T}\left[D(X)\cdot \frac{g^{-1}(D(X))-X^\top \alpha^*}{X^\top \beta^*}\right] \nonumber \\
\text{s.t.} &   \ \mathbb{E}_{\Bar{\mathcal{P}}_T}\left[D(X)\right] \leq c, \nonumber \\
& D(x)\in[g(x^\top\alpha^*+x^\top \beta^*\cdot \Bar{p}),g(x^\top\alpha^*+x^\top \beta^*\cdot \underline{p})] \quad \forall x\in \mathcal{X}, \nonumber \\
& D(x): \mathcal{X}\rightarrow [0,\Bar{D}] \ \text{is a measurable function.}\  \nonumber
\end{align}
By Assumption \ref{assmp_feasible}, $r^*_{\bar{\mathcal{P}}_T}$ is a concave maximization problem and Slater’s condition holds. Thus, by strong duality (Section 5.2.3 in \cite{boyd2004convex}) and replacing the optimal demand function with its corresponding pricing function, we get the dual problem $\min_{\mu\geq 0} c\mu +\mathbb{E}_{\Bar{\mathcal{P}}_T}\left[ \tilde{r}^*(\mu;X^\top\alpha^*,X^\top\beta^*)\right]= r^*_{\bar{\mathcal{P}}_T}\leq c\Bar{p}$. By summing this inequality on both sides over $t=\tau+2,\ldots,T$ when $\tau\leq T-2$ and noting $\tau=T-1$ the result still holds, and the proof is finished.
\end{proof}

\section{More Algorithm Discussions and Extensions}

\subsection{Discussion on the UCB Pricing Algorithm}
\label{apx:diss_UCB}

\subsubsection{Discussion on the Original (GLM-)UCB Algorithm}
The main difference between Algorithm \ref{alg_UCB} and GLM-UCB \citep{filippi2010parametric} is that the estimation of $\hat{\theta}_t$ comes from the quasi-MLE on the observed demands as in \eqref{eq_quasi_reg}, instead of the observed rewards (revenues) as in GLM-UCB. This difference is based on the special structure of the pricing problem, i.e., the misalignment of the direct observations (demands) and the rewards (revenues), which is crucial to reducing the estimation error. Specifically, if the price $p_t>1$, the error term on revenue $r_t$ will become $p_t\bar{\sigma}^2$-sub-Gaussian, which results in larger variances and estimation errors. Indeed, when $g$ is the identical mapping and $\epsilon_{t'}$'s are i.i.d., it is well known that the weighted least square estimation on revenues $r_{t'}$'s with weights $w_{t'}=\frac{1}{p_{t'}^2}$ will lead to the best linear unbiased estimated of $\theta^*$, which just recovers the original linear regression on demands $D_{t'}$'s. This difference will be further exemplified in the case of Thompson sampling. As we noted in the analysis of Algorithm \ref{alg_TS}, this special structure is the key to reducing the dependence on $d$ and relaxing the convexity assumption.

\subsubsection{Sample Complexity on the Monte Carlo Approximation of UCB Optimization}
Given a target approximation error $\epsilon$ such that $\|(\alpha_t,\beta_t)-(\alpha^{\text{MC}}_t,\beta^{\text{MC}}_t)\|_2\leq \epsilon$, noting that $\Theta_t$'s volume is $\text{Vol}(\Theta_t)\propto\text{det}(M_{t})$, and a ball with radius $\epsilon$ has volume $\propto \epsilon^d$, by standard covering number computation (Proposition 4.2.12 in \citet{vershynin2018high}), the covering number of $\Theta_t$ is $\propto \frac{\text{det}(M_{t})}{\epsilon^d}$. We assume the covering number is achieved by the set of $\epsilon-$balls denoted as $\mathcal{B}_{\epsilon}$, then by the definition of  the covering number, there needs at least one Monte Carlo sample in every ball in $\mathcal{B}_{\epsilon}$ if $\|(\alpha_t,\beta_t)-(\alpha^{\text{MC}}_t,\beta^{\text{MC}}_t)\|_2\leq \epsilon$. Thus, the expected number of samples needed for the Monte Carlo to achieve $\|(\alpha_t,\beta_t)-(\alpha^{\text{MC}}_t,\beta^{\text{MC}}_t)\|_2\leq \epsilon$ is $\propto\frac{\text{det}(M_{t})}{\epsilon^d}$, which follows the standard analysis of coupon collector's problem \citep{motwani1995randomized} and suffers from the curse of dimensionality.

\subsection{Discussion on Original Thompson Sampling Algorithm}
\label{apx:TS_vs_newTS}
For completeness, here we introduce the original Thompson sampling in \cite{abeille2017linear} with some ancillary lemmas needed in the proof of Theorem \ref{thm_TS_NC}.  We first state some requirements for the sampling distribution.

\begin{definition}[Sampling distribution \citep{abeille2017linear}]
\label{assp: AbeilleTS}
A distribution $\mathcal{D}^{\text{TS}}$ is \textit{suitable} for Thompson sampling if it is a multivariate distribution on $\mathbb{R}^{2d}$ absolutely continuous with respect to the Lebesgues measure which satisfies the following properties:
\begin{itemize}
    \item (anti-concentration) there exists a positive probability $q$ such that for any $u \in \mathbb{R}^{2d}$ with $\|u\|_2=1$,
    $$\mathbb{P}_{\eta \sim \mathcal{D}^{\text{TS}}} (u^\top \eta\geq 1)\geq q,$$
    \item (concentration) there exist positive constants $c, c'$ such that $\forall \delta \in (0,1)$,
    $$\mathbb{P}_{\eta \sim \mathcal{D}^{\text{TS}}}\left(\|\eta\|_2\leq \sqrt{2cd\log\frac{2c'd}{\delta}}\right)\geq 1-\delta .$$
\end{itemize}
\label{def_xi}
\end{definition}

As shown in \citet{abeille2017linear}, the Gaussian distribution $\eta \sim \mathcal{N}(0,I_{2})$ that we use in Algorithm \ref{alg_TS} satisfies the above definition with $d=1$, $\delta=1/T^2$, $c=c'=2$ and $q=\frac{1}{4\sqrt{e\pi}}$.

As a comparison, the direct application of the original Thompson Sampling in \citet{abeille2017linear} on the pricing problem is shown as below:

\begin{algorithm}[H]
\centering
\caption{Original Thompson Sampling \citep{abeille2017linear} for Dynamic Pricing}
\label{alg_TS_exante}
\begin{algorithmic}
\STATE{\textbf{Input:} Regularization parameter $\lambda$.}
\FOR{$t=1,\ldots,T$}
\STATE{Compute the estimator $\hat{\theta}_{t-1}$ by \eqref{eq_quasi_reg} and observe the covariates $x_t$.}
\STATE{Sample $\eta_t\sim \mathcal{N}(0,I_{2d})$ and compute the parameter
$$\tilde{\theta}_{t-1}= (\Tilde{\alpha}_{t-1},\Tilde{\beta}_{t-1}) \coloneqq \hat{\theta}_{t-1}+\left(2\sqrt{\lambda}\bar{\theta}+ \frac{2\bar{\sigma}}{\underline{g}}\sqrt{2\log T+2d\log\left(\frac{2d\lambda +T}{2d\lambda}\right)}\right)M_{t-1}^{-1/2}\eta_t.$$
Set the price by
\begin{equation*}
    p_{t}=\argmax_{p\in [\underline{p},\bar{p}]} r(p;x^\top_t \Tilde{\alpha}_{t-1},x^\top_t\Tilde{\beta}_{t-1}),
\end{equation*}
and observe the demand $D_t$.
}
\ENDFOR
\end{algorithmic}
\end{algorithm}
With a modification of Assumption \ref{assp_g} which enlarges the domain width of $g$ from $O(d)$ to $O(d^{3/2})$, we can get the original Thompson Sampling's regret bound as follows:

\begin{assumption}[Properties of $g(\cdot)$]
Let $$\tilde{\Theta}:=\left\{\theta \in \mathbb{R}^{2d}:  \|\theta- \tilde{\theta}\|_2\leq 2\sqrt{d\log (4dT^2)}\left(2\bar{\theta}+ \frac{2\bar{\sigma}}{\underline{g}}\sqrt{2\log T+2d\log\left(\frac{2d+T}{2d}\right)}\right) \text{ for some } \tilde{\theta}\in \Theta\right\}$$ where $\Theta$ is defined in Assumption \ref{assp_bound}. We assume $g(z)$ is strictly increasing, differentiable, and there exist constants $\underline{g}, \bar{g} \in \mathbb{R}$ such that $0<\underline{g}\leq g'(z)\leq \bar{g} <\infty$ for all $z=x^\top\alpha+x^\top\beta \cdot p$ where $x\in \mathcal{X}$, $\theta = (\alpha,\beta) \in \Tilde{\Theta}$, and $p \in [\underline{p},\bar{p}]$.
\label{assp_g''}
\end{assumption}

\begin{theorem}
Under Assumption \ref{assp_bound}, \ref{assp_error}, \ref{assp_g''}, with any $T\geq 6$, if we choose the regularization parameter  $\lambda=1$, the regret of Algorithm \ref{alg_TS_exante} can be bounded by
$$36d\bar{g}\bar{p} \bar{\gamma} \sqrt{2T\log (4dT^2) \log\left(\frac{2d+T}{2d}\right)  }+2\bar{p}\bar{D}=\tilde{O}\left(d^{\frac{3}{2}} \sqrt{T}\right),$$
where $\Bar{\gamma}=2\bar{\theta}+ \frac{2\bar{\sigma}}{\underline{g}}\sqrt{2\log T+2d\log\left(\frac{2d+T}{2d}\right)}$.
\label{thm_TS_exante}
\end{theorem}
We omit the proof here since it is largely following \citet{abeille2017linear}.  Compared to Algorithm \ref{alg_TS}, there is an extra factor of $\sqrt{d}$ in both the Assumption \ref{assp_g''} (for domain enlargement) and regret bound. To the best of our knowledge, this extra factor is also inevitable for the existing analyses of Thompson sampling algorithms on the linear bandits problem \citep{agrawal2013thompson,abeille2017linear}.

\subsection{Pricing with Approximate Quasi-MLE Estimator}
\label{apx:appx_MLE}

As discussed in Section \ref{sec_Quasi_MLE}, the quasi-MLE problem \eqref{eq_quasi_reg} generally cannot be solved in closed-form. In this subsection, we first show how to compute its approximate (optimal) solution $\check{\theta}_{t}$ through standard projected gradient descent (Algorithm \ref{alg_GD_theta}) with bounded approximation gap (Proposition \ref{prop_GD_theta}). We further show that both the UCB and the Thompson sampling pricing algorithms can also be adapted to the case of only accessing to the approximate solution $\check{\theta}_{t}$ (Algorithm \ref{alg_UCB_appx} and \ref{alg_TS_appx} respectively) with regret upper bounds (Theorem \ref{thm:UCB_UB_2} and Theorem \ref{thm_TS_UB_appx} respectively). Finally, we provide a numerical experiment (Figure \ref{fig:reg_error}) to showcase the influence on the regret by using such approximate solutions.

\subsubsection{(Approximate) Optimization Algorithm of Quasi-MLE problem.}\

In general, the quasi-MLE problem  \eqref{eq_quasi_reg} does not have a closed-form solution. However, by introducing the regularization term, the (minus of) objective function $\lambda \underline{g} \|\theta\|_2^2/2-\sum_{t'=1}^{t} l_{t'}(\theta)$ is strongly convex and thus \eqref{eq_quasi_reg} can be solved by the standard projected gradient descent as shown in Algorithm \ref{alg_GD_theta} (we refer readers to Section 3.4.1 in \cite{bubeck2015convex} for more details).

\begin{algorithm}[ht!]
\centering
\caption{Projected Gradient Descent for \eqref{eq_quasi_reg}}
\label{alg_GD_theta}
\begin{algorithmic}
\STATE{\textbf{Input:} Total update steps $K$, initial point $\theta'_1\in\Theta$.}
\FOR{$k=1,\ldots,K-1$}
\STATE{Update
$$\theta'_{k+1}=\text{Proj}_{\Theta}\left(\theta'_{k}-\eta_k \left(\sum_{t'=1}^t(D_{t'}-g(z^\top_{t'}\theta'_{k})\cdot z_{t'}-\lambda\underline{g}\theta'_{k}\right)  \right),$$
where $\text{Proj}_{\Theta}(\theta)$ is the projection function to project $\theta$ in $\Theta$ and $\eta_k=\frac{2}{\underline{g}\lambda_{\text{min}}(M_t)(k+1)}$ (here $\lambda_{\min}(M_t)\geq \lambda$ is the minimum eigenvalue of $M_t$).}

\ENDFOR
\STATE{\textbf{Output:} Approximate solution $\check{\theta}_t\coloneqq  \frac{\sum_{k=1}^K 2k\theta'_k}{K(K+1)}$.}
\end{algorithmic}
\end{algorithm}

The following Proposition \ref{prop_GD_theta} shows that the approximation gap between the output approximate solution $\check{\theta}_{t}$ of Algorithm \ref{alg_GD_theta} and the optimal solution $\hat{\theta}_{t}$ of the quasi-MLE problem \eqref{eq_quasi_reg}, which measures the gap of their values of \eqref{eq_quasi_reg}, can be bounded as a function of the total update steps $K$:

\begin{proposition}
\label{prop_GD_theta}
    For any $\lambda>0$ and $t=1,\ldots,T-1$, function $ \frac{\lambda \underline{g}\|\theta\|_2^2}{2} -\sum_{t'=1}^{t} l_{t'}(\theta)$ is $\underline{g}\lambda_{\text{min}}(M_t)$-strongly convex, where $\lambda_{\min}(M_t)\geq \lambda$ is the minimum eigenvalue of $M_t$.  By applying Algorithm \ref{alg_GD_theta} with $K$ steps, its output $\check{\theta}_{t}$ is an approximate solution of \eqref{eq_quasi_reg} such that
    $$\sum_{t'=1}^t l_{t'}(\hat{\theta}_t)-\frac{\lambda \underline{g}\|\hat{\theta}_t\|^2_2}{2} -\sum_{t'=1}^t l_{t'}(\check{\theta}_t)+\frac{\lambda \underline{g}\|\check{\theta}_t\|^2_2}{2} \leq \frac{2L_t^2}{\underline{g} \lambda_{\min}(M_t)(K+1)},$$
     where $L_t\coloneqq \max_{\theta\in\Theta} \|\sum_{t'=1}^t(D_{t'}-g(z_{t'}^\top \theta))\cdot z_{t'}-\lambda \underline{g}\theta \|_2$ and $\mathbb{E}[L_t^2]\leq 12\bar{g}^2\Bar{\theta}^2t^2+3\Bar{\sigma}^2t+3\lambda^2\underline{g}^2\bar{\theta}^2.$
\end{proposition}

 Although $\mathbb{E}[L_t^2]$ increases squarely in $t$, the curvature of $\sum_{t'=1}^{t} l_{t'}(\theta)$ and thus $\lambda_{\min}(M_t)$ will also change based on the collated data. As an example, if $x_t$'s are i.i.d. from some distribution with a positive definite covariance matrix, $ -\frac{\lambda \underline{g}\|\theta\|_2^2}{2} +\sum_{t'=1}^{t} l_{t'}(\theta)$ will be strongly convex with parameter $\lambda_{\min}(M_t)=\Omega(t)$ in expectation and thus the above convergence rate is roughly $O(t/K)$ for all $t=1,\ldots,T$.

\paragraph{Proof of Proposition \ref{prop_GD_theta}.}\

\begin{proof}
We first prove the strong convexity in $\Theta$ of function $\lambda \underline{g} \|\theta\|_2^2/2-\sum_{t'=1}^{t} l_{t'}(\theta)$. By the second-order Taylor's expansion, for any $\theta, \theta' \in\Theta$, recall $Q_t(\theta)=\sum_{t'=1}^{t} l_{t'}(\theta)$,
\begin{align*}
    &\frac{\lambda \underline{g} \|\theta\|_2^2}{2}-Q_t(\theta)-\frac{\lambda \underline{g} \|\theta'\|_2^2}{2}+Q_t(\theta')\\
   =&\left\langle -\nabla Q_t(\theta)+\lambda\underline{g}\theta,\theta-\theta'\right\rangle-\frac{1}{2}\left\langle\theta-\theta',\left(-\nabla^2 Q_t(\tilde{\theta})+\lambda \underline{g} I_{2d}\right)(\theta-\theta')\right\rangle \nonumber\\
\leq &\left\langle -\nabla Q_t(\theta)+\underline{g}\theta,\theta-\theta'\right\rangle-\frac{ \underline{g} \lambda_{\text{min}}(M_t)}{2}\|\theta-\theta'\|^2_2, \nonumber
\end{align*}
where $\tilde{\theta}\in \Theta$ is some point on the line segment between $\theta$ and $\theta^*$ by noting $\Theta$ is convex and the second line is from Assumption \ref{assp_g}:
\begin{equation*}
    -\nabla^2 Q_t(\tilde{\theta})+\lambda \underline{g} I_{2d}= \sum_{t'=1}^t g'\left( z_t^\top \tilde{\theta} \right) z_{t'}z_{t'}^\top +\lambda \underline{g} I_{2d}\geq \underline{g}\cdot \sum_{t'=1}^{t}z_{t'}z_{t'}^\top+\lambda \underline{g} I_{2d}\geq \underline{g} M_{t}.
\end{equation*}
Thus, we can conclude the strong convexity. Further, noting $\lambda \underline{g} \|\theta\|_2^2/2-Q_t(\theta)$ is $L_t$-Lipschitz with respect to Euclidean norm since $L_t$ upper bounds the Euclidean norm of gradient $\lambda\underline{g}\theta-\nabla Q_t(\theta)$ by definition, we can directly apply the standard convergence analysis of projected gradient descent on strongly convex and Lipschitz objective function (e.g., Theorem 3.9 in \cite{bubeck2015convex}), and get the second part of the proposition.

Now for bounding $\mathbb{E}[L_t^2]$, note
\begin{align*}
    L_t&=\max_{\theta\in\Theta} \left\|\sum_{t'=1}^t(D_{t'}-g(z_{t'}^\top \theta))\cdot z_{t'}-\lambda \underline{g}\theta \right\|_2\\
    &\leq \max_{\theta\in\Theta} \sum_{t'=1}^t\|(g(z_{t'}^\top \theta^*)-g(z_{t'}^\top \theta))\cdot z_{t'}\|_2+\left\|\sum_{t'=1}^t\epsilon_{t'}z_{t'}\right\|_2+\lambda \underline{g}\|\theta \|_2\\
    &\leq \max_{\theta\in\Theta}  \bar{g}\|\theta^*-\theta\|_2\sum_{t'=1}^t\|z_{t'}\|_2^2+\left\|\sum_{t'=1}^t\epsilon_{t'}z_{t'}\right\|_2+\lambda \underline{g}\|\theta \|_2\\
    &\leq 2\bar{g}\Bar{\theta}t+\left\|\sum_{t'=1}^t\epsilon_{t'}z_{t'}\right\|_2+\lambda\underline{g}\bar{\theta}
\end{align*}
where the second line is by triangle inequality and definition of $D_{t'}$, the third line is by Holder's inequality with Assumption \ref{assp_bound}, and the last line is still from Assumption \ref{assp_bound}. Further, since
$$\mathbb{E}\left[\left\|\sum_{t'=1}^t\epsilon_{t'}z_{t'}\right\|_2^2\right]=\mathbb{E}\left[\sum_{t'=1}^t\epsilon_{t'}^2\left\|z_{t'}\right\|_2^2\right]\leq \mathbb{E}\left[\sum_{t'=1}^t\epsilon_{t'}^2\right]\leq t\bar{\sigma}^2,$$
by Assumption \ref{assp_error} and properties of sub-Gaussian variables (Proposition 2.5.2 in \cite{vershynin2018high}), with elementary inequality $(x+y+z)^2\leq 3(x^2+y^2+z^2)$ we can conclude the result.
\end{proof}
\subsubsection{Pricing Algorithms with Approximate Quasi-MLE Estimator}
We first denote  $\Delta_t, t=1,...,T-1$ as the upper bounds of approximation gap between the optimal solution $\hat{\theta}_t$ of \eqref{eq_quasi_reg} and its approximate solution $\check{\theta}_t$ (from Algorithm \ref{alg_GD_theta} with some $K$) as the following:
\begin{equation}
    \label{eq_delta_gap}
     \sum_{t'=1}^t l_{t'}(\hat{\theta}_t)-\frac{\lambda\underline{g}\|\hat{\theta}_t\|^2_2}{2}-\sum_{t'=1}^t l_{t'}(\check{\theta}_t)+\frac{\lambda\underline{g}\|\check{\theta}_t\|^2_2}{2}\leq \Delta_t \quad \forall t=1,...,T-1,
\end{equation}
which can be upper bounded as a function of $K$ by Proposition \ref{prop_GD_theta}. The following Algorithm \ref{alg_UCB_appx} and Algorithm \ref{alg_TS_appx} show how to adapt the UCB (Algorithm \ref{alg_UCB}) and the Thompson sampling   (Algorithm \ref{alg_TS}) pricing algorithms to the case with the approximate solution $\check{\theta}_t$, with provable regret upper bounds as demonstrated by Theorem \ref{thm:UCB_UB_2} and Theorem \ref{thm_TS_UB_appx} respectively.

To account for this approximation gap, the confidence bound needs to be enlarged in the UCB pricing (Algorithm \ref{alg_UCB_appx}), and the sampling distribution should also be inflated in the Thompson sampling pricing (Algorithm \ref{alg_TS_appx}). We also remark that the Assumption \ref{assp_g'} for Algorithm \ref{alg_TS_appx} should be slightly changed in that the set $\tilde{\Theta}$ should be redefined by $$\tilde{\Theta}\coloneqq \left\{\theta \in \mathbb{R}^{2d}:  \exists \tilde{\theta}\in \Theta, \|\theta-\tilde{\theta}\|_2\leq  4\bar{\theta}\sqrt{\log (4T^2)}+ \frac{4\bar{\sigma}}{\underline{g}}\sqrt{\left(2\log T+2d\log\left(\frac{2d+T}{2d}\right)\right)\log (4T^2)}+ 2\sqrt{\log (4T^2)}\bar{\Delta}\right\},$$
where $\bar{\Delta}=\max_{t=1,\ldots,T-1}\Delta_t$. This is because the approximate quasi-MLE solution may further enlarge the sampling region.

\begin{algorithm}[ht!]
\centering
\caption{UCB Pricing with Approximation}
\label{alg_UCB_appx}
\begin{algorithmic}
\STATE{\textbf{Input:} Regularization parameter $\lambda$.}
\FOR{$t=1,...,T$}
\STATE{Compute the estimators $\check{\theta}_{t-1}$ of \eqref{eq_quasi_reg} with approximation gap upper bounded by $\Delta_{t-1}$ and its confidence set
$$\Theta_{t}:=\left\{\theta\in \Theta: \left\|\theta-\check{\theta}_{t-1}\right\|_{M_{t-1}}\leq 2\sqrt{\lambda}\bar{\theta}+ \frac{2\bar{\sigma}}{\underline{g}}\sqrt{2\log T+2d\log\left(\frac{2d\lambda +T}{2d\lambda}\right)}+\sqrt{\frac{2}{\underline{g}}\Delta_{t-1}} \right\}.$$}
\STATE{Observe covariates $x_t$ and choose the UCB parameter which maximizes the expected revenue:
\begin{equation}
\label{UCB_opt}
   (\alpha_t,\beta_t)=\argmax_{(\alpha,\beta)\in \Theta_{t}} \ \  r^*\left(x_t^\top\alpha,x_t^\top\beta\right).
\end{equation}}
\STATE{Set the price by
\begin{equation*}
    p_{t}=\argmax_{p\in [\underline{p},\bar{p}]} r(p;x^\top_t\alpha_t,x^\top_t\beta_t).
\end{equation*}
}
\ENDFOR
\end{algorithmic}
\end{algorithm}

\begin{theorem}
\label{thm:UCB_UB_2}
Under Assumptions \ref{assp_bound}, \ref{assp_g} and \ref{assp_error}, with any sequence $\{x_t\}_{t=1,...,T}$, if we choose the regularization parameter $\lambda=1$, the regret of Algorithm \ref{alg_UCB_appx} is upper bounded by
$$4\bar{p}\bar{g}\left(\bar{\gamma}+\sqrt{\frac{2\bar{\Delta}}{\underline{g}}}\right)\sqrt{Td\log\left(\frac{2d+T}{2d}\right)}+2\bar{p}\bar{D}=\tilde{O}\left(d \sqrt{T}\right),$$
where $\bar{\Delta}=\max_{t=1,\ldots,T-1} \Delta_t$ and $\bar{\gamma}=2\bar{\theta}+ \frac{2\bar{\sigma}}{\underline{g}}\sqrt{2\log T+2d\log\left(\frac{2d+T}{2d}\right)}$
represents an upper bound for the confidence volume.
\end{theorem}

\begin{algorithm}[ht!]
\centering
\caption{Thompson Sampling Pricing with Approximation}
\label{alg_TS_appx}
\begin{algorithmic}
\STATE{\textbf{Input:} Regularization parameter $\lambda$.}
\FOR{$t=1,\ldots,T$}
\STATE{Compute the estimator $\check{\theta}_{t-1}$ of \eqref{eq_quasi_reg} with approximation gap upper bounded by $\Delta_{t-1}$, observe feature $x_t$.}
\STATE{Compute the estimator $(\hat{a}_t,\hat{b}_t)\coloneqq (x_t^\top \check{\alpha}_{t-1}, x_t^\top \check{\beta}_{t-1})$.}
\STATE{Sample $\eta_t\sim \mathcal{N}(0,I_{2})$ and compute the parameter
\begin{equation}
    (\tilde{a}_{t},\tilde{b}_t)\coloneqq (\hat{a}_t,\hat{b}_t)+\left(2\sqrt{\lambda}\bar{\theta}+ \frac{2\bar{\sigma}}{\underline{g}}\sqrt{2\log T+2d\log\left(\frac{2d\lambda +T}{2d\lambda}\right)}+\sqrt{\frac{2}{\underline{g}}\Delta_{t-1}}\right)\tilde{M}_{t-1}^{-1/2}\eta_t.
\end{equation}
Set the price by
\begin{equation*}
    p_{t}=\argmax_{p\in [\underline{p},\bar{p}]} r(p;\tilde{a}_t,\tilde{b}_t),
\end{equation*}
and observe the demand $D_t$.}
\ENDFOR
\end{algorithmic}
\end{algorithm}

\begin{theorem}
\label{thm_TS_UB_appx}
Under Assumption \ref{assp_bound}, \ref{assp_error},  \ref{assp_g'}, with any sequence $\{x_t\}_{t=1,\ldots,T}$ and any $T\geq 6$, if we choose the regularization parameter  $\lambda=1$, the regret of Algorithm \ref{alg_TS_appx} can be bounded by
$$36\Bar{g}\Bar{p}\left(\bar{\gamma}+\sqrt{\frac{2\Bar{\Delta}}{\underline{g}}}\right)\sqrt{2e \pi dT\log\left(\frac{2d\lambda+T}{2d\lambda}\right)\log(4T^2)}+2\bar{D}\bar{p}=\tilde{O}\left(d \sqrt{T}\right),$$
where $\bar{\Delta}=\max_{t=1,\ldots,T-1} \Delta_t$ and $\bar{\gamma}=2\bar{\theta}+ \frac{2\bar{\sigma}}{\underline{g}}\sqrt{2\log T+2d\log\left(\frac{2d+T}{2d}\right)}$.
\end{theorem}

Theorem \ref{thm:UCB_UB_2} and Theorem \ref{thm_TS_UB_appx} provide regret upper bounds for Algorithm \ref{alg_UCB_appx} and Algorithm \ref{alg_TS_appx} respectively.  The proof ideas of them are almost identical to those of Theorem \ref{thm:UCB_UB} and Theorem \ref{thm_TS_NC} by using the following lemma captures the distance between the approximate solution $\check{\theta}_t$ and the optimal solution $\hat{\theta}_t$. The lemma justifies the choice of the confidence set in Algorithm \ref{alg_UCB_appx} and the sampling step in Algorithm \ref{alg_TS_appx}, and can be implied from the optimality condition. The exact proof can be founded in the end of this subsection.

\begin{lemma}
Recall that $\hat{\theta}_t$ is the optimal solution to the optimization problem \eqref{eq_quasi_reg}. For any $\theta\in\Theta$, we have
$$\sum_{t'=1}^t l_{t'}(\hat{\theta}_t)-\frac{\lambda\underline{g}\|\hat{\theta}_t\|^2_2}{2}-\sum_{t'=1}^t l_{t'}(\theta)+\frac{\lambda\underline{g}\|\theta\|^2_2}{2}\geq \frac{1}{2}\underline{g}\|\hat{\theta}_t-\theta\|_{M_{t}}^2.$$
\label{lemma_opt_gap}
\end{lemma}

\paragraph{Numerical experiment.}\
\begin{figure}[ht!]
    \centering
    \includegraphics[width=0.5\textwidth]{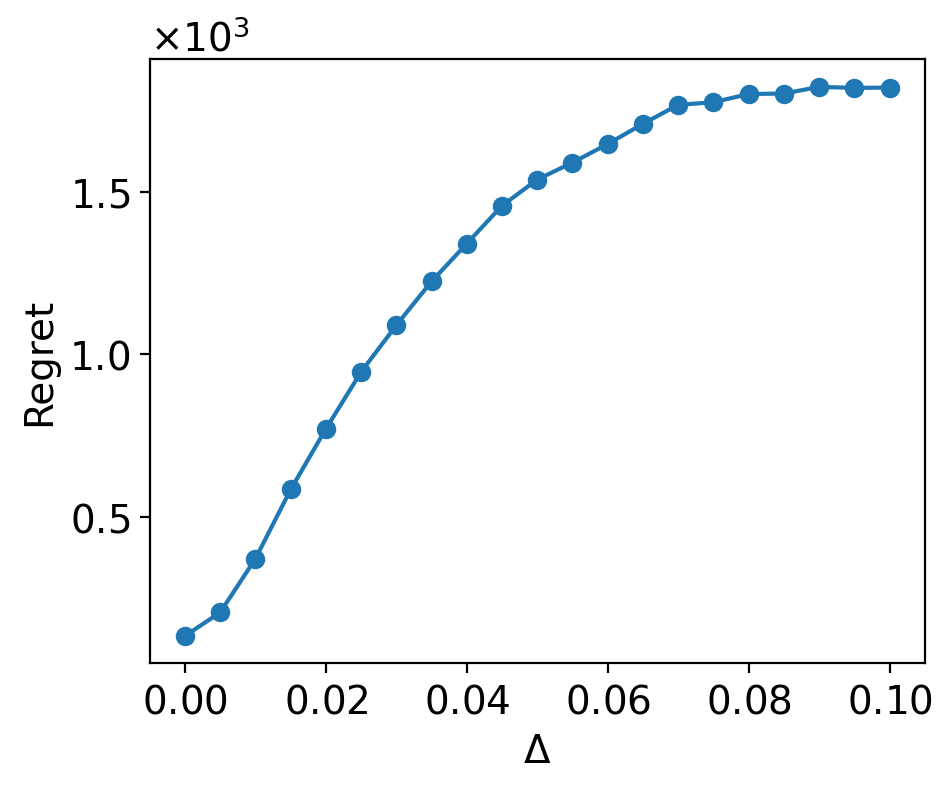}
    \caption{Regret when varying $\Delta$ in Algorithm \ref{alg_TS_appx}. }
    \label{fig:reg_error}
\end{figure}

Figure \ref{fig:reg_error} demonstrates the performances of Algorithm \ref{alg_TS_appx} on the same demand setting as \S \ref{sec_num_diss} with different approximation gaps $\bar{\Delta}$. We set $$\check{\theta}_t=\hat{\theta}_t+\Delta\times [1,1,\ldots,1]$$
as the approximate estimators used in Algorithm \ref{alg_TS_appx} and thus by simple computation $\bar{\Delta}\propto \Delta$.  We plot the regrets at $T=1500$ and $d=6$ over $\Delta=0,0.005,0.01,\ldots,0.1$. From Figure \ref{fig:reg_error},  the regret is roughly $ \propto \sqrt{\Delta}$, which is consistent with the analysis of Algorithm \ref{alg_TS_appx} in Theorem \ref{thm_TS_UB_appx}. Here the sampling step of Algorithm \ref{alg_TS_appx} is by $$ (\hat{a}_t,\hat{b}_t)+\left(\frac{\sqrt{d}}{10}+\frac{T\Delta }{10}\right)\tilde{M}^{-1/2}_{t-1}\eta_t.$$

\paragraph{Proof of Lemma \ref{lemma_opt_gap}.}\

\begin{proof}
Recall that the regularized quasi-MLE at time $t$ is defined as
$$Q_t(\theta)-\frac{\lambda\underline{g}\|\theta\|^2_2}{2},$$
where $Q_t(\theta)=\sum_{t'=1}^t l_t(\theta)$, with Hessian matrix
$$\sum_{t'=1}^t -g'(z_{t'}^\top \theta)z_{t'}z_{t'}^\top-\lambda \underline{g} I_{2d},$$
which is negative definite by the assumption that $g'(z_t^\top \theta)>0$ (in the feasible domain of related parameters). Thus, the regularized quasi-MLE is concave in $\theta$.  We can then perform a second-order Taylor's expansion around the optimal solution $\hat{\theta}_t$ in $\Theta$ with any point $\theta\in \Theta$,
\begin{align*}
    &Q_t(\hat{\theta}_t)-\frac{\lambda\underline{g}\|\hat{\theta}_t\|^2_2}{2}-Q_t(\theta)+\frac{\lambda\underline{g}\|\theta\|^2_2}{2}\\
    =&-\left\langle \nabla Q_t(\hat{\theta}_t)-\lambda\underline{g}\hat{\theta}_t,\theta-\hat{\theta}_t\right\rangle-\frac{1}{2}\left\langle\hat{\theta}_t-\theta,\left(\nabla^2 Q_t(\theta')-\lambda \underline{g} I_{2d}\right)(\hat{\theta}_t-\theta)\right\rangle \\
    \geq& -\frac{1}{2}\left\langle\hat{\theta}_t-\theta,\left(\nabla^2 Q_t(\theta')-\lambda \underline{g} I_{2d}\right)(\hat{\theta}_t-\theta)\right\rangle\\
    \geq & \frac{1}{2}\underline{g}\|\hat{\theta}_t-\theta\|_{M_{t}}^2,
\end{align*}
where $\theta'\in \Theta$ is a point between $\theta$ and $\hat{\theta}_t$, the first inequality is by the concavity and the optimality of $\hat{\theta}_t$ in the compact set $\Theta$, and the second inequality is by $-\nabla^2 Q_t(\theta')\geq \underline{g} \sum_{t'=1}^tz_{t'} z_{t'}^\top$.
\end{proof}

\section{Appendix for Numerical Experiments}
\label{sec_appn_num}
\paragraph{Benchmarks.}\

\begin{itemize}
    \item \textbf{CILS.} By covariate-free constrained iterated least square (CILS) algorithm \citep{keskin2014dynamic}, the price $p_t$ is set by
\begin{equation*}
p_t=\begin{cases}
 \bar{p}_{t-1}+\text{sgn}(\delta_t)\kappa t^{-\frac{1}{4}}, &\text{if } |\delta_t|<\kappa t^{-\frac{1}{4}} ,  \\
p^*(\hat{\theta}_t), &\text{otherwise },
\end{cases}
\end{equation*}
where $\hat{\theta_t}$ is the least square estimator for the unknown parameters, $\bar{p}_{t-1}$ is the average of the prices over the period $1$ to $t-1$, and $\delta_t=p^*(\hat{\theta}_t)-\bar{p}_{t-1}$. The intuition is that if the tentative price $p^*(\hat{\theta}_t)$ stays too close to the history average, we will introduce a small perturbation as price experimentation to encourage the parameter learning. The parameter $\kappa$ is a hyper-parameter.
\item \textbf{Greedy\_Single.} The price $p_t$ is set as the solution of the single-period pricing problem with inventory constraint based on  $(\hat{a}_t,\hat{b}_t)$:
\begin{align*}
     \max_{p_t} & \quad p_t\cdot g(\hat{a}_t+\hat{b}_t\cdot p_t) \\
\text{s.t.} & \quad   g(\hat{a}_t+\hat{b}_t\cdot p_t)\leq c, \nonumber \\
&  \quad p_t\in [0.1,5].
\end{align*}
Further, when $\hat{a}_t<0$ or $\hat{b}_t>0$, we choose $p_t=5$ for saving the inventory under large uncertainty.

\item \textbf{Greedy\_Dual.} Greedy\_Dual is same as TS\_Dual expect replacing $(\tilde{a}_t,\tilde{b}_t)$ by $(\hat{a}_t,\hat{b}_t)$. Specifically, we replace equation \eqref{eq_price_con} by
$$p_{t}=\argmax_{p\in [\mu_t,5]}\Tilde{r}(p;\mu_{t},\hat{a}_{t},\hat{b}_t).$$
\end{itemize}

\paragraph{Hyper-parameter Tuning.}\

For all UCB and TS algorithms, we choose the regularization parameter $\lambda=1$. And after moderate tuning, we choose the hyper-parameter (if any) of each algorithm as follows:
\begin{itemize}
    \item CILS:  We choose $\kappa=\frac{d}{10}$ in our experiments.
    \item UCB: We set the confidence set for UCB algorithm (with any number of Monte Carlo samples) by
$$\Theta_{t}=\left\{\theta\in \Theta: \left\|\hat{\theta}_{t-1}-\theta\right\|^2_{M_{t-1}}\leq\frac{d}{10}\right\}.$$
\item TS and TS\_Ori: We sample the TS parameter by
$(\hat{a}_t,\hat{b}_t)+\frac{\sqrt{d}}{10}\tilde{M}^{-1/2}_{t-1}\eta_t$
for TS and
$ \hat{\theta}_{t-1}+\frac{\sqrt{d}}{25}M^{-1/2}_{t-1}\eta_t$
for TS\_Ori.
\item TS\_Dual: We sample the TS parameter by
$(\hat{a}_t,\hat{b}_t)+ \frac{\sqrt{d}}{10}\tilde{M}^{-1/2}_{t-1}\eta_t,$
and the update step for the dual variable is set by:
$$\mu_{t+1}=\text{Proj}_{[0,5]}\left(\mu_t+0.05\cdot(D_t-c)\right).$$
\item Greedy\_Dual: The update step for the dual variable is set by:
$$\mu_{t+1}=\text{Proj}_{[0,5]}\left(\mu_t+0.05\cdot(D_t-c)\right).$$
\end{itemize}

\end{document}